\documentclass[11pt,oneside,english,reqno]{amsart}
\usepackage[T1]{fontenc}
\usepackage[utf8]{inputenc}
\usepackage[a4paper]{geometry}
\usepackage{mathrsfs}
\usepackage{amstext}
\usepackage{amsthm}
\usepackage{aliascnt}
\usepackage{amssymb}
\usepackage{color}
\usepackage[dvipsnames]{xcolor}
\usepackage{tikz-cd}
\usetikzlibrary{cd}

\usepackage[colorlinks,linkcolor=blue, urlcolor  = blue,citecolor=blue,pagebackref,hypertexnames=true, breaklinks]{hyperref}
\usepackage{dsfont}
\usepackage{enumerate}
\usepackage{verbatim} 
\usepackage{stmaryrd}
\usepackage{enumitem}
\usepackage{graphicx}
\usepackage{todonotes}
\usepackage{xcolor}
\usepackage{cleveref} %

\usepackage{float}
\usepackage{booktabs}

\usepackage{algorithm}
\usepackage{algorithmic}

\usepackage{wrapfig}
\usepackage{shuffle}

\usepackage{caption}
\usepackage{subcaption}
\usepackage{bbm}

\newcommand{\argmin}{\operatornamewithlimits{argmin}}

\usepackage{xstring}

\usepackage{tikz-cd}

\makeatletter

\numberwithin{equation}{section}
\numberwithin{figure}{section}

\theoremstyle{plain}
\newtheorem{thm}{Theorem}[section]

\newaliascnt{lem}{thm}
\newtheorem{lem}[lem]{Lemma}
\aliascntresetthe{lem}

\newaliascnt{prop}{thm}
\newtheorem{prop}[prop]{Proposition}
\aliascntresetthe{prop}

\newaliascnt{cor}{thm}
\newtheorem{cor}[cor]{Corollary}
\aliascntresetthe{cor}

\newaliascnt{example}{thm}
\newtheorem{example}[example]{Example}
\aliascntresetthe{example}

\newaliascnt{hyp}{thm}
\newtheorem{hyp}[hyp]{Hypothesis}
\aliascntresetthe{hyp}

\theoremstyle{definition}

\newaliascnt{defn}{thm}
\newtheorem{defn}[defn]{Definition}
\aliascntresetthe{defn}

\newaliascnt{nota}{thm}
\newtheorem{nota}[nota]{Notation}
\aliascntresetthe{nota}

\theoremstyle{remark}

\newaliascnt{rem}{thm}
\newtheorem{rem}[rem]{Remark}
\aliascntresetthe{rem}

\crefname{thm}{theorem}{theorems}
\Crefname{thm}{Theorem}{Theorems}

\crefname{lem}{lemma}{lemmas}
\Crefname{lem}{Lemma}{Lemmas}

\crefname{prop}{proposition}{propositions}
\Crefname{prop}{Proposition}{Propositions}

\crefname{cor}{corollary}{corollaries}
\Crefname{cor}{Corollary}{Corollaries}

\crefname{example}{example}{examples}
\Crefname{example}{Example}{Examples}

\crefname{hyp}{hypothesis}{hypotheses}
\Crefname{hyp}{Hypothesis}{Hypotheses}

\crefname{defn}{definition}{definitions}
\Crefname{defn}{Definition}{Definitions}

\crefname{nota}{notation}{notations}
\Crefname{nota}{Notation}{Notations}

\crefname{rem}{remark}{remarks}
\Crefname{rem}{Remark}{Remarks}

\makeatother

\usepackage{babel}

\newcommand{\Vone}{\mathcal{V}^1([0,T];\RR^m)}

\newcommand{\cA}{\mathcal{A}}

\newcommand{\cC}{\mathcal{C}}
\newcommand{\cD}{\mathcal{D}}

\newcommand{\cH}{\mathcal{H}}

\newcommand{\cK}{\mathcal{K}}
\newcommand{\cL}{\mathcal{L}}

\newcommand{\cT}{\mathcal{T}}

\newcommand{\cV}{\mathcal{V}}
\newcommand{\cW}{\mathcal{W}}

\definecolor{darkred}{rgb}{0.7,0.1,0.1}
\definecolor{darkblue}{rgb}{0.4,0.1,0.8}
\definecolor{darkgreen}{rgb}{0.1,0.7,0.1}
\definecolor{franck}{rgb}{0,0,1}
\definecolor{pagebackground}{rgb}{1,1,1}
\definecolor{pageforeground}{rgb}{0,0,0}

\colorlet{symbols}{blue!90!black}
\colorlet{connection}{red!30!black}
\colorlet{boxcolor}{blue!50!black}

\newcommand{\NN}{\mathbb{N}}

\newcommand{\RR}{\mathbb{R}}

\newcommand{\TT}{\mathbb{T}}

\newcommand{\bK}{\mathbf{K}}

\newcommand{\bw}{\mathbf{w}}

\newcommand{\bx}{\mathbf{x}}
\newcommand{\bY}{\mathbf{Y}}
\newcommand{\by}{\mathbf{y}}
\newcommand{\bZ}{\mathbf{Z}}
\newcommand{\bz}{\mathbf{z}}

 \newcommand{\VSig}[2]{\mathrm{VSig}(#1;#2)}

\definecolor{dg}{rgb}{0, 0.5, 0}
\definecolor{dp}{rgb}{0.50, 0, 0.40}

\newcommand{\la}{\langle}
\newcommand{\ra}{\rangle}
\newcommand{\Id}{\mathrm{Id}}

\newcommand{\vertiii}[1]{{\left\vert\kern-0.25ex\left\vert\kern-0.25ex\left\vert #1 
		\right\vert\kern-0.25ex\right\vert\kern-0.25ex\right\vert}}

\newcommand{\dd}{\mathop{}\!\mathrm{d}}

\newcommand{\conv}[6]{%
  \left[#1 \ast #2\right]%
  ^{%
    #3%
    \if\relax\detokenize{#3}\relax %
    \else
      \if\relax\detokenize{#4}\relax
      \else ,\fi
      #4%
    \fi
  }%
  _{%
    #5%
    \if\relax\detokenize{#5}\relax %
    \else
      \if\relax\detokenize{#6}\relax
      \else ,\fi
      #6%
    \fi
  }%
}

\newcommand{\oconv}[6]{%
  \left[#1 \oast #2^{#3}\right]^{#4}%
  _{%
    #5%
    \if\relax\detokenize{#5}\relax %
    \else
      \if\relax\detokenize{#6}\relax
      \else ,\fi
      #6%
    \fi
  }%
}

\makeatletter
\renewcommand{\l@section}{\@tocline{1}{0pt}{0em}{1.5em}{}}
\renewcommand{\l@subsection}{\@tocline{2}{0pt}{2em}{3em}{}}
\renewcommand{\l@subsubsection}{\@tocline{3}{0pt}{4em}{5em}{}}
\makeatother

\begin{document}

\title[
The Volterra signature %
]{The Volterra signature
}

\author{Paul P. Hager}
\address{Paul P. Hager, Department of Statistics and Operations Research, University of Vienna,
Kolingasse 14--16, 1090 Vienna, Austria}
\email{paul.peter.hager@univie.ac.at}

\author{Fabian N. Harang}
\address{Fabian N. Harang, Department of Economics, BI Norwegian Business School, Nydalsveien 37, 0484 Oslo, Norway}
\email{fabian.a.harang@bi.no}

\author{Luca Pelizzari}
\address{Luca Pelizzari, Department of Statistics and Operations Research, University of Vienna,
Kolingasse 14--16, 1090 Vienna, Austria}
\email{luca.pelizzari@univie.ac.at}

\author{Samy Tindel}
\address{Samy Tindel, Department of Mathematics, Purdue University,
150 N. University Street, West Lafayette, IN 47907--2067, USA}
\email{stindel@purdue.edu}

\date{\today}

\begin{abstract}

Modern approaches for learning from non-Markovian time series, such as recurrent neural networks, neural controlled differential equations or transformers, typically rely on implicit memory mechanisms that can be difficult to interpret or to train over long horizons.
We propose the \emph{Volterra signature} \(\VSig{x}{K}\) as a principled, explicit feature representation for history-dependent systems.
By developing the input path \(x\) weighted by a temporal kernel \(K\) into the tensor algebra, we leverage the associated Volterra--Chen identity to derive rigorous learning-theoretic guarantees.
Specifically, we prove an \emph{injectivity} statement (identifiability under augmentation) that leads to a 
\emph{universal approximation} theorem on the infinite dimensional path space, which in certain cases is achieved by \emph{linear functionals} of $\VSig{x}{K}$. 
Moreover, we demonstrate applicability of the \emph{kernel trick} by showing that the inner product associated with Volterra signatures admits a closed characterization via a two-parameter integral equation, enabling numerical methods from PDEs for computation. 
For a large class of exponential-type kernels, \(\VSig{x}{K}\) solves a linear state-space ODE in the tensor algebra.
Combined with inherent invariance to time reparameterization, these results position the Volterra signature as a robust, computationally tractable 
feature map for data science.
We demonstrate its efficacy in dynamic learning tasks on real and synthetic data, where it consistently  improves  classical path signature baselines.
\end{abstract}


\thanks{
\emph{Acknowledgments.} 
FH: This work was supported by the SURE-AI Centre grant 357482, Research Council of Norway. All authors would like to thank Eduardo Abi Jaber for helpful discussions.}

\maketitle

\section{Introduction}

\emph{Memory-effects} in real-world data streams are ubiquitous: their future evolution depends not only on the current state but also on the \emph{history} of the signal, often in intricate ways, with different importance assigned to distant versus recent changes and to periods corresponding to different regimes.
To name only a few examples of history-dependent phenomena, we mention biological and neuroscientific systems exhibiting synaptic plasticity and delay effects~\cite{Gerstner2002,Kou08}, memory in disease spread models~\cite{BCF19}, the long- and short-range dependence observed in realized volatility~\cite{comte1996long,Gatheral2018,guyon2023volatility} and a variety of engineering and signal-processing applications featuring long or short memory~\cite{Boyd1985,Leland1994,BB08,CC14}.

To learn such phenomena from data, increasingly complex architectures such as recurrent neural networks (RNNs) \cite{Elman1990, Jordan1997SerialOA},
long short-term memory networks (LSTMs) \cite{HochreiterSchmidhuber1997}, Transformers \cite{Vaswani2017,wang2020linformer}, and more recently structured State Space Models (SSMs) like S4 and Mamba \cite{gu2022efficiently,gu2023mamba,dao2024transformers}
have been introduced and widely employed. These approaches emulate memory \textit{implicitly} by storing information about the past in hidden states or attention matrices.
While powerful, they often function as highly parameterized ``black boxes'' \cite{Lipton2016,Rudin2019}:
the learned dependence structure is difficult to interpret, and training can be notoriously unstable over long horizons due to vanishing or exploding gradients
\cite{BengioSimardFrasconi1994,PascanuMikolovBengio2013}. Moreover, implicit memory models may require large amounts of data to recover decay structures that can be
encoded a priori in scientific settings \cite{RaissiPerdikarisKarniadakis2019,KarniadakisEtAl2021}.

In contrast, a classical mathematical approach to modeling systems with memory is via \emph{Volterra} dynamics, where a kernel \(K(t,s)\) explicitly encodes how past inputs at time \(s<t\) influence the present at time \(t\). To make this concrete in discrete time, consider a sequence of innovations of a signal \((\Delta x_{t_i})_{i=1,\dots,n}\) driving the system; in our context, these are the increments of an incoming \emph{data stream}. Information from the past is then propagated forward through accumulation,
\begin{equation}\label{eq:first_order_memory}
    \sum_{i\le j} K(t_j,t_i)\,\Delta x_{t_i}.
\end{equation}
This quantity can be interpreted as a (first-order) memory state that influences the system at time \(t_j\).
Such a kernel-based formulation accounts for long-range effects through the off-diagonal decay of $K$ (e.g., exponential or power-law decay), short-range effects through its near-diagonal behavior (e.g., fractional-type behavior), and oscillatory regimes through periodic terms.
Passing now to continuous time, a Volterra model driven by a signal \(x\) can be formulated through the integral equation
\begin{equation}\label{eq:vcde_intro}
    z_t = z_0 + \int_0^t f(z_s)\, K(t,s)\,\dd x_s, \qquad t \ge 0,
\end{equation}
where, for simplicity, we interpret \(\dd x_s\) as \(\frac{\dd}{\dd s}x_s\,\dd s\) (i.e., \(x\) is assumed differentiable).
Such models are standard in control theory \cite{corduneanu2008integral,gripenberg1990volterra} and in stochastic analysis \cite{pardoux1990stochastic,oksendal1993stochastic}. 
In domain sciences applications the vector field \(f\) is typically specified by a small number of interpretable parameters.
To turn \eqref{eq:vcde_intro} into a versatile machine-learning model, one can---similarly to successful approaches for neural ordinary and neural controlled differential equations \cite{chen2018neural,KidgerMorrillLyons2020}---replace \(f\) by a larger parametric model, leading to an instance of a \emph{neural integral equation} \cite{zappala2024learning, proemel2025neural}.

Here we propose a different, representation-based approach by constructing a universal feature map above \eqref{eq:vcde_intro}.
In this context, the \emph{path signature}---the sequence of iterated integrals as put forward by Chen \cite{Chen1957} and later by Lyons \cite{Lyons1998}---arises canonically in controlled and stochastic differential equations as a feature map 
that separates paths up to
reparameterization invariances \cite{HamLy10} and yields a universal linearizing feature map for continuous functionals of paths.
These properties have powered a vibrant line of work employing signatures in machine learning tasks on sequential data 
(see \cite{BayerDosReisHorvathOberhauser2025SignatureMethodsInFinance,McLeodLyons2025SignatureMethodsML, CrisanEtAl2026SAA2025} and references therein).
The applicability of the kernel trick~\cite{Kiraly2019,Salvi2021}---namely, the fact that the signature kernel can be computed directly and efficiently---has opened up the kernel universe and provided an additional boost to the field.
However, the classical signature is inherently tailored to {differential} (local) systems driven directly 
by the signal---it does not {natively} encode the additional, kernel-mediated (global) interactions that define a {Volterra} (history-dependent) system.

This article develops an object called \emph{Volterra signature} (denoted in the sequel by \(\mathrm{VSig}\)) for smooth paths, a kernel-weighted counterpart of the classical signature designed to 
add control over the memory effects in signature based modeling and learning.  
Informally, \(\mathrm{VSig}(x;K)\) contains higher-order extensions of the \(K\)-weighted memory state defined in~\eqref{eq:first_order_memory}.
Formally, it is the collection of iterated integrals arising from the Picard expansion of linear Volterra equations, i.e., of~\eqref{eq:vcde_intro} for \emph{all} linear \(f\).

Our main contributions can be summarized as follows:

\begin{itemize}
    \item \emph{Algebraic and analytic structure.}
    We construct \(\mathrm{VSig}\) in the smooth setting and prove a Chen-type identity: \(\mathrm{VSig}\) is closed under a convolution product that encodes temporal concatenation for Volterra paths.
    This structure incorporates the propagation of information from the past when composing features across time segments.
    We further demonstrate that solutions to linear Volterra equations admit a full expansion in terms of the Volterra signature, which thus plays the role of a multidimensional \emph{resolvent} for controlled Volterra equations.
    For a large class of exponential-type kernels (including sums of exponential and periodic kernels), we show that \(\mathrm{VSig}\) can be obtained by solving a state-space ODE.
    Moreover, for \(K(t,s)=\alpha e^{-\lambda (t-s)}\), we prove an invertible conversion formula to the classical signature.

    \item \emph{Invariance, injectivity, universality and the kernel trick:}
    We provide the theoretical foundation for using \(\mathrm{VSig}\) as a feature map by proving time-reparameterization invariance and injectivity (ensuring model identifiability under natural augmentations).
    This entails a Stone--Weierstrass-type universality theorem for approximating continuous functionals on path space.
    We also prove universality for linear functionals in the case of exponential kernels, leaving the general kernel case as an interesting open problem.
    To complete the picture, we show that the feature kernel induced by \(\mathrm{VSig}\) admits a closed-form characterization via a two-parameter integral equation, thereby enabling standard kernel methods.

    \item \emph{Numerical experiments:}
    We present a collection of experiments demonstrating the numerical efficacy of
the Volterra signature for learning problems with underlying dynamical
structure. We begin with a synthetic problem, where the goal is to learn the
solution map of a stochastic Volterra equation as a function of the driving
noise. We then move to real-world data and consider S\&P 500 time series, where
we forecast realized volatility. In the third experiment, we classify several
UEA time-series datasets \cite{bagnall2018uea}. In all settings, incorporating
a temporal kernel yields substantial gains in predictive accuracy over classical
signature baselines. The code required to reproduce all experiments is available at
\url{https://github.com/lucapelizzari/Volterra_signature_learning}.

\end{itemize}

For the practical applicability of a feature map, it is of course crucial to have efficient algorithms for its computation.
For the Volterra signature, these algorithmic aspects are intricate, as they require resolving both the algebraic and analytic components of the Chen identity.
We therefore address these issues in a companion paper \cite{ii_part}, where we derive efficient algorithms applicable to a wide range of kernels, including fractional and gamma kernels. The algorithms presented in \cite{ii_part} are implemented in the publicly
available package \texttt{tensordev}; see
\url{https://github.com/hagerpa/tensordev}.
The latter supports general higher-order schemes for computing the Volterra signature, which due to the underlying Volterra structure exhibit quadratic scaling in the number of time steps. For convolutional kernels, however, this cost can be reduced to almost linear complexity, of order $O(J\log(J))$, by exploiting fast Fourier transforms (FFT). For exponential-type kernels, this can be pushed further: as foreshadowed by the state-space ODEs presented in this paper, the Volterra signature can be computed with a cost that scales \emph{linearly in the number of time steps}.

While the above covers the most essential questions for the employment of \(\mathrm{VSig}\) in data science, it also opens a wide-ranging programme of mathematical questions and possibilities for tackling challenging learning tasks.
To name a few directions on the theoretical side, we have not yet developed a precise algebraic structure satisfied by the Volterra signature.
Some work in this direction has been done in \cite{HTW23,BrunedKatsetsiadis2023VolterraRough}; however, the representations of Volterra signatures for exponential kernels given below suggest that simpler structures may arise in this case.
Another question of interest is to characterize the class of paths with trivial Volterra signature, which again deviates from the tree-like equivalence in the classical signature setting.
Finally, when transitioning from smooth drivers \(x\) to stochastic processes such as Brownian motion, the Volterra signature can still be defined via the canonical lift of \cite{HT21}.
From the computational perspective adopted in this paper, where the signal \(x\) is obtained by smooth interpolation, a further pressing question concerns Wong--Zakai-type limit results, with renormalization required in the singular-kernel case, as in \cite{hairer2015wong,BFGMS20}.

Beyond the numerical examples studied in this paper, there are several further application areas in which the Volterra signature may offer advantages over the classical signature. For instance, we expect \(\mathrm{VSig}\) to perform particularly well in dynamical learning tasks, such as \emph{reinforcement learning}, and in settings where the underlying structure exhibits stationarity, long-range dependence, or fractional scaling relations. Furthermore, damped periodic kernels (Prony) within the Volterra signature framework appear especially promising, and they could lead to a powerful class of path-dependent models in mathematical finance. Finally, and perhaps most intriguingly, further exploration of kernel learning within the Volterra signature framework provides a natural next step, for instance in path-dependent regression. Conditional time-series generation in the spirit of \cite{liao2024sig} provides a natural application, where Volterra signature kernels can be used to define MMD-type losses with explicit control of the memory structure through the choice of Volterra kernel.

\noindent\textbf{Related literature.}
The definition of the Volterra signature and its Chen relation are rooted in the analytic construction of Volterra rough paths in \cite{HT21,HTW23}.
In extension to classical rough path theory, these works establish continuity properties of the solution map for singular Volterra equations.
As such, they are related to other extensions, most notably those based on regularity structures \cite{BFGMS20, BrunedKatsetsiadis2023VolterraRough}.
However, these analytic results will not play a direct role here, as we focus primarily on the nonsingular case and on the expressive properties of the \emph{full lift}.
While there are general approaches to defining features from regularity structures \cite{chevyrev2024feature}, this has not yet been developed for the Volterra setting.
By contrast, other extensions of rough path theory, such as branched rough paths, have already been suggested for use in data science \cite{diehl2023generalized,ali2025branched}.

More closely related to the data-science context, the work \cite{jaber2025exponentially} is the closest to the present paper.
It introduces a modification of the signature with \emph{exponentially fading memory} (EFM), providing an alternative way to incorporate an exponential kernel.
This approach exploits the fact that the Volterra equation \eqref{eq:vcde_intro} with \(K(t,s)=\alpha e^{-\lambda (t-s)}\) can be rewritten in a Markovian (mean-reverting) dynamical form, which naturally induces dynamics on the signature group.
While this has the strong advantage that the EFM signature retains the classical algebraic properties of the signature, it is tightly tied to the exponential-kernel setting.
In contrast, our approach lifts the Volterra equation directly, which alters the algebraic structure, but applies to any suitably integrable memory-kernel.
Finally, on the architectural side, our formulation of finite state space kernels provides a rigorous tensor-algebraic counterpart to the linear continuous-time memory mechanisms driving recent deep learning state space models \cite{gu2022efficiently,gu2023mamba}, a connection we formalize in Section \ref{sec:conection to deep learning}.

\noindent\textbf{Organization of the paper.}
We begin by formally defining the Volterra signature and establishing its convolutional Chen property in \Cref{sec: the volterra signature}, including a brief preliminary on tensor-algebra notation.
In \Cref{sec:fundamental} we characterize \(\mathrm{VSig}\) as the solution to a fundamental linear Volterra equation in the tensor algebra.
In \Cref{sec:dyn_exp} we prove a dynamical representation of \(\mathrm{VSig}\) for a large class of exponential-type kernels.
Section \ref{sec: Sig as feature map} then develops key properties of reparameterization invariance, injectivity, universality and the applicability of the kernel trick.
Finally, in Section \ref{sec:learning Volterra sde} we demonstrate the effectiveness of the Volterra signature in numerical experiments on synthetic and real data.

\section{The Volterra signature}\label{sec: the volterra signature}

This section is devoted to introduce the definition and fundamental properties of the Volterra signature studied in this article. 
While some notions here are borrowed from~\cite{HT21}, we focus exclusively on smooth Volterra paths and signatures. 
We will present a bottom-up construction of the Volterra signature, starting from the coordinate-wise definition as a sequence of iterated Volterra integrals. In analogy with the classical signature, we show that the intrinsic definition can be characterized via a fundamental Volterra equation in the tensor algebra. This formulation immediately yields the associated \emph{Chen's relation} with respect to a suitable convolution-type tensor product $\oast$. We justify the terminology ``fundamental'' by showing that the Volterra signature acts as a resolvent for linear controlled Volterra equations. Finally, we study a class of \emph{finite-state-space kernels} which allow one to realize Volterra signatures as solutions to a mean-reverting system of equations, which are later exploited in the numerical experiments of Section~\ref{sec:learning Volterra sde}.

\subsection{Preliminaries}\label{preliminaries}

Before defining our main objects of study, let us introduce some general notation which will be used throughout the article. We start by labeling our notation for simplexes.

\begin{defn}
The $n$-simplex in a hypercube $[s,t]^n$ is defined by 
\begin{equation}\label{a1}
    \Delta^n_{s,t} := \{(r_1,\ldots,r_n)\in [s,t]^n \,|\, r_1\leq \dots \leq r_n\,\}.
\end{equation}
Whenever the underlying interval $[s,t]$ is not relevant (or whenever it is clear from context), we simply write $\Delta^n$. 
 \end{defn}
 \noindent
 Next we will need the notion of words created from an alphabet in order to properly define our notion of signature.
\begin{defn}\label{def:words}
 For any $d\in\NN_{\ge1}$ we set $\cA_d:=\{1,\dots,d\}$ and call it an \emph{alphabet} with $d$ letters.
 A \emph{word} in the alphabet $\mathcal{A}_d$ of length $n\in\NN$ is given by  $w=i_1\cdots i_n$ where $ i_k\in \cA_d$ for all $k=1,\dots, n$.
 We denote by $\cW^n_d$ the set of all words of length $n\in\NN$ from the alphabet $\cA_d$. 
 For $n=0$ we write $\cW^0_d=\{\varnothing\}$, where $\varnothing$ is the empty word. We further let $\cW_d$ denote the set of all words over the alphabet $\cA_d$. 
 For a word $w = i_1\cdots i_n\in\cW$ we denote the length of the word by $|w| = n$, so that $\cW^n_d = \{ w\in\cW_d:\; |w| = n\}$. 
 Whenever the dimension $d$ is not important or otherwise well understood, we simply write $\cW$ or $\cA$. 
 Eventually we introduce the algebra $\RR\langle\cA_d\rangle= T(\text{vect}_\RR(\cA_d))$, the vector space generated by $\cW_d$.
\end{defn}
 \noindent
We now define concatenations of two words in $\cW$.
\begin{defn}\label{def:free_algebra}
On the set of words $\cW_d$ we introduce the \emph{concatenation product} by setting, for two words $w=i_1\cdots i_n\in\cW^n_d$ and $v=j_1\cdots j_m\in\cW^m_d$,
$$wv=i_1\cdots i_n j_1\cdots j_m\in \cW^{n+m}_d.$$
In particular, for concatenation with the empty word we set $w\varnothing=\varnothing w=w$.
In other words, $\cW_d$ forms the free non-commutative monoid with neutral element $\varnothing$ over the indeterminates $\mathcal{A}_d$.
The free algebra over $\mathcal{A}_d$ is then represented by
$$\RR\la \cA_d\ra:=\left\{\ell =\sum_{w\in \cW_d}\alpha_w w \;\bigg\vert\; \alpha_w\in\RR,\; \alpha_w\neq 0\text{ for only finitely many }w\in \cW_d\right\},$$
where the concatenation product on $\cW_d$ extends by bilinearity to a product on $\RR\la \cA_d\ra$.
\end{defn}

The algebraical dual of $\RR\la \cA_d\ra$ consist of infinite series of words, which will play a crucial role to use.
We will represent these objects as tensor series over the euclidean vector space $V = \RR^d$.
The basic notions we will use are summarized below.
\begin{defn}\label{def:extended_TA}
    Let $\{e_1, \dots, e_d\} \subset V$ be a basis.
    Set \(V^{\otimes 0}=\RR\) and for $n\in \NN_{\ge1}$ denote by \(V^{\otimes n}\) the \(n\)-fold tensor power of the vector space \(V\).
    For any word $w = i_1 \cdots i_n\in \cW_d$ we define $e_{w} = e_{i_1} \otimes \cdots \otimes e_{i_n} \in V^{\otimes n}$ and note that $\{ e_w \,\vert\, w\in \cW^n_d\}$ forms a basis of $V^{\otimes n}$.
    The \emph{extended tensor algebra} over $V$ is given by the direct product
    $$T((V)) := \prod_{n\ge0} V^{\otimes n},$$
    which forms a vectors space by componentwise summation and scalar multiplication.
    For element in $\bx \in T((V))$ we ill use the following equivalent notations as a sequence and as a formal series for the decomposition into its tensor \emph{levels}  $$\bx = (\bx^{(0)}, \bx^{(1)}, \dots) = \sum_{n=0}^\infty \bx^{(n)} \in T((V)),$$  so that $\bx^{(n)} \in V^{\otimes n}$ for all $n\in \NN$.
    We define the \emph{tensor product} $\otimes$ on $T((V))$ in terms of the concatenation product of basis elements, i.e., 
    $e_{w} \otimes e_{v} = e_{wv}$ for all $w,v\in\cW,$
    which extends to all of $T((V))$ by bilinearity.
    In particular, for ${\bf a}, \bx,\by\in T((V))$ we may write
$$\bz={\bf a} + \bx\otimes\by,\qquad \bz^{(n)}={\bf a}^{(n)} + \sum_{k=0}^{n}\bx^{(k)}\otimes \by^{(n-k)}\in V^{\otimes n},\qquad n\in\NN.$$
    \end{defn}

    \begin{defn}\label{def:proj_and_trunc}
        We write \(\pi_n:T(V)\to V^{\otimes n}\) for the canonical projection, so that $\pi_n(\bx) = \bx^{(n)}$ for all $\bx\in T((V))$ and $n\in \NN$.
        Furthermore, for $N\in\NN$ we define the \emph{tensor truncation} $$\pi_{\le N} : T((V)) \to T((V)), \qquad \bx \mapsto (\bx^{(0)}, \bx^{(1)}, \dots, \bx^{(N)}, 0, \dots).$$
        The image of the truncation map $T^N(V) := \pi_{\le N} T((V))$ is called the \emph{truncated tensor algebra}, which indeed forms an algebra under the truncated tensor product
        $$\bx \otimes_N \by := \pi_{\le N}(\bx \otimes \by), \qquad \bx,\by\in T^N(V).$$
    \end{defn}
    
    \begin{defn}
        We define a dual paring $$\langle \cdot, \cdot\rangle: \RR\la \cA_d\ra \times T((V))\to \RR,$$ by setting $\langle w, e_v \rangle
        = 1_{w=v}$ for all $v,w\in \cW_d$ and extending by bilinearity.
        In this vein we define components of a tensor series by setting
        $$\bx^w = \langle w, \bx \rangle, \qquad \bx\in T((V)),\quad  w\in \cW_d.$$
        In particular, any element $\ell = \sum_{k=1}^n \ell_k w_k \in \RR\la \cA_d\ra$ defines a linear functional on $T((V))$ by
        $$\langle \ell, \cdot \rangle : T((V))\to\RR, \quad \bx \mapsto \sum_{k=1}^n \ell_k \bx^{w_k}$$
    \end{defn}

The above is only one instance of a dual pairing within $T((V))$.
Another, important paring is given through the 
the following natural Hilbert-space in $T((V))$.

\newcommand{\cThilb}{\cT^2}
\begin{defn}\label{def:inner_pro}
    We define the subspace $\cThilb \subseteq T((V))$ by  \[\cThilb:=\left  \{\mathbf{x}\in T((V))\,\Big\vert\, \Vert \bx \Vert_{\cThilb}:= \sqrt{\la \bx,\bx \ra_{\cThilb}}<\infty \right \},
    \] where $\la \mathbf{x},\mathbf{y}\rangle_{\cThilb}:= \sum_{k\geq 0}\langle \mathbf{x}^{(k)},\mathbf{y}^{(k)}\rangle_{V^{\otimes k}}$ for all $\bx,\by \in T((V))$, and $\la \cdot,\cdot \ra_{V^{\otimes k}}$ is defined by the Hilbert-Schmidt inner product on $V^{\otimes k}$, i.e.  \[
    \la v,w\ra_{V^{\otimes k}}:=\prod_{l=1}^k\la v_l,w_l\ra, \qquad v=v_1\otimes \cdots \otimes v_k,\quad w=w_1\otimes \cdots \otimes w_k,
    \]
    where $\la v_l,w_l\ra = \la v_l,w_l\ra_V$ denotes the standard product on $V = \mathrm{span}\{e_1, \dots, e_d\}$.
\end{defn}

At last, we will introduce the following concept of extending linear maps on $V$ to linear operators acting on tensors, as it will be needed later for the expansion of linear Volterra equations.
\begin{defn}
Let $E \cong \RR^m$ be another vector space.
We denote by $\mathcal L(V, E)$ the set of all linear maps from $V$ to $E$.
When $V = E$ we simply write $\mathcal L(V, V) = \mathcal{L}(V)$ for the set of endomorphisms.
Next note that $\mathcal L(V, E)$ equipped with the usual operator norm forms a vector space isomorphic to  $\RR^{n\times m}$. 
Then for any \(A\in\mathcal L(V,\mathcal L(E))\) we define a linear map  $\widetilde{A}: \bigoplus_{n=0}^\infty V \to \mathcal L(E)$
by setting $\widetilde A(1)=\Id_E$ and 
\begin{equation}\label{eq:linear_tensor_maps}
\widetilde A(v_1\otimes\cdots\otimes v_n)=A(v_n)\cdots A(v_1), \qquad  v_1, \dots, v_n \in V, \quad n\in \NN_{\ge 1},
\end{equation}
and then extend by linearity.
\end{defn}

\subsection{Iterated Volterra integrals}\label{sec:smooth-v-sig}

This section is devoted to define our main object of study, that is iterated Volterra integrals of a Lipschitz continuous path $x$. Related to this notion, we will also specify what we mean by the Volterra signature. We begin by introducing the following shorthand.
\begin{nota}\label{hyp:smoothness} For a Lipschitz continuous path $x$ starting in $0$, which we denote by $x \in \cC^{0,1}([0,T],\RR^d)$, satisfying $x = \int_0^\cdot \dot{x}_t \dd{t}$, we will write $\dd x_t = \dot{x}_t\dd t$.
In particular, for $A\in \mathcal{L}(\RR^d,\RR^m)$ and an index $i\in\mathcal{A}_{m} = \{1, \dots, m\}$ we set
\begin{equation}\label{a2}
     A^i\dd x_t = \sum_{j=1}^d  A^{ij}(\dot x^j_t) \dd{t}.
\end{equation}
\end{nota}

\subsubsection{Basic definition of Volterra signature}
Our generic Volterra path is obtained by weighting a differentiable path by a kernel $K$. 
In the same line as $x$ is assumed to be differentiable, we will assume that the kernel is sufficiently regular.
Specifically, since the noise $x$ is assumed to be differentiable, it will suffice to suppose integrability of $K$ in the second variable uniformly over the first one. 
\newcommand{\Lkernel}{L^{\infty,1}(\Delta^2; \mathcal{L}(\RR^d;\RR^m))}
\newcommand{\Ltwokernel}{L^{\infty,2}(\Delta^2; \mathcal{L}(\RR^d;\RR^m))}

\begin{hyp}\label{def:kernel_class}
Consider a matrix-valued kernel $K: \Delta^2 \to \mathcal{L}(\RR^d, \RR^m)$. In the sequel we assume $K\in \Lkernel$, where the latter space stands for the set of measurable functions on $\Delta^2$ which satisfy $$\sup_{t\in[0,T]}\int_0^t |K(t,s)| \dd{s} < \infty,$$
    with $|\cdot|$  denoting here any matrix norm on $\mathcal{L}(\RR^d,\RR^m)$.
\end{hyp}

\begin{example}
    Clearly any continuous function $K:\Delta^2 \to \mathcal{L}(\RR^d;\RR^m)$ lies in the space $\Lkernel$ alluded to in Definition~\ref{def:kernel_class}.
    In particular, the constant kernel $K\equiv A\in \mathcal{L}(\RR^d;\RR^m)$, as well as sum of matrix exponentials kernel $$K(t,s) = \sum_{p=1}^q e^{-\lambda_p(t-s)} A_p,\qquad \lambda_p > 0, \quad A_p \in \RR^{m\times d}, \quad (s,t) \in \Delta^2,$$ which will be studied in detail in Section~\ref{sec:dyn_exp}, lie in this class.
\end{example}

\begin{example}
Another relevant class of kernels which fall under the banner of our Hypothesis~\ref{def:kernel_class} are those $K$'s admitting  at most a fractional singularity of order $\gamma \in(0,1)$. Otherwise stated, if we assume $|K(t,s)| \lesssim|t-s|^{-\gamma}$, then $K$ will be in $\Lkernel$.
    This type of kernel is also at the heart of the analysis for the rough setting in~\cite{HT21}.
    Note that this includes fractional kernels of the form $K(t,s) = (t-s)^{\beta-1} A$ for some $A\in \cL(\RR^d; \RR^m)$ and $\beta>0$.
\end{example}

\noindent
Matrix-valued kernels give us some extra freedom to model data with nontrivial dependence on the past. However, our setting obviously encompasses the case of scalar kernels considered in~\cite{HT21}.
We now set up a notation for this specific case.
\begin{nota}
     For fixed $d=m$ and a given one-dimensional kernel $k:\Delta^2 \to \RR$ satisfying Hypothesis~\ref{def:kernel_class}, we will identify $k(t,s)$ with the diagonal matrix
\begin{equation}\label{eq:diag}
         K(t,s) = \mathrm{diag}(k(t,s), \dots, k(t,s)) \in \RR^{m\times m}.
     \end{equation}
\end{nota}
With our notion of matrix-valued kernel in hand, we are now ready to introduce the most general form of Volterra signature considered in this paper. This object summarizes all the iterated integrals of the signal $x$ weighted by the kernel $K$.

\begin{defn}\label{def:vsig} Recall that the simplex $\Delta_{s,t}^n$ is defined by~\eqref{a1} and that the alphabet $\cA$ is introduced in Definition \ref{def:words}. Consider a signal $x\in\mathcal{C}^{0,1}([0,T]; \RR^d)$, as well as a kernel $K$ with values in $\cL(\RR^{d};\RR^{m})$ that satisfies Hypothesis~\ref{def:kernel_class}.
We define the \emph{Volterra signature} $\mathrm{VSig}$ component-wise for all $n\in\NN$, $i_1 \cdots i_n \in \mathcal{W}^n_m$ and $(s,t,\tau)\in\Delta^3$ by
\begin{equation}\label{eq:def_vsig_comp}\VSig{x}{K}^{i_1\cdots i_n,\tau}_{s,t} = \int_{\Delta^n_{s,t}} \prod_{l=1}^n K^{i_l}(r_{l+1},r_l)\dd{x_{r_l}},\end{equation}
with the convention $r_{n+1} = \tau$, where we also recall from Definition \ref{hyp:smoothness} that $ K^{i_l}\dd x$ denotes component $i_l$ in the  $\RR^m$ valued vector $K(\tau,\cdot)\dd{x_\cdot}$.
\end{defn}

\begin{rem}
    Note in particular that the word $w=i_1\cdots i_n$ in Definition \ref{def:vsig} is taken from the alphabet $\cA_m$, generated from the dimension of the product $K\dd x \in \RR^m $. This is in contrast to the case of classical signatures, where the words are generated over the alphabet $\cA_d$ coming from the $d$-dimensional nature  of the path $x$. Of course, in the case where $K$ takes the diagonal matrix form in \eqref{eq:diag}  with $m=d$, then the alphabet $\cA_m=\cA_d$. Therefore, in this situation, we keep the same intuitive understanding of the word related to the signal $x$. The flexibility introduced in allowing for matrix valued $K$ is often necessary in practical modeling, as illustrated for example in the theory of Ambit fields \cite{BarndorffNielsen2018AmbitStochastics}, or modeling of commodity derivatives (as illustrated for example in \cite{BenthKruehner2023EnergyCommodity}). 
\end{rem}
\newcommand{\Sig}[1]{\mathrm{Sig}(#1)}

\begin{rem}\label{rem:usual_sig}
In view of the above \Cref{prop:vsig_z} we readily observe that in case $K(\cdot, \cdot) \equiv A \in \mathcal{L}(\RR^d, \RR^m)$ 
the Volterra signature reduces to the usual signature of a smooth path $Ax$ (as introduced in \cite{chen1954iterated, Lyons1998}).
Indeed, for $n\in \mathbb{N}$ and $i_1\cdots i_n \in \cW_d^n$ it reads 
\begin{equation}\label{usual_sig_def}
    \VSig{x}{A}^{i_1\cdots i_n} = \mathrm{Sig}(Ax)_{s,t}^{i_1\cdots i_n} = \int_{\Delta_{s,t}^n} \prod_{l=1}^n A^{i_l}\dd x_{r_l},
    \end{equation}
The above will serve as a definition for the ``usual'' or ``classical'' signature.
\end{rem}
\begin{rem}
We stress that in general, unlike the usual signature of a path, $\VSig{x}{K}$ does not define a group-like element, i.e., does not admit a (simple) shuffle identity. See ~\cite[Remark 11]{HT21} for more details about this fact.
\end{rem}

\subsubsection{Intrinsic definition of Volterra signature}
In order to specify $\VSig{x}{K}$ in Definition~\ref{def:vsig}, we were given a signal $x$, a kernel $K$ and every word in $\cW=\cW_m$. However, in our smooth setting one can achieve a full description of the Volterra signature starting from its first level only. This first level is denoted by 
\begin{equation}\label{eq:first_level}
    z = \{ z^i= \VSig{x}{K}^i: i=1,\dots,d \}.
\end{equation}
Below we properly define paths of the form \eqref{eq:first_level} and see how they relate to the full Volterra signature.

\begin{defn}\label{def:volterra_path}
    We define the set of \emph{differentiable Volterra paths} $\Vone$ as the set of functions 
    $z:\Delta^2 \to \RR^m: (t,\tau)\mapsto z^\tau_t$ such that
    \begin{enumerate}[label=(\roman*)]
        \item $z$ is bounded and measurable,
        \item $z^\tau_0 = 0$ for all $\tau \geq 0$,
        \item $t \mapsto z^{\tau}_{t}$ is %
        absolutely continuous for almost every $\tau \in[0,T]$.%
    \end{enumerate}
We further identify $\Vone$ with its quotient modulo almost everywhere equivalence. In addition, for $z\in \cV^1([0,T];\RR^m)$ we set \begin{equation}\label{eq:diff_volterra_path}
    \dd{z^\tau_t} := \frac{\dd z^\tau_{t}}{\dd t}\dd{t}.
\end{equation}
A norm on $\Vone$ is then defined by \begin{equation}\label{eq:norm}
    \Vert z \Vert_{\cV^1}= \sup_{t\in [0,T]} \int_0^t \Big \vert \frac{\dd z_u^t}{\dd u} \Big \vert \dd u.
\end{equation}
\end{defn}

As the title of this section indicates, the notion of Volterra path in Definition \ref{def:volterra_path} is more intrinsic than Definition~\ref{def:vsig}. Otherwise stated, it does not refer to a specific path $x$ or a kernel $K$. One can go back to a more extrinsic version thanks to the lemma below.
\begin{lem}\label{lem:equivalence_def_VP}
Let $\cV^1([0,T];\RR^m)$ be the space introduced in Definition \ref{def:volterra_path}. For $z: \Delta^2\to\RR^m$ it holds that  $z \in \Vone$ if and only if it is of the form %
\begin{equation}\label{eq:Volterra_path_def}
    z^\tau_t = \int_0^t K(\tau, s)\dd{x_s},
\end{equation}
in for some $x\in\mathcal{C}^{0,1}([0,T];\RR^d)$
and $K\in\Lkernel$.
\end{lem}
\begin{proof}
Clearly, any path of the form \eqref{eq:Volterra_path_def} satisfies the properties (i)-(iii) of \Cref{def:volterra_path}.
Hence, we only have to prove that $z\in\Vone$ is of the form \eqref{eq:Volterra_path_def} for some $x\in\mathcal{C}^{0,1}([0,T];\RR^d)$ and $K\in\Lkernel$.
 To this end, we define a finite signed measure on $[0,T]^2$ by setting $\nu([0,t]\times[0,\tau]):= \int_{0}^{\tau} z^{\theta}_{t\wedge \theta} \dd{\theta}$.
 By differentiability of $z$ it follows that $\nu$ is absolutely continuous with respect to the Lebesgue measure on $[0,T]^2$. Thus by the Radon-Nikodym theorem there exists a measurable $K:[0,T]^2 \to \RR$ such that $\nu(A) = \int_{A} K(t,s) \dd{s}\dd{t}$.
 By definition of $\nu$ we readily verify that  \eqref{eq:Volterra_path_def} holds almost everywhere on $\Delta^2$ with $K(\tau,t) \in \RR^m \cong \RR^{m\times 1}$ and $x: [0,T] \to \RR^1$ defined by $x_{t}=t$.
 Since $\nu$ is finite, it is also clear that $K\in\Lkernel$. It is then readily checked that $z_t^\tau$ can be realized as \[
 z_t^\tau = \int_0^tK(\tau,s)\dd s,
 \] which finishes our proof. 
\end{proof}

\begin{rem}\label{rem:proof_rem}
    The proof of Lemma \ref{lem:equivalence_def_VP} highlights the non uniqueness of the pair $(K,x)$ whenever a Volterra path $z$ is given. Indeed, we have simply used $K(\tau,t)\in \RR^m$ and $x_t=t$, while other choices would certainly be in order. 
    To justify why we choose to define the Volterra signature associated to a given pair $(x,K)$, instead of giving a somewhat less ambiguous definition directly based on a suitable class of \emph{smooth Volterra paths} $z$ (defined below), we note that in the application we have in mind, $K$ is usually fixed but $x$ will vary (or otherwise $x$ might also be fixed and $K$ may vary). Moreover, in Section~\ref{sec:injectivity} we will study the injectivity of the Volterra signature map when keeping either of the components fixed.
\end{rem}

We are now in a position to prove that one can generate the full signature $\VSig{x}{K}$ from its first level.
\begin{prop}\label{prop:vsig_z}
    Let $x \in \mathcal{C}^{0,1}([0,T]; \RR^d)$ and $K\in L^{\infty,1}(\Delta^2; \mathcal{L}(\RR^d;\RR^m))$. We consider the Volterra path $z\in\Vone$ defined by \eqref{eq:Volterra_path_def}.
    Then for every word $i_1\cdots i_n \in \mathcal{W}$  it holds
    \begin{equation}\label{eq:iterated Volterra path_coordinate}
            \VSig{x}{K}^{i_1\cdots i_n, \tau}_{s,t} = \int_{\Delta^n_{s,t}} \dd z_{r_1}^{i_1,r_2} \cdots   \dd z^{i_{n-1},r_n}_{r_{n-1}}  \dd z_{r_n}^{i_n, \tau}, \qquad (s,t, \tau)\in \Delta^{3},
    \end{equation}
    where we recall that $\dd z_t^\tau = \frac{\dd}{\dd t}z_t^\tau $. 
    In particular, for any $\widetilde{x} \in \mathcal{C}^{0,1}([0,T]; \RR^d)$ and $\widetilde{K} \in L^{\infty,1}(\Delta^2; \mathcal{L}(\RR^d;\RR^m))$, it holds $$\VSig{x}{K} \equiv \VSig{\widetilde{x}}{\widetilde{K}}$$ almost everywhere if and only if
    $$\int_{s}^t K(\tau, r) \dd{x_r} = \int_{s}^t \widetilde{K}(\tau, r) \dd{\widetilde{x}_r}, \quad \text{for almost every } (s,t,\tau) \in \Delta^{3}.$$
\end{prop}
\begin{proof}
    The first statement is a direct consequence of \eqref{eq:Volterra_path_def}. 
    It then follows from the representation \eqref{eq:iterated Volterra path_coordinate} that the Volterra signature $\VSig{x}{K}$ only depends on the Volterra path defined in the first level $z = (\VSig{x}{K}^{i})_{i\in\mathcal{A}} \in \Vone$.
    In other words, two Volterra signatures coincide if their first levels coincide in $\Vone$, which is precisely translates to the equivalence condition stated above.
    Note that the almost everywhere quantifier comes from the fact that $\Vone$ was also equipped with almost every everywhere equivalence.
\end{proof}

\begin{defn}\label{def:VP_def_z}
For a Volterra path $z\in\Vone$ we call $\bz = \VSig{x}{K}$ the \emph{full lift of $z$}, where $(x,K)$ is any representation of $z$, i.e., any $x\in\mathcal{C}^{0,1}([0,T];\RR^m)$  and $K\in\Lkernel$ such that
\eqref{eq:Volterra_path_def} holds.
\end{defn}

\subsubsection{Chen's relation}
Our next endeavor is to provide a Chen's relation for Volterra signatures.
The crucial ingredient to express this relation in a concise form is the availability of suitable convolution product for Volterra paths. In the accompanying work \cite{ii_part}, it is shown how to mitigate this lack of algebraic structure when performing numerical computations.

\begin{defn}\label{def:convolution_coordinate}
    Let $z$ be a Volterra path in $\Vone$ as given in \cref{def:volterra_path}.
    Then for any bounded measurable path $y:[0,T] \to \RR$ we define the \emph{$n$th order convolution} of $y$ with $z$ by
    \begin{equation}\label{eq:convol_def_words}
        \conv{y}{z}{i_1\cdots i_n}{\tau}{s}{t} = \int_{\Delta^{n}_{s,t}} y_{r_1}  \dd{z_{r_1}^{i_1,r_2}}  \cdots \dd{z^{i_{n-1}, r_n}_{r_{n-1}}}  \dd{z^{i_{n},\tau}_{r_{n}}}, \qquad (s,t,\tau) \in \Delta^3,
    \end{equation}
    for all $n\in\NN$ and $i_1\cdots i_n\in\mathcal{W}$.
    As a convention we also set \begin{equation}\label{eq:convol_def_empty}
    \conv{y}{z}{\varnothing}{\tau}{s}{t} = y_\tau, \qquad (s,t,\tau) \in \Delta^3.
    \end{equation}
\end{defn}
\begin{rem}\label{rem:full_lift}
    For any Volterra path $z\in\Vone$, recall that the full lift of $z$ is the proper intrinsic generalization of the Volterra signature from Definition \ref{def:vsig}. It can also be expressed from~\eqref{eq:convol_def_words} by setting \begin{equation}\label{eq:convol_def_full}\bz^{i_1\cdots i_n} = \conv{1}{z}{i_1\cdots i_n}{}{}{},\end{equation} for every word $i_1\cdots i_n \in \cW$.
\end{rem}
\begin{rem}
    While for the smooth case the convolution product  \eqref{eq:convol_def_full} is defined canonically from the underlying Volterra path, we caution that for the rough case this changes fundamentally (see \cite{HT21}). Indeed, in that case the convolution product for the first levels need to be provided as part of the definition of the Volterra rough path.
\end{rem}

Before stating Chen's relation, we first collect a few basic properties of the convolution product which will be used in the sequel.
\begin{lem} \label{lem:convolution_properties}
    Consider a Volterra path $z$ and a function $y$ as in Definition \ref{def:convolution_coordinate}. Then the following holds true 
    \begin{itemize}
        \item[(i)] For every word $i_1\cdots i_n \in \cW$, one has the recursive property \begin{equation}\label{eq:conv_tower_property}
        \int_{s}^t \conv{y}{z}{i_1\cdots i_{n-1}}{r}{s}{r} \dd{z^{i_{n},\tau}_r} = \conv{y}{z}{i_1\cdots i_{n}}{\tau}{s}{t}.
    \end{equation}
    \item[(ii)] Conversely, any $n$-th order convolution can be expressed in terms of the full signature $\bz$. Namely, for every word $i_1\cdots i_n \in \cW$ we have \begin{align}\label{eq:conv_ibp}
    \conv{y}{z}{i_1\cdots i_{n}}{\tau}{s}{t} 
    &= \int_s^t y_u  \conv{\frac{\dd}{\dd u}z^{i_1,\cdot}_u}{z}{i_2\cdots i_{n}}{\tau}{u}{t} \dd{u} \\ \nonumber
    &=- \int_s^t y_u  \left( \frac{\dd}{\dd u}\bz^{i_1\cdots i_{n}, \tau}_{u,t}\right) \dd{u}.
    \end{align}
    \end{itemize}
    
\end{lem}

\begin{rem}
    Note that both identities \eqref{eq:conv_tower_property} and \eqref{eq:conv_ibp} may serve to define the higher-order convolution inductively.
\end{rem}
\begin{rem}\label{rem:backward} We have seen in Lemma \ref{lem:equivalence_def_VP} that any Volterra path $z\in \Vone$ admits a representation of the form \eqref{eq:Volterra_path_def}. Owing to this representation we easily get $$\frac{\dd}{\dd \sigma}z^\tau_\sigma = K(\tau,\sigma)\dot{x}_\sigma.$$ Therefore using Fubini, expressions like the bracket in \eqref{eq:conv_ibp} can be conveniently written as
    \begin{align*}
        &\conv{\frac{\dd}{\dd\sigma}z^{i_1,\cdot}_\sigma}{z}{i_2\cdots i_{n}}{\tau}{s}{t} \\
        =& \int_s^t \int_{s}^{r_{n}} \cdots \int_s^{r_3} K^{i_1}(r_2, \sigma)\dot{x}_\sigma \left(\prod_{k=2}^{n-1}K^{i_k}(r_{k+1}, r_k)\dd{x_{r_k}} \right)K^{i_n}(\tau, r_n)\dd{x_{r_n}} \\
        =& \int_s^t \int_{r_2}^t \cdots \int_{r_{n-1}}^t  K^{i_n}(\tau, r_n)\dd{x_{r_n}}\left(\prod_{k=2}^{n-1}K^{i_k}(r_{k+1}, r_k)\dd{x_{r_k}} \right)K^{i_1}(r_2, \sigma)\dot{x}_\sigma,
    \end{align*}
    for all $(\sigma, s, t, \tau) \in \Delta^4$.
    These terms arise naturally for the case $t = \tau$ when performing Picard iteration of a linear Volterra equations starting from $\tau$ backwards in time. 
\end{rem}
\begin{proof}[Proof of Lemma~\ref{lem:convolution_properties}]
    The first identity \eqref{eq:conv_tower_property} %
     follows directly from \Cref{def:convolution_coordinate}.
   Hence we will focus our efforts on the proof of \eqref{eq:conv_ibp}. To this aim we first consider the case $n=1$, for which one can start again from \Cref{def:convolution_coordinate}. This enables to write
    $$\conv{y}{z}{i}{\tau}{s}{t} = \int_s^t y_r \dd{z}^{i,\tau}_r = \int_s^t y_r \frac{\dd}{\dd r} z^{i, \tau}_r \dd{r} = \int_s^t y_r \conv{\frac{\dd}{\dd r} z^{i, \cdot}_r}{z}{\varnothing}{\tau}{s}{t}\dd{r},$$ where we have invoked the convention \eqref{eq:convol_def_empty}.
   The case $n\ge 2$  in \eqref{eq:conv_ibp} is more conveniently handled by resorting to a representation of $z$ as in relation \eqref{eq:Volterra_path_def}. Then recalling (from Definition \ref{def:vsig}) that $K^i(\tau,\cdot) \dd x$ denotes the component $i$ in the vector $K(\tau,\cdot)\dd x$ and applying Fubini's theorem we get 
    \begin{align*}
        \conv{y}{z}{i_1\cdots i_{n}}{\tau}{s}{t} =& \int_s^t  \int_{s}^{r_{n}}\cdots \int_{s}^{r_2} y_{r_1} \prod_{l=1}^{n}K^{i_l}(r_{l+1}, r_l)\dot{x}_{r_l}\dd{r_l} \\
        =& \int_s^t  \int_{r_1}^t\cdots \int_{r_n}^ t y_{r_1} \prod_{l=1}^{n}K^{i_l}(r_{l+1}, r_l)\dot{x}_{r_l}\dd{r_l}.
    \end{align*}
    Thus, applying Fubini again we obtain \begin{equation}\label{eq:fubini_2}
        \begin{aligned}
            [y \ast z]_{s,t}^{i_1\cdots i_n,\tau} = & \int_s^t y_{r_1} \left(\int_{r_1}^t\cdots \int_{r_n}^ t   K^{i_1}(r_{2}, r_1)\dot{x}_{r_1}\prod_{l=2}^{n}K^{i_l}(r_{l+1}, r_i)\dot{x}_{r_l}\dd{r_l} \right) \dd{r_1}\\
        =& \int_s^t y_{r_1} \left(\int_{r_1}^t\int_{r_1}^{r_n}\cdots \int_{r_1}^{r_3}    \frac{\dd}{\dd r_1}z^{i_1,r_2}_{r_1} \prod_{l=2}^{n}K^{i_l}(r_{l+1}, r_l)\dot{x}_{r_l}\dd{r_l} \right) \dd{r_1}\\
        =&\int_s^t y_u  \conv{\frac{\dd}{\dd u}z^{i_1,\cdot}_u}{z}{i_2\cdots i_{n}}{\tau}{u}{t} \dd{u}.
        \end{aligned}
    \end{equation} This proved the first identity in \eqref{eq:conv_ibp}. For the second identity in \eqref{eq:conv_ibp}, it is now enough to apply the above identity to the case $y=1$. This yields, for $u<t$
    \begin{align*}
        \bz^{i_1 \cdots i_n, \tau}_{u,t} = \conv{1}{z}{i_1\cdots i_{n}}{\tau}{u}{t} =&\int_u^t  \conv{\frac{\dd}{\dd v}z^{i_1, \cdot }_v}{z}{i_2\cdots i_{n}}{\tau}{v}{t} \dd{v}.
    \end{align*} 
    It is readily checked from the above relations that \[ \frac{\dd }{\dd u}\bz_{u,t}^{i_1\cdots i_n,\tau} = - \Big [ \frac{\dd }{\dd u}z_u^{i_1,\cdot} \ast z\Big ]_{u,t}^{i_1\cdots i_n,\tau}.
    \]
    Plugging this relation into \eqref{eq:fubini_2}, the second part of  \eqref{eq:conv_ibp} is now achieved. This finishes our proof.
\end{proof}
Combining our definitions of convolution and Lemma \ref{lem:convolution_properties}, we can now state Chen's identity for smooth Volterra signatures.

\begin{prop}\label{prop:Chen_coordiante}
Let $\bz$ be the full lift of a Volterra path $z \in \Vone$, defined as in Remark \ref{rem:full_lift}. Then $\bz$ satisfies
\begin{equation}\label{eq:chen_coordinate}
    \bz_{s,t}^{i_1\cdots i_n, \tau} = \sum_{k=0}^n \conv{\bz^{i_1\cdots i_{k} \cdot}_{s,u}}{z}{ i_{k+1}\cdots i_n}{\tau}{u}{t},\qquad \;(s,u,t,\tau) \in \Delta^4.
\end{equation}
for all $n\in\NN$ and $i_1\cdots i_n \in \cW$, where we recall that the convolution product above is defined by~\eqref{eq:convol_def_words}.
\end{prop}

In the spirit of presentation of the remaining paper, we will postpone the proof of Proposition \ref{prop:Chen_coordiante} and present it below in \Cref{cor:Chen} based on a dynamic representation of the Volterra signature.
This is the theme of the next section.

\subsection{A fundamental linear Volterra equation}\label{sec:fundamental}

We recall from Section \ref{preliminaries} that $\mathcal{W}$ denotes the set of words in the alphabet $\mathcal{A} = \{1, \dots, m\}$ and $e_{i_1\cdots i_n} := e_{i_1} \otimes \cdots \otimes e_{i_n} \in (\RR^m)^{\otimes n}$ denotes the basis tensor of level $n$ corresponding to the coordinate $i_1\cdots i_n\in\mathcal{W}$.
To give a dynamic representation of the Volterra signature, we henceforth represent the collection of its iterated integrals by a tensor series.

\begin{nota}\label{b1}
Let $x$ be a signal in $\mathcal{C}^{0,1}([0,T]; \RR^d)$, and consider a kernel $K$ with values in $\cL(\RR^{d};\RR^{m})$ that satisfies Hypothesis~\ref{def:kernel_class}. Then the iterated integrals $\{ \VSig{x}{K}^w\}_{w\in\mathcal{W}}$ characterizing $\mathrm{VSig}$ in Definition~\ref{def:vsig} can be represented by a formal tensor series, i.e.,
\begin{align*}
    \VSig{x}{K}^{\tau}_{s,t} := \sum_{n=0}^\infty \VSig{x}{K}^{(n),\tau}_{s,t}  := \sum_{n=0}^\infty \sum_{\substack{w\in\mathcal{W}\\|w| =n}}\VSig{x}{K}^{w,\tau}_{s,t}  e_w \quad \in T((\RR^m)),
\end{align*}
for all $(s,t,\tau)\in\Delta^3$, where we recall that $T((\RR^m))$ is introduced in \Cref{def:extended_TA}.
\end{nota}

Notation~\ref{b1} allows an easy characterization of 
the Volterra signature by a linear Volterra equation in the extended tensor algebra, which is manifested by the following
\begin{prop}\label{prop:fundamental_equation}
Under the same conditions as in Notation~\ref{b1}, the Volterra signature $\bz = \VSig{x}{K}: \Delta^3 \to T((\RR^m))$ uniquely solves a linear system of the form
\begin{align}\label{eq:vsig_fundamental}
\bz^{\tau}_{s,t} &= 1 +\int_s^t \bz^{u}_{s,u} \otimes K(\tau,u) \dd{x_u}, \qquad (s,t,\tau)\in \Delta^3 \, ,
\end{align}
where we have used the representation~\eqref{eq:Volterra_path_def} of $z$.
\end{prop}
\begin{rem}
    Using once more \eqref{eq:diff_volterra_path}, i.e.~$\dd{z^\tau_t} = \frac{\dd z^\tau_{t}}{\dd t}\dd{t}$, we can equivalently present the equation \eqref{eq:vsig_fundamental} using the underlying Volterra path $z^{\tau}_{s,t} = \int_s^t K(\tau, u)\dd{x_u}$ by
    \begin{align}\label{eq:vsig_fundamental_z}
\bz^{\tau}_{s,t} &= 1 +\int_s^t \bz^{u}_{s,u} \otimes\dd{z^\tau_u}, \qquad (s,t,\tau)\in \Delta^3.
\end{align}
\end{rem}

\begin{rem}\label{rem:classical_case_dynamic}
Drawing on Remark \ref{rem:usual_sig}, the linear equation for the usual signature $\widetilde{\bz}= \mathrm{Sig}(x)$ can be written as \begin{equation}\label{eq:classical_sig_eq}
    \widetilde{\bz}_{s,t} = 1 + \int_s^t \widetilde{\bz}_{s,u} \otimes \dd x_u, \quad (s,t) \in \Delta^2.
\end{equation}
That is the convolution aspect, apparent in  \eqref{eq:vsig_fundamental}, is absent here. As we will see below, this is a substantial computational advantage.
    \end{rem}

\begin{proof}[Proof of Proposition~\ref{prop:fundamental_equation}]
    Equation \eqref{eq:vsig_fundamental}  is given meaning to -- and in fact is solved by -- projection to coordinates:
\begin{align}\label{eq:fundamental_coord}
\bz^{wi,\tau}_{s,t} = & \int_s ^t \bz^{w,u}_{s,u}  K^{i}(\tau,u) \dd{x_u}, 
\quad\text{for all}\quad 
w\in\mathcal{W},\; i\in\mathcal{A},\; (s,t,\tau)\in \Delta^3.
\end{align}
    Indeed, starting with $w=\varnothing$, with $\bz^\varnothing \equiv 1$,  we observe inductively over the length of words, that $\bz^w$ is measurable and bounded over $\Delta^3$, hence, all integrals are well defined.
    This already fully determines the solution and hence uniqueness is settled.
    One then readily verifies form the \Cref{def:vsig} that with $w=j_1\cdots j_n$ it holds
\begin{multline}\label{eq:fundamental_coord_2}
    \int_s ^t \VSig{x}{K}^{w,u}_{s,u}  K^{i}(\tau,u) \dd{x_u} 
    = \int_s ^t \int_{\Delta^n_{s,u}} \left(\prod_{l=1}^n K^{j_l}(r_{l+1},r_l)\dd{x_{r_l}} \right) K^{i}(\tau,u) \dd{x_u} \\
    = \int_{\Delta^{m+1}_{s,t}} \left(\prod_{l=1}^n K^{j_l}(r_{l+1},r_l)\dd{x_{r_l}} \right) K^{i}(\tau,u) \dd{x_u}
    = \VSig{x}{K}^{wi,u}_{s,t}.
\end{multline}
Hence, $\bz = \VSig{x}{K}$ solves equation~\eqref{eq:vsig_fundamental}.
\end{proof}

To express the properties of the Volterra signatures more conveniently in the tensor setting, we formally extend the convolution operation $\ast$ from \Cref{def:convolution_coordinate} for both $y$ and $z$. We obtain an operation denoted by $\oast$, acting on paths in $T((\RR^m))$.

\begin{defn}\label{def:b2}
Let $\bz:\Delta^3\to T((\RR^m))$ be the full lift of a Volterra path $z\in\Vone$.
Then for any $\by: [0,T] \to T((\RR^m))$ such that $\by^{w}$ is measurable and bounded for all $w\in\mathcal{W}$ we define for all $n,k\in\NN$:
    \begin{equation}\label{eq:tensor_convol_def}
    \oconv{\by^{(n)}}{\bz}{(k)}{\tau}{s}{t} 
     := \sum_{\substack{w, v\in\mathcal{W}, \\|v| =n,|w|=k}}e_v\otimes e_w \conv{\by^v}{z}{w}{\tau}{s}{t}, \qquad (s,t,\tau) \in \Delta^3,
    \end{equation} with $\conv{\by^v}{z}{w}{\tau}{s}{t}$ still defined by \eqref{eq:convol_def_words},
    and further 
    \begin{equation}\label{eq:full_tensor_convol_def}
    \oconv{\by}{\bz}{}{\tau}{s}{t}
     :=\sum_{n=0}^\infty \sum_{k=0}^n \oconv{\by^{(n-k)}}{\bz}{(k)}{\tau}{s}{t}, \qquad (s,t,\tau) \in \Delta^3.
    \end{equation}
    We also denote by $\by\oast\bz$ the function from $\Delta^3$ to $T((\RR^m))$ given by $(s,t,\tau)\mapsto\oconv{\by}{\bz}{}{\tau}{s}{t}$.
\end{defn}

As a first demonstration of the utility of a compact notation like $\oast$, we will translate the inductive properties of~\Cref{lem:convolution_properties} into associative properties for the convolution product in $T((\RR^m))$. The proof is immediate and therefore omitted.

\begin{lem}\label{rem:convolution_properties_tensor}
    Let $\bz$ and $\by$ be as in Definition~\ref{def:b2}. Then
    \begin{equation}\label{eq:tensor_convol_iterate_0}
        \oconv{\by}{\bz}{(0)}{\tau}{s}{t} = \oconv{\by}{1}{}{\tau}{s}{t} =  \by_\tau, \qquad (s,t,\tau) \in \Delta^3.
    \end{equation}
    Furthermore, for $k,n\in\NN$:
\begin{equation}\label{eq:tensor_convol_iterate}\oconv{\oconv{\by}{\bz}{(k)}{\cdot}{s}{\cdot}}{\bz}{(n)}{\tau}{s}{t} = \oconv{\by}{\bz}{(n+k)}{\tau}{s}{t}, \qquad (s,t,\tau) \in \Delta^3. \end{equation}
    In particular, the linear equation \eqref{eq:vsig_fundamental_z} for $\bz$ can be recast as 
    \begin{align*}
        \bz^{\tau}_{s,t} = 1 + \oconv{\bz^{\cdot}_{s,\cdot}}{z}{}{\tau}{s}{t}, \qquad (s,t,\tau)\in\Delta^3.
    \end{align*}
\end{lem}

\noindent

The notion of convolutional tensor product in Definition \ref{def:b2} also allows to express Chen's relation from Proposition \ref{eq:chen_coordinate} in a more elegant way. As promised above, we will prove this relation using the curent tensor product setting.
\begin{cor}\label{cor:Chen} Let $\bz$ be the full lift of a Volterra path $z\in\Vone$ considered as a function $\bz:\Delta^3\to T((\RR^m))$. Then $\bz$ satisfies
\begin{equation}\label{eq:chen_tensor}
    \bz^{\tau}_{s,t} = \oconv{\bz^{\cdot}_{s,u}}{\bz}{}{\tau}{u}{t} , \qquad (s,u,t,\tau) \in \Delta^4.
\end{equation}
\end{cor}
\begin{proof}[Proof of \Cref{prop:Chen_coordiante} and \Cref{cor:Chen}]
    One first notices that \eqref{eq:chen_coordinate} is precisely the coordinate wise projection of \eqref{eq:chen_tensor}, hence the two identities are equivalent and it suffices to prove \eqref{eq:chen_tensor}.
    
    In order to prove \eqref{eq:chen_tensor}, let us start from equation \eqref{eq:vsig_fundamental} which characterizes the signature $\bz$, for $(s,t,\tau) \in \Delta^3$. Then we further split the interval $[s,t]$ into $[s,u] \cup [u,t]$, which yields 
\begin{equation}\label{eq:chen_decomp_11}
\begin{aligned}\bz^{\tau}_{s,t} &= 1 +\int_s ^t \bz^{r}_{s,r} \otimes \dd{z^\tau_r} = 1 +\int_s^u \bz^{r}_{s,r} \otimes \dd{z^\tau_r} + \int_u^t \bz^{r}_{s,r} \otimes \dd{z^\tau_r} .
    \end{aligned}
\end{equation}
On the right hand side of \eqref{eq:chen_decomp_11} we then recognize the expression for $\bz_{s,u}^\tau$ from \eqref{eq:vsig_fundamental}. We thus get \[
\bz^{\tau}_{s,t}=\bz^{\tau}_{s,u} + \int_u^t \bz^{r}_{s,r} \otimes \dd{z^\tau_r}.
\]
Now fix $(s,u) \in \Delta^2$ arbitrarily.
It follows from the above that $\by^\tau_t := \bz^\tau_{s,t}$ for $(t,\tau) \in \Delta^2_{u,T}$ solves the equation
\begin{align}\label{eq:chen_proof}
\by^\tau_t &= \bz^{\tau}_{s,u} +\int_u ^t \by^r_r \otimes \dd{z^\tau_r}, \qquad (t,\tau) \in \Delta^2_{u,T}.
\end{align}
Since the above equation can be solved by inductively projecting to tensor levels as in the proof of \Cref{prop:fundamental_equation}, it is clear that it admits a unique solution.
Hence, the claim now follows after verifying that this is equation is also solved by 
$\widetilde{\by}^\tau_t := \oconv{\bz^{\cdot}_{s,u}}{\bz}{}{\tau}{u}{t}$ for $(t,\tau) \in \Delta^2_{u,T}$.
To this end we plug $\widetilde{\by}$ into the right-hand side of \eqref{eq:chen_proof}, and invoke both definitions \eqref{eq:tensor_convol_def} and \eqref{eq:full_tensor_convol_def} to write 
\[
\bz^{\tau}_{s,u} +\int_u^t \widetilde{\by}_r^r\otimes \dd{z^\tau_r} = \bz^{\tau}_{s,u} +  \sum_{n=0}^\infty \sum_{k=0}^{n} \oconv{\oconv{\bz^{(n-k),\cdot}_{s,u}}{\bz}{(k)}{\cdot}{u}{\cdot}}{\bz}{(1)}{\tau}{u}{t}{}. 
\]
Now resort to \eqref{eq:tensor_convol_iterate} and then \eqref{eq:tensor_convol_iterate_0} to get
 \begin{align*}
\bz^{\tau}_{s,u} +\int_u^t \widetilde{\by}_r^r\otimes \dd{z^\tau_r}
&=  \bz^{\tau}_{s,u} +  \sum_{n=0}^\infty \sum_{k=0}^{n}\oconv{\bz^{(n-k),\cdot}_{s,u}}{\bz}{(k+1)}{\tau}{u}{t}  \\
&=  \sum_{n=0}^\infty \oconv{\bz^{(n),\cdot}_{s,u}}{\bz}{(0)}{\tau}{u}{t} +  \sum_{n=1}^\infty \sum_{k=1}^{n} \oconv{\bz^{(n-k),\cdot}_{s,u}}{\bz}{(k)}{\tau}{u}{t}.
\end{align*}
Gathering the two terms in the right hand side above and recalling our definition \eqref{eq:full_tensor_convol_def} again we end up with
\[\bz^{\tau}_{s,u} +\int_u^t \widetilde{\by}_r^r\otimes \dd{z^\tau_r}=
 \sum_{n=0}^\infty \sum_{k=0}^{n
}\oconv{\bz^{(n-k),\cdot}_{s,u}}{\bz}{(k)}{\tau}{u}{t}   =\bigg. \oconv{\bz^{\cdot}_{s,u}}{\bz}{}{\tau}{u}{t}.
\]
We have thus shown that $\widetilde{\by}$ solves \eqref{eq:chen_proof}, which proves that $\widetilde{\by} = {\by}$ and hence $\bz^{\tau}_{s,t} = \oconv{\bz^{\cdot}_{s,u}}{\bz}{}{\tau}{u}{t} $. 
\end{proof}

We end this section with a result that presents how linear Volterra equations are expanded in terms of the Volterra signature, thus justifying, in particular, the discussion from the introduction to this subsection and the choice of terminology ``fundamental linear Volterra equation''. To ensure global existence of such expansions, we require the Volterra signature to have an \emph{infinite radius of convergence}, by which we mean \begin{equation}
    \label{eq:ROC_Vsig}
    \sup_{(s,t,\tau) \in \Delta^{3}}\sum_{n=0}^\infty \lambda^n |\VSig{x}{K}_{s, t}^{(n),\tau}| <\infty, \qquad \lambda > 0.
\end{equation}
The following lemma shows that condition \eqref{eq:ROC_Vsig} is fulfilled under mild assumptions on $K$.
\begin{lem}\label{lem:ROC_examples}
Let $K:\Delta^2 \rightarrow \cL(\RR^d;\RR^m)$ be a matrix-valued kernel. Assume that $K$ satisfies an improved version of \cref{def:kernel_class}, that is suppose $K$ is in  $L^{\infty,p}(\Delta^2;\cL(\RR^d;\RR^m))$ for some $p>1$. Otherwise stated, we assume \begin{equation}\label{eq:Lp_kernel}
\Vert K \Vert_{L^{\infty,p}} := \sup_{t \in [0,T]}\left ( \int_0^
t|K(t,s)|^p\dd s \right )^{1/p}< \infty.
\end{equation} Then condition \eqref{eq:ROC_Vsig} is fulfilled. In particular, $\VSig{x}{K}$ can be considered as an element of $\cT^2$ (see \Cref{def:inner_pro}).
\end{lem}
\begin{proof}
    Defining the scalar kernel $k=|K|$, for any $x\in \cC^{0,1}([0,T];\RR^d)$ and $(s,t) \in \Delta^2$ we have \[
    |\VSig{x}{K}_{s, t}^{(n),\tau}| \leq \int_{\Delta_{s,t}^n}\prod_{l=1}^n|K(r_{l+1},r_l)| | \dot{x}_{r_l} |\dd r_l \leq \Vert \dot{x}\Vert_{\infty}^n \int_{\Delta_{s,t}^n}\prod_{l=1}^nk(r_{l+1},r_l)\dd r_l.
    \] 
    Moreover, an inductive argument shows that $\int_{\Delta_{s,t}^n}\prod_{l=1}^nk(r_{l+1},r_l)^p\dd r_l \leq \Vert k \Vert_{L^{\infty,p}}^{np},$ so that by the Hölder inequality we have \[
    \int_{\Delta_{s,t}^n}\prod_{l=1}^nk(r_{l+1},r_l)\dd r_l \leq \Vert k \Vert_{L^{\infty,p}}^{n} \mathrm{Vol}(\Delta_{s,t}^n)^{1/q}, \quad q = p/(p-1) >0.
    \]
    Using $\mathrm{Vol}(\Delta_{s,t}^n) = \frac{(t-s)^n}{n!}$, we can conclude that $$\sup_{(s,t,\tau)\in \Delta^3}\sum_{n=0}^\infty \lambda^n |\VSig{x}{K}_{s, t}^{(n),\tau}| \leq \sum_{n\geq 0} \frac{z^n}{(n!)^{1/q}}<\infty, \quad z = \lambda \Vert \dot{x}\Vert_{\infty} \Vert k \Vert_{L^{\infty,p}} T^{1/q},$$ so that \eqref{eq:ROC_Vsig} indeed holds true.
\end{proof}
We are now ready to state our representation of linear equations in terms of Volterra signatures.
\begin{prop}\label{prop:linear_VCDE}
Let $\xi \in \RR^k$, $x\in \cC^{0,1}([0,T];\RR^{d})$ and $K \in L^{\infty,p}(\Delta^2;\cL(\RR^d;\RR^m))$ for some $p>1$. For $A\in \mathcal{L}(\RR^m; \mathcal{L}(\RR^k; \RR^k))$, we consider the linear Volterra equation \begin{equation}\label{eq:volterra}
y_{t}
=
\xi \;+\;\int_{t_0}^{t} A\big(K(t,u)\,\dot x_u\big)\,y_{u} \dd u, \qquad t\in[t_0,T].
\end{equation} Also recall our convention \eqref{eq:linear_tensor_maps} for linear maps acting on tensors, specialized here to $E= \RR^k$. Then it holds that  
\begin{equation}\label{eq:psi_eq}
\Delta^3 \ni (s,t,\tau) \mapsto \Psi_{s,t}^\tau = \sum_{n=0}^\infty\widetilde A\big(\VSig{x}{K}^{(n),\tau}_{s,t}\big)
\end{equation}
lies in $L^{\infty}(\Delta^3;\mathcal{L}(\RR^k;\RR^k))$.
Moreover, for any initial condition $\xi \in \RR^k$, the path $y_t =\Psi_{t_0,t}^t (\xi)$ is the unique, bounded and measurable solution to \eqref{eq:volterra}.

\end{prop}
\begin{proof}
 By definition of $\tilde{A}$ and \cref{lem:ROC_examples}, it readily follows that \[
 \sup_{(s,t,\tau) \in \Delta^3} \sum_{n\geq 0} |\widetilde A\big(\VSig{x}{K}^{(n),\tau}_{s,t}\big)| \leq \sup_{(s,t,\tau) \in \Delta^3} \sum_{n\geq 0} |A|^n |\VSig{x}{K}^{(n),\tau}_{s,t}|< \infty.
 \]
 Therefore the Volterra increment $\Psi$, as introduced in \eqref{eq:psi_eq}, is analytically well defined. In addition, as a direct consequence of the definition \eqref{eq:linear_tensor_maps} of $\widetilde{A}$, it holds that $A(v)\widetilde{A}(v_1 \otimes \cdots \otimes v_n) = \widetilde{A}(v_1 \otimes \cdots \otimes v_n \otimes v)$ for all $v_1,\dots,v_n,v \in \RR^m$. Thus we get \begin{align*}
     \int_{t_0}^tA(K(t,u)\dot{x}_u)\Psi_{t_0,u}^u(\xi)\dd u & = \sum_{n\geq 0} \int_{t_0}^t\widetilde{A}\big (\VSig{x}{K}_{t_0,u}^{(n),u} \otimes K(t,u) \dd x_u\big )(\xi) 
 \end{align*}
 Therefore resorting to \eqref{eq:tensor_convol_iterate} we find
 \[\int_{t_0}^tA(K(t,u)\dot{x}_u)\Psi_{t_0,u}^u(\xi)\dd u
=\sum_{n\geq 1} \widetilde{A}(\VSig{x}{K}_{t_0,t}^{(n),t})(\xi)  = \Psi_{t_0,t}^t(\xi)-\xi,
 \] where the second identity stems from \eqref{eq:psi_eq}. We have thus proved that  $y_t = \Psi_{t_0,t}^{t}(\xi)$ indeed solves \eqref{eq:volterra}. Finally, uniqueness follows from a Gronwall inequality for Volterra equations: Assume $y$ and $\hat{y}$ are bounded and measurable solutions to \eqref{eq:volterra}, and define $f(t)=|y_t-\hat{y}_t|$. Then by linearity we have \[
 f(t) \leq |A | \Vert x \Vert_{\cC^{0,1}}\int_{t_0}^t |y_u-\hat{y}_u||K(t,u)|\dd u = C \int_{t_0}^tf(u)|K(t,u)| \dd u.
 \] An application of the Volterra-Gronwall Lemma \cite[Lemma 9.8.2]{gripenberg1990volterra} shows that $f=0$ a.e., which concludes the proof.
\end{proof}

\subsection{Finite state space representation}\label{sec:dyn_exp}
As mentioned in the introduction, Volterra signatures are useful objects when one wishes to model systems with memory. However, one common pitfall of all representations for this type of systems is their computational cost. In this section we focus on a class of exponential kernels which encompasses relevant examples and leads to simplified computations. This class is defined below.
\begin{defn}\label{def:finite_state_space_kernels}
    Consider a constant matrix $\Lambda \in \RR^{R\times R}$ and a family of coefficients $\{A_r,b_r:1\leq r \leq q\}$ with $A_r \in \RR^{m\times d}$ and $b_r \in \RR^{R}$. Denote by $\mathbf{1}^\top$ the $\RR^{1,R}$ row vector $(1,\dots,1)$. For $(s,t) \in \Delta^2$ we define \begin{equation}\label{eq:exponential_kernel}
K_{A,b}^{\Lambda}(t,s) = \sum_{r=1}^q (\mathbf{1}^{\top} e^{-\Lambda(t-s)}b_r) A_r \in \cL(\RR^d;\RR^m).
\end{equation} 
\end{defn}
\begin{rem}\label{rem:exp_remark}
    Notice that for diagonal matrices $\Lambda$, Definition \ref{def:finite_state_space_kernels}  includes, in particular, sums of exponential kernels like $k(t,s) = \sum_{i=1}^R\alpha_ie^{-\lambda_i(t-s)}$. Such multi-factor kernels are widely used to model processes with short-term or fading memory, 
as the exponential decay naturally captures the rapid loss of dependence on past inputs. 
Moreover, completely monotone kernels - such as the fractional kernels %
 - can be arbitrarily well approximated by finite sums of exponential factors. For these reasons, Volterra systems driven by (sums of) exponential kernels appear in various applications, 
most prominently in path-dependent volatility modeling \cite{abi2019multifactor,abi2025volatility,bayer2023markovian,guyon2023volatility}, 
and therefore it is natural to study the Volterra signature induced by this type of kernel matrices.
\end{rem}
\begin{rem}\label{rem:prony_kernels}
Beyond (sum of) exponential kernels, the class \eqref{eq:exponential_kernel} also includes mixed exponential-polynomial kernels
\begin{equation}\label{c1}
K(t,s) = \sum_{r=1}^Q \sum_{l=1}^{m_r} e^{-\lambda_r(t-s)}\frac{(t-s)^{l-1}}{(l-1)!}A_{r,l}, \quad A_{r,l} \in \mathbb{R}^{m\times d}, \quad m_i \in \mathbb{N}.
\end{equation}
More generally, it also covers damped periodic (Prony) kernels obtained by allowing complex conjugate pairs of exponents, which lead to factors of the form $e^{-a(t-s)}\cos(\omega(t-s))$ and $e^{-a(t-s)}\sin(\omega(t-s))$, possibly multiplied by polynomials in $(t-s)$. More details on such representations, as well as their implications regarding computational aspects of Volterra signatures, are discussed in the accompanying article \cite{ii_part}.
\end{rem}

The main observation of this section is that Volterra signatures with kernels restricted to the class \eqref{eq:exponential_kernel} can be realized as the solution to a system of \emph{mean-reverting equations} in the tensor algebra. As motivated in the introduction, this observation already suggests the possibility of computing
\(\VSig{x}{K_{A,b}^{\Lambda}}\) with \emph{linear} complexity, much like in the classical case
\(K=I_d\). Indeed, the dynamics of the
lifted Volterra signature reduce to a classical controlled differential equation in the
tensor algebra, in contrast to the general case~\eqref{eq:vsig_fundamental} involving convolutions. %

\begin{prop}\label{prop:mean_reverting_prop}
    Let $x \in \cC^{0,1}([0,T];\RR^d)$ and consider a kernel $K_{A,b}^\Lambda$ as in Definition \ref{def:kernel_class}. Then the Volterra signature related to $K_{A,b}^\Lambda$ can be decomposed as  \[
    \VSig{x}{K_{A,b}^{\Lambda}}_{s,t}^t = 1+ \sum_{\ell =1}^R \bZ_{s,t}^{\ell},
    \] where $\bZ_{s,t} \in T((\RR^m))^R$ solves the  system of ordinary differential equations \begin{equation}\label{eq:mean_reverting}
    \bZ_{s,s}= \mathbf{0}, \quad \dd \bZ_{s,t} = -\Lambda. \bZ_{s,t}\dd t + \left (1+\sum_{i=1}^R\bZ^i_{s,t} \right)\otimes \dd (B.x_t), 
    \end{equation} and where the vectors $\Lambda. \bZ_{s,t}$, $B.x_t$ are respectively defined by \begin{equation}\label{eq:vector_prod_lifts}
    (\Lambda. \bZ_{s,t})_{\ell} := \sum_{k=1}^R\Lambda_{\ell k}\bZ_{s,t}^k \in T((\RR^m)), 
    \quad\text{and}\quad 
    (B.x_t)_\ell := \sum_{r=1}^q b_r^\ell A_rx_t \in \RR^m.
    \end{equation}  
    \end{prop}

\begin{rem}
Let $m=d$, $q=1$ and $ A_{1} = I_d$. Choosing $\Lambda = \mathrm{diag}(\lambda_1,\dots,\lambda_R)$ and $b_1= (\alpha_1,\dots,\alpha_R)^{\top},$ for some $\lambda_i,\alpha_i \in \mathbb{R}_+$, we have \begin{equation}\label{eq:diagonal_exp}
    K_{A,b}^{\Lambda}(t,s) := \sum_{\ell=1}^R\alpha_\ell e^{-\lambda_\ell(t-s)}I_{d}
\end{equation} in \eqref{eq:exponential_kernel}, that is the standard sum of exponentials kernel. Projected to the first level $\pi_1(T((\RR^d))^R) \cong \RR^{d \times R}$, the dynamics of $Z_{s,t} = \pi_1(\bZ_{s,t})$ in \cref{eq:mean_reverting} read
\begin{equation}\label{eq:OU_repre}
dZ^\ell_{s,t} = - \lambda_\ell Z^\ell_{s,t} \dd t + \alpha_\ell \dd x_t, \quad (s,t) \in \Delta^2, \quad 1 \leq \ell \leq R.
\end{equation}
Note that the representation \eqref{eq:OU_repre} corresponds to the multivariate Ornstein–Uhlenbeck dynamics, acting as a \emph{Markovian lift} above the Volterra path $\int_0^tK(t,s)\dd x_s$. In the same way, we may view $\bZ$ in \eqref{eq:mean_reverting} acting as a lift above $\mathrm{VSig}$ for exponential kernels, which leads to computational benefits as illustrated in  \cite{ii_part}.
\end{rem}

\begin{proof}[Proof of Proposition \ref{prop:mean_reverting_prop}]
We start by expressing the equation for $\bz = \VSig{x}{K}$ and $K=K_{A,b}^\Lambda$ in convolutional form thanks to Proposition \ref{prop:fundamental_equation}. That is plugging the expression~\eqref{eq:exponential_kernel} for $K$ into equation \eqref{eq:vsig_fundamental}, we get that $\VSig{x}{K}$ satisfies
\[
\VSig{x}{K}_{s,t}^t- 1  = \int_s^t\VSig{x}{K}_{s,u}^{u} \otimes \left (\sum_{r=1}^q (\mathbf{1}^{\top} e^{-\Lambda(t-u)}b_r) A_r \right)\dd x_u. 
\]
Expanding part of the matrix products in coordinates, we thus obtain 
\begin{equation}\label{eq:proof_MR_1}
    \VSig{x}{K}_{s,t}^t- 1 = \sum_{i=1}^R {\bZ}_{s,t}^{i}, 
\end{equation}
where the increments ${\bZ}_{s,t}^{i}$ are defined by
\begin{equation}\label{c2}
\bZ^i_{s,t} = \sum_{r =1}^q \sum_{j=1}^R \int_s^t (e^{-\Lambda(t-u)})_{ij}b_r^j\VSig{x}{K}_{s,u}^{u} \otimes \dd (A_{r} x_u).
\end{equation}
Let us now give more information about the dynamics \eqref{eq:proof_MR_1}. First we clearly have ${\bZ}^{\ell}_{s,s}= 0$. Next an application of the Leibniz rule (with respect to $t$) to expression \eqref{c2} shows that 
\begin{multline}\label{eq:MR_proof_2}
    \dd {\bZ}_{s,t}^\ell  
    = \sum_{r =1}^q \sum_{j=1}^R\left [\int_s^t \frac{\dd }{ \dd t} \left (e^{-\Lambda(t-u)}\right)_{\ell j}b_r^j\VSig{x}{K}_{s,u}^u \otimes \dd (A_{r} x_u) \right ]\dd t \\ 
    + \sum_{r=1}^q \sum_{j=1}^R 1_{\{j= \ell \}}b_r^j\VSig{x}{K}_{s,t}^{t} \otimes \dd (A_{r} x_u),
\end{multline} where we have used the fact that $e^{-\Lambda t}|_{t=0}=\mathrm{Id}_R$ to get the term $1_{\{j=\ell\}}$. In addition, from basic matrix exponential calculus \cite{higham2008functions} we know that \[
\frac{\dd}{\dd t}(e^{-\Lambda(t-s)})_{ij}=(-\Lambda e^{-\Lambda(t-s)})_{ij} = -\sum_{k=1}^R \Lambda_{ik}(e^{-\Lambda(t-s)})_{kj}.
\]  Reporting this idendity into \eqref{eq:MR_proof_2} we obtain 
\begin{multline}\label{eq:proof_MR_3}
    \dd {\bZ}_{s,t}^\ell  =  
    -\sum_{k=1}^R\Lambda_{\ell k}  \sum_{r =1}^q \sum_{j=1}^R \left [\int_s^t(e^{-\Lambda(t-s)})_{kj}  b_r^j\VSig{x}{K}_{s,u}^u \otimes \dd (A_{r} x_u) \right ]\dd t \\ 
     + \VSig{x}{K}_{s,t}^{t} \otimes \left ( \sum_{r=1}^q b_r^{\ell}A_{r} \right ) \dd x_t.
\end{multline}
We now easily recognize the relation \eqref{c2} defining the coordinates $\bZ_{s,t}^i$ (of $\VSig{x}{K}_{s,t}^t-1)$ in the right hand side of \eqref{eq:proof_MR_3}. Thus we end up with
\begin{equation}\label{eq:proof_MR_4}
     \dd {\bZ}_{s,t}^\ell = - \sum_{k=1}^R\Lambda_{\ell k} \bZ_{s,t}^{k} \dd t + \left (1 + \sum_{i=1}^R \bZ_{s,t}^i \right )\otimes  \left ( \sum_{r=1}^q b_r^{\ell}A_{r} \right ) \dd x_t,
\end{equation}
for all $\ell=1,\dots,R$. Recalling our notation \eqref{eq:vector_prod_lifts}, this concludes the proof of our claim~\eqref{eq:mean_reverting}.
\end{proof}

As mentioned above, the advantage of equation \eqref{eq:mean_reverting} is that one obtains an Ornstein-Uhlenbeck type system which is easy to simulate. On top of that, in case $K$ is a scalar-valued kernel, one obtains an explicit solution in terms of the classical signature of the underlying noise $x$. We detail this in the following proposition.

\begin{prop} \label{prop:mean reverting 1d}
We consider the same situation as in Proposition \ref{prop:mean_reverting_prop}, albeit in the scalar case $q=R=1$ and $K_{A,b}^\Lambda(t,s)=:K_\alpha^\lambda(t,s)= \alpha e^{-\lambda(t-s)}I_d$. Recall that the usual signature $\mathrm{Sig}(x)$ is defined by~\eqref{usual_sig_def}. Then equation \eqref{eq:mean_reverting} admits an explicit solution as a function of $\mathrm{Sig}(x)$:
    \begin{equation}\label{eq:exp_dynamical_sig}
    \VSig{x}{K}_{s,t}^t =   e^{-\lambda (t-s)}\mathrm{Sig}(\alpha x)_{s,t}+  \lambda \int_s^t   e^{-\lambda (t-r)} \mathrm{Sig}(\alpha x)_{r,t} \dd{r}, \quad (s,t) \in \Delta^2.
    \end{equation}
    More generally, for any $(s,t,\tau) \in \Delta^3$ we have \begin{equation}\label{eq:exp_dynamics_gen}
    \VSig{x}{K}_{s,t}^{\tau} = 1 +  e^{-\lambda(\tau - s)}(\mathrm{Sig}(\alpha x)_{s,t}-1) + \lambda \int_s^t   e^{-\lambda (\tau-u)}(\mathrm{Sig}(\alpha x)_{u,t} - 1) \dd{u}.
    \end{equation}
\end{prop}
\begin{proof} Let us denote by $\tilde{\by}$ the right hand side of \eqref{eq:exp_dynamical_sig}. Applying the product rule like in~\eqref{eq:MR_proof_2}, we get \[
\dd \tilde{\by}_{s,t}  =-\lambda \tilde{\by}_{s,t} \dd t +  e^{-\lambda (t-s)}\dd \mathrm{Sig}(\alpha x)_{s,t}
+\lambda \left(\int_s^t   e^{-\lambda (t-r)} \dd \mathrm{Sig}(\alpha x)_{r,t} \right)\dd{r}.
\] Then one resorts to relation \eqref{eq:classical_sig_eq}, which enables to write \begin{equation}\label{eq:proof_scalar_1}
    \dd \tilde{\by}_{s,t} = -\lambda \tilde{\by}_{s,t} \dd t +  e^{-\lambda (t-s)} \mathrm{Sig} (\alpha x)_{s,t}\otimes \dd (\alpha x_t)
     +\lambda \left(\int_s^t   e^{-\lambda (t-r)}  \mathrm{Sig}(\alpha x)_{r,t} \otimes \dd (\alpha x_t) \right) \dd{r}.
\end{equation} Similarly to what we did in \eqref{eq:proof_MR_4}, we now recognize the right hand side of \eqref{eq:exp_dynamical_sig} in \eqref{eq:proof_scalar_1}, and thus \[
\dd \tilde{\by}_{s,t} = -\lambda \tilde{\by}_{s,t} \dd t+  \tilde{\by}_{s,t} \otimes \dd (\alpha x_t),
\] so that
     so that indeed $\VSig{x}{K}_{s,t}^t = \tilde{\by}$ by \cref{prop:mean_reverting_prop}. This finishes the proof of~\eqref{eq:exp_dynamical_sig}. Finally, the second assertion  \eqref{eq:exp_dynamics_gen} is achieved by plugging the explicit expression for $\VSig{x}{K}_{s,t}^t$ in \eqref{eq:exp_dynamical_sig} into the fundamental equation \eqref{eq:vsig_fundamental}, and follow the same lines of arguments used to prove \eqref{eq:exp_dynamical_sig}. This finishes our proof.
     u where we applied Fubini in the penultimate equation.
\end{proof}

\subsection{Connection to deep learning state space models}\label{sec:conection to deep learning}

The finite state space kernel formulation introduced in Definition \ref{def:finite_state_space_kernels} shares its fundamental mathematical construction with recent models proposed in deep learning called state space models (SSMs), such as structured state space sequence models (S4) and Mamba/Mamba-2 \cite{gu2022efficiently,gu2023mamba,dao2024transformers}. 

At their core,  modern continuous-time neural SSMs map a one-dimensional input signal $u_t$ to a one-dimensional output $y_t$ via a latent state $h_t\in \RR^R$,  where the latent state is governed by the linear differential equation 
\[
\dd h_t = Mh_t \dd t + B u_t \dd t,\quad y_t = Ch_t. 
\]
Here, we have $M\in \RR^{R\times R}$, $C\in \RR^{1\times R}$ and $B\in \RR^{R\times 1}$. 
In the case when $h_0=0$ the exact solution to this system can be written in terms of a convolution of the form 
\[
y_t = \int_0^t Ce^{M(t-s)}Bu_s \dd s. 
\]
This is precisely a Volterra integral, as described in the first level of the Volterra signature, and which we sometimes refer to as the memory state. Indeed, identifying $M=-\Lambda$, $B=b_r$ and $C=\mathbf{1}^{\top}$ yields the kernel $k_r(t,s) = \mathbf{1}^\top e^{-\Lambda(t-s)}b_r $ as proposed in Definition \ref{def:finite_state_space_kernels}. While neural SSMs typically handle multi-dimensional sequences by applying these scalar linear-time-invariant  systems independently across feature channels, our formulation handles the multi-dimensional input-output mapping via the linear combination $K(t,s)= \sum_{r=1}^q k_r(t,s)A_r\in \cL(\RR^d;\RR^m)$, for matrices $A_r$.

A fundamental distinction, however, lies in how non-linear expressivity is achieved. Mamba and its successors achieve universal approximation by making the state matrices data-dependent (selective SSMs) and stacking multiple layers interwoven with non-linear activation functions. In contrast, the Volterra signature achieves non-linearity algebraically by lifting this linear memory mechanism into the extended tensor algebra $T((\RR^m))$. 

This algebraic lift can be seen explicitly through the lens of the fundamental equation described in Proposition \ref{prop:mean_reverting_prop} which governs the Volterra signature dynamics over the state-space kernel as seen in \eqref{eq:mean_reverting}. We can interpret this equation as a high-order tensor valued SSM, where the first level of the equation describes exactly (under specific matrix choices) the structure of models like S4, and Mamba (1 and 2), but where the equation also contains high order information that captures more intricate dynamics and connections.  The signature state $\bZ_{s,t}$ in \eqref{eq:mean_reverting} undergoes a linear, mean-reverting decay $-\Lambda \bZ_{s,t} \dd t$, while the increment $\dd (B \, x_t)$ interacts with the aggregated history multiplicatively via the tensor product $\otimes$. Here we can think of $x_t = \int_0^t u_t \dd t$, where $u_t$ is the input signal to the SSM.  It replaces black-box stacked layers with a graded Picard expansion that retains strong universality guarantees, as will be shown in the next section, while maintaining the exact same underlying memory architecture.

\section{The Volterra signature as feature map}\label{sec: Sig as feature map}
This section presents the key properties of the Volterra signature that make it relevant as a feature map on path space.
Specifically, we will discuss its invariance under time reparameterization, continuity and injectivity. As a consequence, we will characterize classes of simple Volterra signature functionals, which are universal for the approximation of continuous functionals defined on the path space. Recall (from Definition \ref{def:vsig} and Proposition~\ref{prop:vsig_z}) that the Volterra signature can either be defined above the data $(x,K)$, or equivalently above Volterra paths $z \in \Vone$. This observation would compel us to study invariance and universality of the signature by considering only the $T((\RR^m))$-valued signature $\bz$ given intrinsically by equation \eqref{eq:vsig_fundamental}. However, as mentioned in Remark~\ref{rem:proof_rem}, the kernel $K$ can be seen as an additional  modeling component. It allows for more flexibilty compared to the classical signature transform. Therefore we also consider the mappings introduced in Notation \ref{b1}, namely \begin{equation}\label{eq:vsig_feature_c}
  \cC^{0,1}([0,T];\RR^d) \rightarrow T((\RR^m)), \quad  x \mapsto \VSig{x}{K},
\end{equation} for a fixed kernel $K$, and similarly \begin{equation}\label{eq:vsig_feature_K}
L^{1,\infty}(\Delta^2,\cL(\RR^d;\RR^m)) \rightarrow T((\RR^m)), \quad K \mapsto \VSig{x}{K}
\end{equation} for given underlying noise $x$.

\subsection{Time reparameterization and continuity properties}

In this section we will address two basic though important properties of Volterra signatures: invariance under time reparameterization and continuity with respect to both the underlying signal $x$ and kernel $K$. The two questions have consequences on data analysis, since good invariance and continuity properties are convenient for modeling in an uncertain environment.

Let us start by discussing invariance under reparametrization.
By time reparameterization we mean a smooth and increasing function $\rho:[0,T] \rightarrow [0,T]$ with $\rho(0)=0$ and $\rho(T)=T$. For a given reparameterization $\rho$ and a two-parameter function $f:\Delta^2 \rightarrow \RR$, we also define $f\circ \rho$ as follows: \begin{equation}\label{eq:reparametriz_f}
f \circ \rho : \Delta^2 \rightarrow \RR, \quad (f\circ \rho)(s,t):=f(\rho(s),\rho(t)).
\end{equation} With this notation in hand, if one considers the case of usual signatures $\mathrm{Sig}(x)$ (introduced in Remark \ref{rem:usual_sig}), it is a well-known fact (see \cite[Proposition 7.10]{Friz2010}) that \begin{equation}\label{eq:classical_invar}
\mathrm{Sig}{(x)}_{0,T} = \mathrm{Sig}{(x\circ \rho)}_{0,T}, \qquad \forall x \in \cC^{0,1}([0,T];\RR^d).
\end{equation}
In the next proposition we generalize \eqref{eq:classical_invar} to the $\mathrm{VSig}$ context, showing that Volterra signatures enjoy valuable invariance properties.

\begin{prop}\label{prop:reparametrization}
        Let $\rho:[0,T]\rightarrow [0,T]$ be a smooth and monotone increasing function such that $\rho(0)=0$ and $\rho(T)=T$. Then the invariance of $\mathrm{VSig}$ under the reparameterization $\rho$ can be expressed in the two following ways. \begin{itemize}
            \item[(i)] When one considers $x \in \cC^{0,1}([0,T];\RR^d)$, a kernel $K \in L^{1,\infty}(\Delta^2,\cL(\RR^d;\RR^m))$ and a signature $\VSig{x}{K}$ as given in Notation \ref{b1}, we have \begin{equation}\label{eq:invar_x_K}
      \VSig{x\circ \rho}{K \circ \rho}_{0,T}^T\;=\;\VSig{x}{K}_{0,T}^T.
  \end{equation}
  \item[(ii)] For an intrinsic Volterra path $z \in \Vone$ from Definition \ref{def:volterra_path} and its corresponding intrinsic Volterra signature $\bz$ satisfying \eqref{eq:vsig_fundamental}, the invariance reads \begin{equation}\label{eq:invar_z}
  \bz_{0,T}^T(z \circ \rho) = \bz_{0,T}^T( z).
  \end{equation} 
        \end{itemize}

\end{prop}
\begin{proof} 
First notice that for any Volterra path $z_t^\tau = \int_0^tK(\tau,u) \dd x_u$, by a simple change of variables it follows that \[
(z \circ \rho)_{t}^\tau= \int_{0}^{\rho(t)}K(\rho(\tau),u)\dd x_u = \int_0^{t}K(\rho(\tau),\rho(u))\dot{x}_{\rho(u)}\dot{\rho}(u)\dd u = \int_0^t K(\rho(\tau),\rho(u)) \dd x_{\rho(u)}.
\] In view of Lemma \ref{lem:equivalence_def_VP}, it is therefore sufficient to prove \eqref{eq:invar_x_K}.

To this end, we aim to prove that for  any $(s,t,\tau) \in \Delta^3$, it holds that \begin{equation}\label{eq:induction_invar}
\VSig{x\circ \rho}{K\circ \rho}_{s,t}^{\tau}= \VSig{x}{K}_{\rho(s),\rho(t)}^{\rho(\tau)},
\end{equation} from which the claim immediately follows for $s=0$ and $t=\tau=T$.
In order to show \eqref{eq:induction_invar}, we proceed by induction over $n\in \mathbb{N}$. On the first level, a change of variables $u = \rho(r)$ shows that 
\begin{eqnarray*}\label{reparametrization invariance first level}
    \VSig{x\circ \rho }{K\circ \rho }_{s,t}^{(1),\tau}
    &=&
    \int_{s}^{t} K(\rho(\tau),\rho(r))  \dd x_{\rho(r)} =  \int_{\rho(s)}^{\rho(t)} K(\rho(\tau),u) \dd x_{u} \\
    &=& \VSig{x}{K}_{\rho(s),\rho(t)}^{(1),\rho(\tau)}.
\end{eqnarray*} 
We have already seen in \eqref{eq:fundamental_coord} and \eqref{eq:fundamental_coord_2} that, for any  word $i_1\cdots i_{n+1} \in \cW$ it holds 
\[
  \VSig{x\circ \rho }{K\circ \rho }^{i_1\cdots i_{n+1},\tau}_{s,t}
  =
  \int_s^t \VSig{x \circ \rho}{K\circ \rho }_{s,r}^{i_1\cdots i_n,r}
   K^{i_{n+1}}(\rho(\tau),\rho(r))\,\dd x_{\rho(r)}.
\]
Now assuming \eqref{eq:induction_invar} holds for the tensor level $n\in \mathbb{N}$, the same change of variables $u=\rho(r)$ shows 
\begin{align*}
  \VSig{x \circ \rho}{K\circ \rho }^{i_1\cdots i_{n+1},\tau}_{s,t}
  & = 
  \int_s^t
  \VSig{x}{K}_{\rho(s),\rho(r)}^{i_1\cdots i_n,\rho(r)}K(\rho(\tau),\rho(r))\dd x_{\rho(r)} \\ & = \int_{\rho(s)}^{\rho(t)} \VSig{x}{K}_{\rho(s),u}^{i_1\cdots i_n,u} K^{i_{n+1}}(\rho(\tau),u) \dd x_u \\ &
  = \VSig{x}{K}_{\rho(s),\rho(t)}^{i_1\cdots i_{n+1},\rho(\tau)},
  \end{align*} where we again invoked \eqref{eq:fundamental_coord_2} for the last equality.
\end{proof}

We now turn to continuity properties of Volterra signatures, which are summarized in the proposition below.

\begin{prop}\label{prop:continuity}
For any Volterra path $z \in \Vone$, recall from \cref{def:VP_def_z}, that we denote by $\bz$ the full signature of $z$. Then, for any $n\in \mathbb{N}$, the mapping \begin{equation}\label{eq:continuity_VSig}
    \bz_{0,T}^{(n),T}: (\Vone, \Vert \cdot \Vert_{\cV^1}) \rightarrow ((\RR^m)^{\otimes n}, \Vert \cdot \Vert), \qquad z \mapsto \bz_{0,T}^{(n),T},
\end{equation}is locally Lipschitz continuous. Similarly, the Volterra signature $\VSig{x}{K}^{(n),T}_{0,T}$, as introduced in \cref{def:vsig}, is locally Lipschitz continuous jointly in $(x,K)$ seen as a pair in $\cC^{0,1}([0,T];\RR^d) \times \Lkernel$.
 \end{prop}

\begin{proof}
    Consider two Volterra paths $z,y \in \Vone$ with Volterra signatures denoted by $\bz$ and $\by$. For $n=1$, we have for any $t \in [0,T]$\[
     \Vert \bz_{0,t}^{(1),t}-\by_{0,t}^{(1),t} \Vert \leq \int_0^t \left |\frac{\dd z_u^t}{\dd u}-\frac{\dd y_u^t}{\dd u}\right | \dd u \leq \Vert z-y \Vert_{\cV^1},
\] where we recall that the norm $\Vert \cdot \Vert_{\cV^1}$ is introduced in \eqref{eq:norm}. Therefore the claim \eqref{eq:continuity_VSig} holds on the first level. We will now show by induction, that for any $n \in \mathbb{N}$ we have
    \begin{equation}\label{eq:induc_contin}
    \Vert \bz_{0,t}^{(n),t}-\by_{0,t}^{(n),t}  \Vert  \leq C_n \Vert z-y \Vert_{\cV^1}, \qquad \forall t \in [0,T],
    \end{equation} for some constant $C_n=C_n(\Vert y\Vert_{\cV^1}, \Vert z \Vert_{\cV^1})$. To this aim, observe that applying Fubini's theorem to \eqref{eq:iterated Volterra path_coordinate} one gets the following expression which stems from \eqref{eq:convol_def_full}: \begin{equation}\label{eq:fubini_proof}
    \bz_{0,t}^{(n+1),t} = \int_0^t \bz_{0,u}^{(n),u} \otimes \frac{\dd z_u^t}{\dd u} \dd u,
    \end{equation} and the same type of expression holds for $\by_{0,t}^{(n+1),t}$. Hence an application of the triangle inequality shows \begin{align}\label{eq:continuity_proof_1}
        \Vert \bz_{0,t}^{(n+1),t}-\by_{0,t}^{(n+1),t}  \Vert & \leq \int_0^t \Big \Vert \bz_{0,u}^{(n),u} \otimes \frac{\dd z_u^t}{ \dd u}-\by_{0,u}^{(n),u} \otimes \frac{\dd y_u^t}{ \dd u}  \Big \Vert \dd u \nonumber \\ & \leq \int_0^t \big \Vert \bz_{0,u}^{(n),u} -\by_{0,u}^{(n),u}  \big \Vert \Big \vert \frac{\dd y_u^t}{ \dd u} \Big \vert \dd u + \int_0^t \big \Vert \bz_{0,u}^{(n),u}\big \Vert\Big \vert   \frac{\dd z_u^t}{ \dd u}-  \frac{\dd y_u^t}{ \dd u}  \Big \vert \dd u. 
    \end{align} We now bound the two terms in the right hand side of \eqref{eq:continuity_proof_1} separately. On the one hand assuming that the induction hypothesis \eqref{eq:induc_contin} is satisfied for $n\in \mathbb{N}$, we have  \[
    \int_0^t \big \Vert \bz_{0,u}^{(n),u} -\by_{0,u}^{(n),u}  \big \Vert \Big \vert \frac{\dd y_u^t}{ \dd u} \Big \vert \dd u 
    \leq 
    C_n \Vert z-y \Vert_{\cV^1}\Vert y \Vert_{\cV^1}.
    \] 
    On the other hand, for the second term in \eqref{eq:continuity_proof_1}, again 
    \[
    \int_0^t \big \Vert \bz_{0,u}^{(n),u}\big \Vert\Big \vert   \frac{\dd z_u^t}{ \dd u}-  \frac{\dd y_u^t}{ \dd u}  \Big \vert \dd u 
    \leq 
    C_n\Vert z \Vert_{\cV^1} \Vert z-y \Vert_{\cV^1}.
    \] 
    Reporting those two observations in the right hand side of \eqref{eq:continuity_proof_1}, we conclude \[
    \Vert \bz_{0,t}^{(n+1),t}-\by_{0,t}^{(n+1),t}  \Vert \leq C_{n+1}\Vert z-y \Vert_{\cV^1},
    \] 
    where the constants $C_n$ satisfy the recursive relation \[C_{n+1}=C_n(\Vert y \Vert_{\cV^1},\Vert z \Vert_{\cV^1})\Vert y \Vert_{\cV^1}+C_n(0,\Vert z \Vert_{\cV^1})\Vert z \Vert_{\cV^1}.\]
    We have thus proved the first part of our statement, namely continuity of the mapping $z \mapsto \bz_{0,T}^{(n),T}$ for every $n$.
    For the second part of the statement, we can notice that the mapping \[
    (x,K) \mapsto \Big\{z_t^\tau = \int_0^t K(\tau,u)\dd x_u: (t,\tau)\in \Delta^2 \Big \} \in \Vone,
    \] is jointly locally Lipschitz continuous. This follows by similar arguments as before: for two pairs $(x,K),(x',K') \in \cC^{0,1}([0,T];\RR^d) \times L^{\infty,1}(\Delta^2;\cL(\RR^d;\RR^m))$, the triangle inequality shows \begin{align*}
        \Big |z(x,K)_t^\tau-z(x',K')_t^\tau \Big | & \leq \int_0^t|K(\tau,u)| |\dot{x}_u-\dot{x}'_u | \dd u+ \int_0^t |K(\tau,u)-K'(\tau,u) | |\dot{x}'_u|\dd u \\ & \leq \Vert x-x' \Vert_{\cC^{0,1}} \Vert K \Vert_{L^{\infty,1}}+ \Vert K-K'\Vert_{L^{\infty,1}} \Vert x' \Vert_{\cC^{0,1}}.
    \end{align*} Therefore, $\mathrm{VSig}$ is the composition of two locally Lipschitz continuous functions, which finishes the proof.
\end{proof}

\subsection{Injectivity}\label{sec:injectivity}
In this section we investigate the injectivity (or \emph{point-separation}) of the Volterra signature. Since we are advocating for the Volterra signature as a feature map on the space of paths (or time-series), injectivity becomes a relevant property in many applications, such as supervised learning problems like classification and regression. 
As observed in the last section, point-separation can fail since for example $z$ and $z \circ \rho$ produce the same Volterra signature when $\rho$ is a monotone function, see Proposition \ref{prop:reparametrization}. This phenomena is well-understood in the classical case, where $K\equiv 1$, and a full characterization of paths with identical signature is given by the notion of \emph{tree-like equivalence}, see \cite{HamLy10,BXTD2016}. A typical strategy to ensure a point-separating signature, is to augment the paths $x$  with a monotone component, and thereby making the tree-like equivalence classes trivial. Indeed, we can for instance notice that on the space of time-augmented paths $\hat{x}_t=(t,x_t)$, the relation $\hat{x}_1=\hat{x}_2 \circ \rho $ is only possible if already $x_1=x_2$.

Returning to the Volterra framework, in order to guarantee point-separation, we introduce the space of Volterra paths augmented by a function $y$.

\begin{defn}\label{def:augmented_Volterra_paths}
For any fixed $y\in \mathcal{V}^{1}([0,T];\mathbb{R}^{m'})$ and $m' \in \mathbb{N}$, we define the space of $y$-augmented Volterra paths by \begin{equation}\label{eq:def_augmented_VP}
    \hat{\mathcal{V}}_y^{1}([0,T];\mathbb{R}^{m+m'}) = \left \{ z \in \cV^1([0,T];\RR^{m+m'}): (z^1,\dots,z^{m'})^{\top} = y \right \}.
\end{equation}
\end{defn}
With the definition of augmented Volterra path in hand, we label a general Hypothesis under which we will achieve point separation by Volterra signatures. The reader is referred to the Examples \ref{ex:inje_ex_1}-\ref{ex:inje_ex_2} below for more specific and natural examples of application. 

\begin{hyp}\label{hyp:injectivity} In the sequel we suppose $y\in \mathcal{V}^{1}([0,T];\mathbb{R}^{m'})$ is such that there exist sequences of words $(w_n)_{n\in \NN},(v_n)_{n\in \NN}$ in the alphabet $\{1,\dots,m' \}$, such that the families\begin{equation}\label{eq:kernel_assumption_inje}
 \{ t \mapsto \by_{0,t}^{w_n,t}: n \geq 1 \} \quad \text{and} \quad  \left \{ t \mapsto \frac{\dd}{\dd t}\by_{t,T}^{v_n,T}: n \geq 1 \right \}\end{equation} lie dense in $L^2([0,T],\RR)$, where $\by$ is defined through \eqref{eq:convol_def_full}.
\end{hyp}
\newcommand{\hVone}{\hat{\mathcal{V}}^1([0,T];\RR^m)}
The following is the main result of this section. It shows that under \cref{hyp:injectivity}, the Volterra signature carries enough information to uniquely determine the two-parameter path \((t,\tau) \mapsto z_{0,t}^{\tau}\). %
\begin{thm}\label{thm:point_separation}
    For any $y\in \mathcal{V}^{1}([0,T];\mathbb{R}^{m'})$ such that \cref{hyp:injectivity} holds true, the following map is injective: \begin{equation}\label{eq:injectivity_1}
\hat{\cV}_y^{1}([0,T];\mathbb{R}^{m'+m}) \longrightarrow T((\mathbb{R}^{m'+m})), \quad z \longmapsto {\bz}^{T}_{0,T}, \end{equation}where we recall that the space $\hat{\cV}_y^1([0,T];\RR^{m+m'})$ is defined by \eqref{eq:def_augmented_VP}.
\end{thm}

\begin{proof}
As a first step, based on the convolution properties seen in Section \ref{sec: the volterra signature} we can show the following identity, similarly to what we did for \eqref{eq:fubini_proof}: 
\begin{equation}\label{eq:partial_rel_proof}
         \bz^{i_1\cdots i_n,\tau}_{s,t} = -\int_s^t \bz_{s,u}^{i_1\cdots i_k,u}\frac{\dd}{\dd u} \bz_{u,T}^{i_{k+1}\cdots i_n,T} \dd u,
     \end{equation} for all $i_1\cdots i_n \in \cW$ and $0\leq k \leq n-1$.
     Indeed, combining the expression for the convolution product in Definition \ref{def:convolution_coordinate} with the identity \eqref{eq:iterated Volterra path_coordinate}, we observe that $\bz_{s,t}^{i_1\cdots i_n,\tau} = [\bz_{s,\cdot}^{i_1\cdots i_k,\cdot} \ast z]^{i_{k+1}\cdots i_n,\tau}_{s,t}$. An application of Lemma \ref{lem:convolution_properties}--(ii) for $y_t=\bz_{s,t}^{i_1\cdots i_k,t}$ then readily implies \eqref{eq:partial_rel_proof}.

     Next let  $z,h \in \hat{\cV}_y^1$ and suppose that $\bz_{0,T}^T =\mathbf{h}_{0,T}^T$. Moreover for $i\in \{1,\dots,m+m'\}$ and the two sequences $(w_n),(v_n)$ in \cref{hyp:injectivity}, define $\pi_{n,k} = w_n\cdot i\cdot v_k$. Then it follows from \eqref{eq:partial_rel_proof} and the fact that $z^{j}=h^j=y^j$ for $j \in \{1,\dots,m'\}$, that \begin{align*} 0=
    \bz_{0,T}^{\pi_{n,k},T}-\mathbf{h}_{0,T}^{\pi_{n,k},T} & = \int_0^T (\mathbf{h}_{0,u}^{w_n \cdot i,u}-\mathbf{z}_{0,u}^{w_n \cdot i,u}) \frac{\dd}{\dd u} \by_{u,T}^{v_k,T}\dd u  \\ & = \int_0^T \left (\int_0^u\by_{0,r}^{w_n,r}\left (\frac{\dd h^{i,u}_r}{\dd r}-\frac{\dd z^{i,u}_r}{\dd r}\right )\dd r\right ) \frac{\dd}{\dd u} \by_{u,T}^{v_k,T}\dd u, \quad \forall n,k \geq 1.
    \end{align*}
    Now by \cref{hyp:injectivity}, the family $\{u \mapsto \frac{\dd}{\dd u} \by_{u,T}^{v_k,T}:k \geq 1 \}$ lies dense in $L^2$, so that \[
    \int_0^u\by_{0,r}^{w_n,r}\left (\frac{\dd z^{i,u}_r}{\dd r}-\frac{\dd h^{i,u}_r}{\dd r}\right )\dd r = 0 \quad \text{for a.e. }u \in [0,T] \text{ and for all }n\geq 1.
    \] Similarly, since $\{r \mapsto \by_{0,r}^{w_n,r}: n\geq 1 \}$ lies dense in $L^2$, we conclude \[
    \frac{\dd z^{i,u}_r}{\dd r}=\frac{\dd h^{i,u}_r}{\dd r} \quad \text{ for a.e. } (u,r) \in \Delta^2.
    \] Since we assume in Definition \ref{def:volterra_path} that $z_0^\tau=0$, we get $z=h$ everywhere. This proves injectivity.
\end{proof}

As a first application of \cref{thm:point_separation}, we can show that the Volterra signature above a time-augmented Volterra path $\hat{z}_{s,t}^\tau =(t-s,z_{s,t}^\tau)$ uniquely characterizes $z$. 
\begin{cor}\label{cor:injective volterra signature z}
Let $m'=1$ and consider the Volterra path $y_{s,t}^\tau = t-s$. We define the $y$-augmented Volterra path $z \in \hat{\cV}_y^1([0,T];\RR^{m+1})$ as in Definition \ref{def:augmented_Volterra_paths}. Then we have injectivity of the map in \eqref{eq:injectivity_1} in this context. %
\end{cor}

\begin{proof} Note that $y$ is indeed a Volterra path according to Definition \ref{def:volterra_path}, applied to $K\equiv 1$ and $x_t=t$. Now choosing the sequences of words in \cref{hyp:injectivity} consisting of $n-1$, resp. $n$ consecutive $1$'s, i.e.   $w_n= 1^{\otimes (n-1)}= 1\cdots 1$ and $v_n= 1^{\otimes n}$ for $n \geq 1$ we can easily verify that for any $z \in \hat{\cV}_y^1$ it holds  \begin{equation}\label{eq:monomials} {\by}_{0,t}^{w_n,t} = \frac{t^{n-1}}{(n-1)!}, \quad \text{and} \quad \frac{\dd}{\dd t}{\by}_{t,T}^{v_n,t} = -\frac{(T-t)^{n-1}}{(n-1)!}, \quad n\geq 1.
\end{equation} But since monomials are dense in $L^p$ by the Weierstrass theorem, we see that \cref{hyp:injectivity} is satisfied and we can conclude using \cref{thm:point_separation}.
\end{proof}

In Corollary \ref{cor:injective volterra signature z} we have augmented the path $z$ with a path $y_{s,t}^\tau = (t-s)$, which is very similar to the augmentations used in classical rough paths theory (see \cite{Chevyrev2016}). In order to be consistent with our Volterra perspective we now specialize to an augmentation by Volterra paths. Let us first define this notion.

\begin{defn}\label{def:Volterra_time_aug}
    Let $z \in \cV^1([0,T];\RR^{m+m'})$ and $K^y \in L^{\infty,1}(\Delta^2;\cL(\RR^{d'};\RR^{m'}))$. In the framework of Definition~\ref{def:augmented_Volterra_paths}, we say that $z$ is $K^y$-augmented if one can write \[
    y_{s,t}^{\tau} \equiv (z^{1,\tau}_{s,t},\dots,z^{m',\tau}_{s,t}) = \int_s^tK^y(\tau,r) \dd x_r^y,
    \] for a path $x^y \in \cC^{0,1}([0,T];\RR^{d'})$. Provided with another kernel $K \in L^{\infty,1}(\Delta^2;\cL(\RR^d;\RR^{m}))$, the corresponding path $z \in \hat{\cV}^1_y([0,T]; \RR^{m+m'})$ will be denoted by \begin{equation}\label{eq:aug_x_Volterra}
        z_{s,t}^\tau= \int_s^t \widehat{K}^y(\tau,u) \dd \widehat{x}^y_u,
    \end{equation} where the pair $(\widehat{x}^y,\widehat{K}^y)$ in $\cC^{0,1}([0,T]; \RR^{d+d'}) \times L^{\infty,1}(\Delta^2;\cL(\RR^{d+d'};\RR^{m+m'}))$ is such that \begin{equation}\label{eq:aug_xK_pair}
        \widehat{x}^y = (x^y,x)^\top, \quad\text{and}\quad 
        \widehat{K}^y(t,s)= \begin{pmatrix}
            K^y & 0\\ 0 & K
        \end{pmatrix}.
    \end{equation}
\end{defn}

While $K_1$-augmented path are a particular case of general augmented paths, it is also true that any augmented Volterra path can be seen as a $K_1$-augmented path. We label this property in the lemma below. Its proof is omitted due to its similarity with Lemma~\ref{lem:equivalence_def_VP}. 

\begin{lem}\label{lem:equivalence_aug}
    Let $z$ be a path in the space $\hat{\cV}_y^1([0,T];\RR^{m+m'})$, as introduced in Definition~\ref{def:augmented_Volterra_paths}. Then \begin{itemize}
        \item[(i)] One can find a pair $(\widehat{x}^y,\widehat{K}^y)$ like in \eqref{eq:aug_xK_pair}, such that $z$ is represented as in \eqref{eq:aug_x_Volterra}.

        \item[(ii)] In particular, one can choose $\widehat{x}_t^y=t$ and $K^y \in L^{\infty,1}(\Delta^2,\RR^{m'})$ in \eqref{eq:aug_x_Volterra}. Then the path $z$ is more specifically written as 
        \begin{equation} \label{eq:noise_augmentation} 
        z_{s,t}^\tau = \int_s^tK(\tau,r)\dd \hat{x}_r, \quad\text{with}\quad \hat{x}_t= (t,x_t)^{\top} \in \RR^{d+1}. 
        \end{equation}
    \end{itemize}
\end{lem}

\begin{rem}\label{rem:aug_rep_K1} In the context of a representation of the form \eqref{eq:noise_augmentation}, consider a fixed kernel $K^y=K_1$. Then, for any word $w=i_1\cdots i_n$ in the alphabet  $\{1,\dots,m'\}$, we have \begin{equation}\label{eq:b_eq}
\bz_{s,t}^{w,\tau}=\VSig{t}{K_1}_{s,t}^{w,\tau} = \int_{\Delta_{s,t}^n}\prod_{l=1}^nK_1^{i_l}(r_{l+1},r_l)\dd r_l, 
\end{equation} where we are still using the convention $r_{n+1} = \tau$ in \eqref{eq:b_eq}.  As we shall see in the following examples, \(\VSig{t}{K_1}\) admits a closed-form expression for many kernels of interest, allowing \cref{hyp:injectivity} to be verified directly.
 
\end{rem}

\begin{example}\label{ex:inje_ex_1}
Following up on Remark \ref{rem:aug_rep_K1}, consider representation \eqref{eq:noise_augmentation} with $K^y \equiv 1$ and $\widehat{x}_t^y = (t,x_t)$. In this case we can observe
$$\VSig{t}{1}^{(n),\tau}_{s,t}=\frac{(t-s)^n}{n!},$$ which led to the expressions \eqref{eq:monomials} in \cref{cor:injective volterra signature z}, so that \cref{hyp:injectivity} was satisfied thanks to the Weierstrass theorem.
\end{example}

\begin{example}\label{ex:inje_ex_2}  Still continuing Remark \ref{rem:aug_rep_K1}, let us now take $K^y(t,s)=\Gamma(\alpha)^{-1}(t-s)^{\alpha-1}$ in \eqref{eq:noise_augmentation}. We then get \[
 \VSig{t}{K_1}^{(n),\tau}_{s,t}=c_{\alpha,n} \int_s^t(\tau-u)^{\alpha-1}(u-s)^{(n-1)\alpha}ds, 
\] for some explicit constant $c_{\alpha,n}$. In particular, one can easily check that \[
\VSig{t}{K_1}_{0,u}^{(n),u} = \hat{c}^1_{\alpha,n}u^{n\alpha}, \qquad \frac{\dd}{\dd u}\VSig{t}{K_1}_{u,T}^{(n),T} = \hat{c}^2_{\alpha,n}(T-u)^{n\alpha-1},
\] for some constants $\hat{c}^1_{\alpha,n},\hat{c}^2_{\alpha,n} \neq 0$.  Both families are dense in $L^2$, as a consequence of the Szász's theorem \cite[Satz A]{szasz1916approximation}: For any sequence  of real numbers $0<\lambda_1<\lambda_2<\cdots$, the linear span of the family $\{t^{\lambda_i}:i \geq 1\}$ is dense in $L^2([0,T],\mathbb{R})$, if and only if $\sum_{i}\frac{1+2\lambda_i}{1+(\lambda_i)^2}= \infty$. 
\end{example}
Having set a notation for augmented paths related to kernels, we now turn to the question of injectivity of the Volterra signature in terms of the kernel $K$ or in terms of the signal $x$. Otherwise stated, we shall discuss the injectivity of the mapping \eqref{eq:vsig_feature_c} and \eqref{eq:vsig_feature_K}. Note that both questions are relevant in applications.

\begin{prop}\label{prop:injectivity_noise_kernel}
    Assume $y \in \cV^1([0,T];\RR^{m'})$ is such that \cref{hyp:injectivity} holds. We consider $z \in \hat{\cV}_y^1([0,T]; \RR^{m+m'}),$ represented as in \eqref{eq:aug_x_Volterra}-\eqref{eq:aug_xK_pair}. Then \begin{itemize}
        \item[(i)] For a fixed $K\in \Lkernel$ such that $\mathrm{det}(K(T,t))\neq 0$ for a.e. $t \in [0,T]$, the map \begin{equation*}
        \Big [\cC^{0,1}([0,T];\RR^d) \cap \{x:x_0 = \xi\} \Big ]\ni x \mapsto \VSig{\widehat{x}^y}{\widehat{K}^{y}}_{0,T}^T \in T((\RR^{m+m'}))
    \end{equation*}
    is injective for any initial value $\xi \in \RR^d$.
\item[(ii)] For a fixed $\mathcal{M} \subset \cC^{0,1}([0,T];\RR^d)$ with $|\mathcal{M}| <\infty$ and such that  $\mathrm{span}\{\dot{x}_t : x\in \mathcal{M}\} = \RR^d$ for a.e. $t\in [0,T]$, the map \begin{equation*}
        L^{\infty,1}(\Delta^2; \mathcal{L}(\RR^d;\RR^m)) \ni K \mapsto \left(\VSig{\widehat{x}^{y}}{\widehat{K}^{y}}_{0,T}^T\right)_{x\in \mathcal{M}} \in T((\RR^{m+m'}))^{|\mathcal{M}|}
    \end{equation*}
    is injective.
    \end{itemize} 
\end{prop}
\begin{proof}
    For any $(x_1,K_1),(x_2,K_2) \in \cC^{0,1} \times L^{1,\infty}$, it follows from \cref{thm:point_separation} that whenever $\VSig{\widehat{x}_1^y}{\widehat{K}_1^{y}}_{0,T}^T=\VSig{\widehat{x}_2^y}{\widehat{K}_2^{y}}_{0,T}^T$, the corresponding Volterra paths coincide, and thus in particular \begin{equation}\label{eq:equal_deriv_proof}K_1(\tau,t) \dot{x}_{1,t} = K_2(\tau,t) \dot{x}_{2,t} \quad \text{for a.e. } (t,\tau)\in \Delta^2.\end{equation} If $K=K_1=K_2$ as in (i), choosing $\tau=T$ and using $\mathrm{det}(K(T,t)) \neq 0$ shows that $\dot{x}_1=\dot{x}_2$ almost everywhere. Since we have assumed in (i) that $x_{1,0}=x_{2,0} = \xi$, we get $x_1=x_2$. Let us now turn to case (ii).
    Then for each $x \in \mathcal{M}$ we obtain from \eqref{eq:equal_deriv_proof} with $x_1=x_2 = x$ that $K_1(\tau,t)\dot{x}_{t} = K_2(\tau,t)\dot{x}_{t}$ for a.e. $(t,\tau)\in \Delta^2$. Then owing to the fact that $\mathrm{span}\{\dot{x}_t: x\in \mathcal{M}\} = \RR^d$ for a.e. $t$ (part of our assumptions in (ii)), we also get $K_1(\tau,t)=K_2(\tau,t)$ for a.e. $(t,\tau) \in \Delta^2$. This finishes our proof. 
\end{proof}

\subsection{Universal approximation}\label{sec:universal}
In this section we provide a theoretical basis for how the Volterra signature can be used for learning continuous functions on path space.
We start by stating a generic result, which leverages the continuity and injectivity properties from the previous section to formulate a universal approximation result.

\begin{defn}\label{defn:universal approximator}
    We say that a sequence of function classes $(\mathcal{H}_n)_{n\ge1}$ with $\mathcal{H}_n \subset \mathcal{C}^{0}(\RR^n;\RR)$ is a universal approximator if, for every $n\ge 1$ and every compact set $\mathscr{K}_n \subset \RR^n$, the set $\{ f\vert_{\mathscr{K}_n} \;\vert\; f\in\mathcal{H}_n \}$ is dense in $(\mathcal{C}^{0}(\mathscr{K}_n;\RR), \Vert \cdot \Vert_\infty)$.
\end{defn}

Paired with a suitable continuous and injective feature map, such universal approximators can be used for approximation of continuous functions on generic (infinite dimensional) topological spaces.
Here we spell this out for the Volterra signature as a feature map on path space.

Recall from the previous section that there are augmentation maps of the form
$$\widehat{\cdot}:\;  \Vone \to \mathcal{V}^{1}([0,T];\RR^{m+m^{\prime}}), \quad z \mapsto \widehat{z}$$
such that the Volterra lift
$z \mapsto \widehat{\bz}_{0,T}^T:=[1\ast \hat{z}]_{0,T}^T$
becomes injective (e.g. time-augmentation in   \Cref{cor:injective volterra signature z}).
Similarly, under suitable assumptions on the kernel $K$, there are augmentation maps
$$\widehat{\cdot}:\;  \mathcal{C}^{0,1}([0,T];\RR^d) \to \mathcal{C}^{0,1}([0,T];\RR^{d+d^{\prime}}), \quad x \mapsto \widehat{x}$$
such that the Volterra signatures $x \mapsto \VSig{\hat{x}}{K}_{0,T}^T$ become injective (see \Cref{rem:aug_rep_K1} and \cref{prop:injectivity_noise_kernel}).
For the following result we fix any such map.

\begin{prop}\label{thm:universal general}
    Let $(\mathcal{H}_n)_{n\ge1}$ be a universal approximator according to Definition \ref{defn:universal approximator} and let $\mathscr{K} \subset \mathcal{C}^{0,1}([0,T];\RR^d)$ (respectively $\mathscr{K} \subset \Vone$) be compact.
    Then for all continuous functionals $F: \mathscr{K}\rightarrow \RR$ and $\epsilon>0$ there exist $n\ge 1$, $\pi_1, \dots, \pi_n \in T(\RR^{d+d^\prime})$ (resp. $\pi_1,\dots,\pi_n\in T(\RR^{m+m^\prime})$) and $f_n\in \mathcal{H}_n$ such that
    \begin{equation*}
        \sup_{x\in \mathscr{K}}\Big|F(x)-f_n\Big(\la \pi_1, \VSig{\hat{x}}{K}\ra,\, \ldots,\,\la \pi_n,\VSig{\hat{x}}{K}\ra\Big)\Big|<\epsilon,
    \end{equation*}
    respectively
    \begin{equation*}
        \sup_{z\in \mathscr{K}}\Big|F(z)-f_n\Big(\la \pi_1, \widehat{\bz}_{0,T}^T\ra,\, \ldots,\,\la \pi_n,\widehat{\bz}_{0,T}^T\ra\Big)\Big|<\epsilon,
    \end{equation*} where the pairings $\la \pi_j, \VSig{\widehat{x}}{K}\ra$ and $\la \pi_j, \widehat{\bz}_{0,T}^T \ra$ are understood thanks to \Cref{def:extended_TA}.
\end{prop}

\begin{proof}
We treat both cases simultaneously.
Let $e\in \{d+d^\prime,\; m+m^\prime\}$ and let $\Phi$ denote the corresponding injective and continuous (see \Cref{prop:continuity}) feature map, i.e.
$$\Phi(x) := \VSig{\widehat{x}}{K} \in T((\RR^e))\qquad \text{or}\qquad \Phi(z) := \widehat{\bz}_{0,T}^T \in T((\RR^e)).$$
Set $\mathcal{X} := T(\RR^e)$ and $\mathcal{X}^\prime := T((\RR^e))=\prod_{k=0}^\infty (\RR^e)^{\otimes k}$, and equip $\mathcal{X}^\prime$ with the product topology.
By construction, this is the coarsest topology such that the linear maps $\la \pi, \cdot\ra : \mathcal{X}^\prime \to \RR$ are continuous for all $\pi \in \mathcal{X}$.
Consider the class of cylindrical functions
$$
\mathcal{A}_{\mathrm{cyl}}
:=
\Big\{
\mathcal{X}^\prime \to \RR:\;
\bx\mapsto f(\la \pi_1,\bx\ra,\ldots,\la \pi_n,\bx\ra)
\ \Big|\ 
n\ge 1,\ \pi_1,\dots,\pi_n\in \mathcal{X},\ f\in \mathcal{C}^{0}(\RR^n;\RR)
\Big\}.
$$
Clearly $\mathcal{A}_{\mathrm{cyl}}$ is a unital subalgebra of $\mathcal{C}^{0}(\mathcal{X}^\prime;\RR)$ that separates points on $\mathcal{X}^\prime$.

Now define the class of functions on $\mathscr{K}$ obtained by composition with $\Phi$,
$$
\mathcal{A}_\Phi
:=
\{\, f\circ \Phi \;|\; f\in \mathcal{A}_{\mathrm{cyl}}\,\}
\subset \mathcal{C}^{0}(\mathscr{K};\RR).
$$
Since $\Phi$ is continuous, $\mathcal{A}_\Phi$ is a unital subalgebra of $\mathcal{C}^{0}(\mathscr{K};\RR)$.
Crucially, due to the injectivity of $\Phi$ it also separates points on $\mathscr{K}$.
As $\mathscr{K}$ is compact, the Stone--Weierstrass theorem implies that $\mathcal{A}_\Phi$ is dense in $\mathcal{C}^{0}(\mathscr{K};\RR)$ in the uniform norm.
Therefore, for any continuous functional $F:\mathscr{K}\to \RR$ and $\varepsilon>0$ there exist $n\ge 1$, $\pi_1,\dots,\pi_n\in \mathcal{X}$ and $f\in \mathcal{C}^{0}(\RR^n;\RR)$ such that
\begin{equation*}
\sup_{u\in \mathscr{K}}
\Big|
F(u) - f\big(\la \pi_1,\Phi(u)\ra,\ldots,\la \pi_n,\Phi(u)\ra\big)
\Big|
<\frac{\varepsilon}{2}.
\end{equation*}
Set
$$
\mathscr{K}_n
:=
\Big\{
\big(\la \pi_1,\Phi(u)\ra,\ldots,\la \pi_n,\Phi(u)\ra\big)\in \RR^n
\ \Big|\ u\in \mathscr{K}
\Big\}.
$$
Then $\mathscr{K}_n$ is compact as the continuous image of the compact set $\mathscr{K}$.
Since $(\mathcal{H}_n)_{n\ge 1}$ is a universal approximator, there exists $f_n\in \mathcal{H}_n$ such that
\begin{equation*}
\sup_{y\in \mathscr{K}_n}|f(y)-f_n(y)|<\frac{\varepsilon}{2}.
\end{equation*}
The statement now follows by combining the two estimates with the triangle inequality.
\end{proof}

\begin{rem}\label{rem:DNN_rmk}
    For the universal approximator $(\mathcal{H}_n)_{n\ge1}$, we can simply use polynomials, as suggested in the context of It\^{o} signatures in \cite{harang2024universalapproximationnongeometricrough}.
    However, numerical methods for stochastic control problems with classical signatures  have proven successful when using deep neural networks, see  \cite{abi2025hedging,bayer2025control,bayer2023optimal,bayer2025pricing}.
\end{rem}

\begin{rem}
        Recall that $\cV^1([0,T];\RR^m)$ is introduced in Definition \ref{def:volterra_path}. We get a compact subspace $\mathscr{K}$ of $\cV^1([0,T];\RR^m)$ if we replace Definition \ref{def:volterra_path}-(iii) by some Hölder regularity condition on the derivative of $z$. Namely define  $\mathscr{K}$ as
        \[
        \mathscr{K} = \left\{ z \in \cV^1([0,T];\RR^m): 
        \Big  \Vert t \mapsto \frac{\dd}{\dd t}z_t^\tau \Big \Vert_{\cC^{\alpha}} \leq M_{1}
        \text{ and } \Big  \Vert \tau \mapsto z_t^\tau \Big \Vert_{\cC^\beta} \leq M_{2}\right\},
        \] 
        where $\alpha,\beta \in (0,1)$ and $M_{1},M_{2}>0$ are fixed constants. Then it is a classical fact (see e.g~\cite[Corollary 2.96]{Chemin2011}) that $\mathscr{K}$ is compact in $\cV^1([0,T];\RR^m)$.
\end{rem}

Compared with the classical signature, the above result is only partially satisfying, as one hopes for a universal approximation theorem by \emph{linear} functionals.
As the (generally kernel-dependent) algebraic properties of the Volterra signature are out of the scope of this paper, we cannot prove such a theorem in the general case.
However, as we demonstrate below, for the case of exponential kernels, linear approximation is possible, suggesting that such a theorem may hold in a more general setting as well.

Specifically, we consider time-augmented signals $t\mapsto \widehat{x}_t = (t, x_t)^\top \in \mathbb{R}^{d+1}$, together with  exponential matrix kernels seen in \Cref{sec:dyn_exp}, of diagonal form \begin{equation}\label{eq:exp_diag_recalled}K_{\alpha}^\lambda(t,s) := \mathrm{diag}(\alpha_0e^{-\lambda_0(t-s)},\dots,\alpha_de^{-\lambda_d(t-s)}) \in \cL(\mathbb{R}^{d+1},\mathbb{R}^{d+1}), \quad \alpha_i \in \RR\setminus \{0\}, \quad \lambda_i \in \mathbb{R}_+.\end{equation} In the following result we show linear universality for the two cases: \begin{itemize}
    \item[(i)] $\lambda_0=0$, that is the first component of the corresponding Volterra path $z_t^\tau = \int_0^tK_\alpha^\lambda(t,s)\dd \widehat{x}_s$ is time, see also  \Cref{cor:injective volterra signature z}.
    \item[(ii)] $\lambda_0=\cdots =\lambda_d>0$ and $\alpha_0=\cdots = \alpha_d \neq 0$, which is equivalent to the scalar-kernel case $k(t,s)=\alpha e^{-\lambda(t-s)}$, see also  \Cref{prop:mean reverting 1d}.
\end{itemize} %

\begin{thm}\label{thm:lin_univ_exp}
    Denote by $K_{\alpha}^{\lambda} \in L^{\infty,1}(\Delta^2;\cL(\mathbb{R}^{d+1};\mathbb{R}^{d+1}))$ the exponential kernel defined in \eqref{eq:exp_diag_recalled}, such that either $(i)$ or $(ii)$ above holds. Then, for any compact set $\mathscr{K} \subset \cC^{0,1}([0,T], \RR^d)$ and continuous  functional $F: \cC^{0,1}([0,T], \RR^d) \to \RR$, it holds that for all $\varepsilon > 0$, there exists $\pi \in T(\RR^{d+1})$ such that
    \begin{equation}\label{eq:linear_universality_claim}
        \sup_{x\in\mathscr{K}}\left|F(x)-\la \pi,\VSig{\widehat{x}}{K_{\alpha}^{\lambda} }\ra \right|<\varepsilon.
    \end{equation}
\end{thm}

For the case (ii) above, the proof relies on the following technical lemma, showing that we can revert the transformation from the classical signature to the Volterra signature in \Cref{prop:mean reverting 1d}.

\begin{lem}\label{eq:sig_vsig_integral_transform}
Let $X$ be a compact topological space and set $\mathscr{C}:= \mathcal{C}^{0}([0,T] \times X; \RR)$ equipped with the supremum norm $\Vert\cdot\Vert_\infty$.
Further, let $\lambda \in \RR$ and $\alpha> 0$ and define, for any $g\in \mathscr{C}$,
\begin{equation*}
\Psi[g](t,x)
:= g(t,x)
+ \lambda \int_t^T e^{\lambda r} g(r,x)\,dr,
\qquad t\in[0,T],\ x\in X.
\end{equation*}
Then $\Psi$ is an isomorphism on the topological vector space $\mathscr{C}$.
\end{lem}

\begin{proof}
It is evident that $\Psi$ is linear and decomposes as $\Psi = \mathrm{Id} + \Phi$.
One readily estimates
$
\Vert\Phi^{\circ n}[g] \Vert_{\infty} \le \frac{1}{n!}|\lambda|^n e^{n|\lambda|T}T^n \Vert g\Vert_{\infty},
$
implying in particular continuity of $\Phi$ and thus of $\Psi$.
Further, this shows that the Neumann series
$
\Psi^{-1} = \sum_{n=0}^\infty (-1)^n\Phi^{\circ n}
$
converges, thus proving that $\Psi$ is an isomorphism.
\end{proof}

Now we are ready to prove the theorem for linear approximation via the Volterra signature with exponential kernels.

\begin{proof}[Proof of \Cref{thm:lin_univ_exp}]
Let us first assume $(i)$ holds true. We first claim that for any element $\pi \in T(\RR^{d+1})$, there exists a $\ell=\ell(\pi)\in T(\RR^{d+1})$, such that \begin{equation}\label{eq:claim_lin_un}
    \langle \pi, \mathrm{Sig}(\widehat{x})_{0,T}\rangle = \langle \ell(\pi), \VSig{\widehat{x}}{K_\alpha^\lambda}_{0,T}^T \rangle.
\end{equation} Assuming the claim \eqref{eq:claim_lin_un} is true, by universality of the classical signature from time-augmented paths (e.g. \cite[Theorem 5]{litterer2014chen}), we can find $\pi^{\varepsilon}\in T(\RR^{d+1})$ such that \[
\sup_{x\in \cK}|F(x)-\langle \pi^\varepsilon,\mathrm{Sig}(\widehat{x})_{0,T}\rangle | \leq  \varepsilon,
\] and thus in particular $ \sup_{x\in \cK}|F(x)-\langle \ell(\pi^\varepsilon),\VSig{\widehat{x}}{K_\alpha^\lambda}\rangle | \leq \varepsilon.$ To show the claim~\eqref{eq:claim_lin_un}, an application of \Cref{prop:mean_reverting_prop}, for the choices $b_r= e_r, A_r= \alpha_r e_r^\top e_r$ with $0\leq r \leq d$ and matrix $\Lambda = \mathrm{diag}(\lambda_0,\dots,\lambda_d)$, shows that for any $i_1,\dots,i_n \in \{0,\dots,d\}$ 
\begin{multline}\label{eq:mean-rev-id_prof}
    \VSig{\widehat{x}}{K_\alpha^\lambda}_{0,T}^{i_1\cdots i_n,T}  = 
    -\lambda_{i_n} \int_0^T\VSig{\widehat{x}}{K_\alpha^\lambda}_{0,s}^{i_1\cdots i_n,s} \dd s
    + \alpha_{i_n}\int_0^T\VSig{\widehat{x}}{K_\alpha^\lambda}_{0,s}^{i_1\cdots i_{n-1},s} \dd \widehat{x}^{i_n}_s  \\ 
    = -\lambda_{i_n}\VSig{\widehat{x}}{K_\alpha^\lambda}_{0,T}^{i_1\cdots i_n 0,T}+ \alpha_{i_n}\int_0^T\VSig{\widehat{x}}{K_\alpha^\lambda}_{0,s}^{i_1\cdots i_{n-1},s} \dd \widehat{x}^{i_n}_s,
\end{multline} 
where the second equality follows  from the assumption $\lambda_0=0$ and $\widehat{x}^0_t=t$. The claim now easily follows by induction over elements in $T^{n}(\RR^{d+1})$ (see \cref{def:proj_and_trunc}), together with~\eqref{eq:mean-rev-id_prof}. Indeed, for $n=1$ and single letter $\pi=i \in \{0,\dots,d\}$, we have 
\begin{eqnarray}
\langle i, \mathrm{Sig}(\widehat{x})_{0,T}\rangle=\widehat{x}_T^i-\widehat{x}_0^{i} 
&=& \frac{1}{\alpha_{i}}\Big(\VSig{\widehat{x}}{K_\alpha^\lambda}_{0,T}^{i,T}
+\lambda_{i}\VSig{\widehat{x}}{K_\alpha^\lambda}_{0,T}^{i 0,T}\Big ) \nonumber  \\
&=& \langle \ell(i),\VSig{\widehat{x}}{K_\alpha^\lambda}_{0,T}^{,T}\rangle,\label{eq:eq_d}
\end{eqnarray}
for $\ell(i)=\frac{1}{\alpha_{i}}i(\varnothing+\lambda_{i}0)$. Since we can extend \eqref{eq:eq_d} linearly to any linear combination of letters, the claim~\eqref{eq:claim_lin_un} holds for $\pi \in T^{1}(\RR^{d+1})$. Assuming the claim holds on level $n$, for $\pi=i_1\cdots i_{n+1}$ we have \[
\langle \pi,\mathrm{Sig}(\widehat{x})_{0,T}\rangle = \int_0^T \langle i_1\cdots i_{n},\mathrm{Sig}(\widehat{x})_{0,s}\rangle \dd \widehat{x}^{i_{n+1}}_s = \int_0^T \langle \ell(i_1\cdots i_n),\VSig{\widehat{x}}{K_\alpha^\lambda}_{0,s}^{s}\rangle \dd \widehat{x}^{i_{n+1}}_s,
\] where we used the definition of the signature and the induction hypothesis. An application of \eqref{eq:mean-rev-id_prof} then shows
\begin{multline*}
\int_0^T \langle \ell(i_1\cdots i_n),\VSig{\widehat{x}}{K_\alpha^\lambda}_{0,s}^{s}\rangle \dd \widehat{x}^{i_{n+1}}_s \\
=  
\Big \langle \frac{1}{\alpha_{i_{n+1}}}\ell(i_1\cdots i_n)i_{n+1}(\varnothing+\lambda_{i_{n+1}}0),\VSig{\widehat{x}}{K_\alpha^\lambda}_{0,T}^{T}\Big \rangle, 
\end{multline*}
and since we can extend linearly to $\pi \in T^{n+1}(\RR^{d+1})$, the claim \eqref{eq:claim_lin_un} follows by induction.

Now suppose $(ii)$ holds, and let $\Psi$ be the operator defined as in \cref{eq:sig_vsig_integral_transform} with $X = \mathscr{K}$.
Let $F: \mathcal{C}^{0,1}([0,T]; \RR^d) \to \RR$ be continuous and let $\varepsilon>0$ be arbitrarily fixed.
Define the function $f:[0,T]\times \mathcal{C}^{0,1}([0,T]; \RR^d)$ by $f(t,x) := F(x_{t\vee \cdot})$, where $x_{t\vee r} := 1_{[0,t)}(r)x_t + 1_{[t,T]}(r)x_r \in \mathcal{C}^{0,1}([0,T]; \RR^d)$.
We first apply the inverse operator of $\Psi$ from \Cref{eq:sig_vsig_integral_transform} to $f\in \mathscr{C}$ to obtain the transformed function $g := e^{\lambda T}\Psi^{-1}[f] \in \mathscr{C}$.
Then, by a time reversed formulation of the universal approximation theorem for signatures on the stopped paths space %
there exists $\pi \in T(\RR^{d+1})$ such that
\[
\Vert g - \langle \pi, \Sig{\widehat{x}}_{\cdot,T} \rangle \Vert_{\infty}
= \sup_{x\in \mathscr{K}}\sup_{t\in[0,T]} \bigl|g(t,x) - \langle \pi, \Sig{\widehat{x}}_{t, T}\rangle\bigr|
< \Vert\Psi\Vert^{-1}e^{\lambda T}\varepsilon,
\]
where $\Vert\Psi\Vert$ denotes the operator norm of $\Psi$ on $(\mathscr{C},\Vert \cdot\Vert_{\infty})$.
Indeed, this follows by an entirely analogous argument to the stopped path space situation (e.g. in \cite[Proposition 3.3]{bayer2025control}) where one verifies that the maps $x|_{[t,T]}\mapsto \langle \pi, \Sig{\widehat{x}}_{t,T}\rangle$, $\pi\in T(\RR^{d+1})$, form a point-separating unital subalgebra of the continuous functions on the forward-started path space $\bigcup_{t\in[0,T]}\cC^{0,1}([t,T];\RR^d)$ equipped with the distance $$d(x|_{[t,T]},y|_{[s,T]}) := |t-s| + \Vert x_{\cdot\vee t}-y_{\cdot\vee s}\Vert_{\cC^{0,1}}.$$
Next note that by \Cref{prop:mean reverting 1d} and the definition of $\Psi$ in \Cref{eq:sig_vsig_integral_transform} it holds
\[
\langle \pi, \VSig{\widehat{x}}{K_{\alpha,\lambda}}_{0,T} \rangle
:= e^{-\lambda T}\Psi[\langle \pi, \mathrm{Sig}(\,\widehat{\cdot}\,)_{\cdot,T} \rangle](0,x), \qquad  x \in \mathscr{K}.
\]
Finally, we can estimate
\begin{multline*}
    \sup_{x\in \mathscr{K}}\bigl|F(x) - \langle \pi, \VSig{\widehat{x}}{K_{\alpha,\lambda}}_{0,T} \rangle\bigr|
    = \sup_{x\in \mathscr{K}}\bigl|f(0, x) - \langle \pi, \VSig{\widehat{x}}{K_{\alpha,\lambda}}_{0,T} \rangle\bigr| \\
    \le \sup_{x\in \mathscr{K}} e^{-\lambda T}\bigl|\Psi[g - \langle \pi, \Sig{\,\widehat{\cdot}\,}_{\cdot,T} \rangle](0, x)\bigr| 
    \le e^{-\lambda T}\Vert\Psi\Vert \,\Vert g - \langle \pi, \Sig{\,\widehat{\cdot}\,}_{\cdot,T} \rangle \Vert_{\infty} 
    \le \varepsilon.
\end{multline*}
This concludes the proof of \eqref{eq:linear_universality_claim} in case (ii).
\end{proof}

\subsection{Kernel trick}\label{sec:kernel-trick}

Kernel methods provide a flexible framework for learning with structured data by specifying similarity through a positive definite kernel, equivalently an inner product in an associated reproducing kernel Hilbert space (RKHS) \cite{aronszajn1950theory,scholkopf2002learning}. This viewpoint underpins classical algorithms such as support vector machines and Gaussian processes
and allows nonlinear learning problems to be treated with linear methods in feature space.
In the context of time series, the \emph{signature kernel} \cite{Kiraly2019, Salvi2021} can be viewed as a canonical choice when one seeks a positive definite kernel that simultaneously (i) captures sequential order, (ii) behaves well under high-frequency sampling, and (iii) is invariant under monotone reparametrizations (see \cite[Section~5]{Kiraly2019} for a detailed comparison with other popular kernels for sequential data \cite{chan2005probabilistic,berndt1994using,cuturi2011autoregressive,moreno2003kullback}).
However, the (standard) signature kernel does not by itself enforce a specific memory structure, such as a recency bias. Since the Volterra signature $\VSig{x}{K}$ already incorporates the memory profile $K$ at the feature level, we transfer this structure to kernel methods by defining the \emph{Volterra signature kernel} as induced by the inner product. In this way, the similarity between two paths is modulated by $K$, yielding a positive definite kernel that inherits the temporal weighting of events.

\begin{rem}
    We are now facing the embarrassing---but highly intriguing---situation of having two kernels present on the same stage: a temporal (or \emph{memory}) kernel
    $K:\Delta^2 \to \mathcal{L}(\RR^d,\RR^m)$, and an induced kernel on the (Volterra-)path space.
    To keep these two notions from stepping on each other's toes, we denote the latter by $\kappa$.
    Things get even more intriguing if we further---as suggested in the classical signature setting in \cite{Kiraly2016}---lift the Volterra path into yet another RKHS \emph{feature space}. %, say
    %$\hat{z}^\tau_t = \int_0^t \phi(\dot{z}^\tau_s)\,\dd s$
    %for some continuous injective map $\phi:\RR^m \to \cH$ into an RKHS $\cH$ (e.g.\ the feature space associated with an RBF kernel).
    We will find some resolution of this kernel-within-kernel situation in \Cref{rem:static_kernel_lift} below.
\end{rem}
Recall that the classical signature kernel is defined as the inner product of signatures,
\[
\big\langle \mathrm{Sig}(x), \mathrm{Sig}(y) \big\rangle_{\cThilb},
\qquad x,y \in \cC^{0,1}([0,T];\RR^d),
\]
where the inner product on $\cThilb$ was introduced in \cref{def:inner_pro}. We extend this construction to Volterra features by restricting to suitable subspaces of Volterra paths, denoted $\Vone$ (see \cref{def:volterra_path}), on which the Volterra signature and the induced inner product are well-defined.

\begin{defn}\label{def:Volterra_paths_in_Lp}
    For any $1\leq p \leq \infty$ we denote by $\cV^{1,p}([0,T];\RR^m)$ the space consisting of Volterra paths \[
    z_{t}^\tau= \int_0^tK(\tau,s)\dd x_s \quad \text{with} \quad x \in \cC^{0,1}([0,T];\RR^d) \, \text{ and } \, K \in L^{\infty,p}([0,T];\cL(\RR^d;\RR^m)),
    \] where the space $L^{\infty,p}$ was defined in \eqref{eq:Lp_kernel}. Moreover, for any $z \in \cV^{1,p}$ we extend the norm \eqref{eq:norm} by setting \begin{equation}\label{eq:norm_Vp}
        \Vert z \Vert_{1, p}:= \sup_{t \in [0,T]}\left (\int_0^t\Big |\frac{\dd z_u^t}{\dd u}\Big|^p \dd u \right )^{1/p}.
    \end{equation}
\end{defn}

Note that $\cV^{1,1}= \cV^1$ thanks to \cref{lem:equivalence_def_VP}, and clearly for any $1 \leq p \leq q \leq \infty$ we have the inclusion $\cV^{1,q} \subseteq \cV^{1,p}$. We can now define the Volterra signature kernel.
\begin{defn}\label{def:volterra-kernel}
Let $p>1$ and consider two Volterra paths $z, y \in \cV^{1,p}([0,T];\RR^m)$ with full lifts denoted by $\bz$ and $\by$, see   \cref{def:VP_def_z}. The \emph{Volterra signature kernel} is defined as the inner product
\begin{equation}\label{eq:sig-kernel general}
\kappa(z,y)_{s,t}=\la \bz^s_{0,s}, \by^t_{0,t} \rangle_{\cThilb}, \quad (s,t) \in [0,T]^2.
\end{equation}
\end{defn}
In view of \cref{rem:usual_sig}, for two paths $x,w \in \cC^{0,1}([0,T];\RR^d)$ and  $z_t^\tau = \int_0^t\dd x_s$ and $y_t^\tau= \int_0^t \dd w_s$, the definition  \eqref{eq:sig-kernel general} coincides with the classical signature kernel of $x$ and $w$ seen above. The following lemma shows that the generalization to $\cV^{1,p}$ in \eqref{eq:sig-kernel general} is well-defined.

\begin{lem}\label{lem:well-defined-kernel}
    For any $p>1$ and $z, y \in \cV^{1,p}([0,T];\RR^m)$, the Volterra signature kernel~\eqref{eq:sig-kernel general} is well-defined. Specifically, we have \[
    \sup_{(s,t)\in [0,T]^2} |\kappa(z,y)_{s,t}|< \infty.
    \] 
\end{lem}
\begin{proof}
    This is a direct consequence of the Cauchy-Schwarz inequality for the inner product in \Cref{def:inner_pro} and the infinite radius of convergence on $\cV^{1,p}$ of the Volterra signature, see  \Cref{lem:ROC_examples}.
\end{proof}

It was shown in~\cite{Salvi2021} that the classical signature kernel can be characterized as the unique solution to a hyperbolic Goursat PDE, enabling its computation without explicit truncation. We now extend this kernel trick to the Volterra setting.
\begin{thm}\label{thm:volterra-kernel-equation}
Let $p>1$ and consider two Volterra paths $z, y \in \cV^{1,p}([0,T];\RR^m)$. The Volterra signature kernel $\kappa_{s,t}:=\kappa(z,y)_{s,t}$ in \eqref{eq:sig-kernel general} uniquely solves
\begin{equation}\label{eq:volterra-kernel-equation}
\kappa_{s,t}= 1+ \int_{0}^{s}\int_{0}^{t}\kappa_{u,v}\langle \dot z_u^t,\dot{y}_v^s\rangle\dd u\dd v.
\end{equation}
\end{thm}

\begin{proof} \emph{Existence:} Recall that from the fundamental linear equation for the Volterra signature (see  \eqref{eq:vsig_fundamental_z}), we know
\[
\bz^s_{0,s} \;=\; 1 + \int_0^s \bz^u_{0,u} \otimes  \dd z^s_u, \quad \by^t_{0,t} = 1 + \int_0^t \by^v_{0,v} \otimes \dd y^t_v.
\]
From \Cref{lem:well-defined-kernel}, we also have that the inner product $\langle \bz^\tau_{0,s}, \by^\rho_{0,t} \rangle$ is well-defined. Moreover, using the bilinearity of the inner product and the fact that the unit tensor $1 \in T((\RR^m))$ is orthogonal to all tensors of level $n \ge 1$, we find
\begin{equation}\label{eq:proof-kernel-step1}
\langle \bz^s_{0,s}, \by^t_{0,t} \rangle = 1+ \left\langle \int_0^s \bz^u_{0,u} \otimes \dot{z}_u^s\dd u, \int_0^t \by^v_{0,v} \otimes \dot{y}_v^t\dd v \right\rangle.
\end{equation}
Next we use the property that the inner product commutes with the integral and satisfies the factorization property $\langle A \otimes a, B \otimes b \rangle = \langle A, B \rangle \langle a, b \rangle$ for tensors $A,B$ and vectors $a,b$ (see \Cref{def:inner_pro}). Specifically, we find 
\begin{multline*}
    \langle \bz^s_{0,s}, \bw^t_{0,t} \rangle  = 1+ \left\langle \int_0^s \bz^u_{0,u} \otimes \dot{z}_u^s  \dd u,  \int_0^t \by^v_{0,v} \otimes \dot{y}_v^t\dd v  \right\rangle \\ 
     = 1 + \int_0^t\int_0^s\big \langle \bz^u_{0,u} \otimes \dot{z}_u^s, \by^v_{0,v} \otimes \dot{y}_v^t \big \rangle \dd u \dd v  
     = 1 + \int_0^t\int_0^s\big \langle \bz^u_{0,u},\by^v_{0,v} \big \rangle \big \langle  \dot{z}_u^s,  \dot{y}_v^t \rangle\dd u \dd v.
\end{multline*}
By definition, $\kappa_{u,v}=\langle \bz^u_{0,u}, \by^v_{0,v} \rangle$, and substituting this back into \eqref{eq:proof-kernel-step1} yields the claimed integral equation.

\noindent \emph{Uniqueness:}
Suppose now that there are two continuous solutions to \eqref{eq:volterra-kernel-equation}, and denote by  $\delta:[0,T]^2\to\RR$ their difference. By linearity, $\delta$ satisfies the homogeneous equation
\[
\delta_{s,t}
=
\int_{0}^{s}\int_{0}^{t}  \delta_{u,v}\,\langle \dot z^{\,s}_u, \dot y^{\,t}_v \rangle \,\dd v\,\dd u ,
\qquad (s,t)\in[0,T]^2.
\]
Similarly as in \Cref{lem:ROC_examples}, we apply Hölder's inequality with respect to $p'=\frac{p}{p-1}$ (first in $v$, then in $u$) to obtain
\begin{align*}
|\delta_{s,t}|
&\le \int_{0}^{s}\int_{0}^{t} |\delta_{u,v}| |\dot z^{s}_u| |\dot y^{t}_v| \dd v \dd u \\
&\le \int_{0}^{s} |\dot z^{s}_u|
     \left(\int_{0}^{t}|\delta_{u,v}|^{p'} \dd v\right)^{1/p'}
     \left(\int_{0}^{t}|\dot y^{t}_v|^{p} \dd v\right)^{1/p}
     \dd u \\
&\le \left(\int_{0}^{s}|\dot z^{s}_u|^{p} \dd u\right)^{1/p}
     \left(\int_{0}^{t}|\dot y^{t}_v|^{p} \dd v\right)^{1/p}
     \left(\int_{0}^{s}\int_{0}^{t}|\delta_{u,v}|^{p'} \dd v \dd u\right)^{1/p'} \\
&\le \|z\|_{1,p}\|y\|_{1,p}
     \left(\int_{0}^{s}\int_{0}^{t}|\delta_{u,v}|^{p'} \dd v \dd u\right)^{1/p'},
\end{align*}
where we recall $\Vert \cdot \Vert_{1,p}$ was defined in \eqref{eq:norm_Vp}.

Raising the previous inequality to the power $p'$ and iterating gives, for every $n\ge1$ and all $s,t\in[0,T]$,
\[
\sup_{(u,v)\in[0,s]\times[0,t]}|\delta_{u,v}|
\le
\frac{\Big(\|z\|_{\infty,p}\,\|y\|_{\infty,p}\,s^{1/p'}t^{1/p'}\Big)^{n}}{(n!)^{2/p'}}
\sup_{(u,v)\in[0,s]\times[0,t]}|\delta_{u,v}|.
\]
Letting $n\to\infty$ implies $\sup_{(u,v)\in[0,s]\times[0,t]}|\delta_{u,v}|=0$ for all $s,t$, hence $\delta\equiv 0$.
\end{proof}

Similar to before, depending on the application, one may view the Volterra signature kernel as being induced by the feature map $ \cC^{0,1} \ni x \mapsto \VSig{x}{K}$, see also   \eqref{eq:vsig_feature_c}, and treat the Volterra kernel $K$ as a modeling component. To this end, for any kernel $K \in L^{\infty,p}([0,T];\cL(\RR^d;\RR^m))$ with $p>1$, we define for $(s,t) \in [0,T]$ \begin{equation}\label{eq:sig-kernel-noise}
    \kappa^K(x,w)_{s,t} := \la \VSig{x}{K}_{0,s}^s, \VSig{w}{K}_{0,t}^t \ra_{\cThilb},  \quad x,w \in \cC^{0,1}([0,T];\RR^d),
\end{equation} and the following kernel trick is a direct consequence of \Cref{thm:volterra-kernel-equation}.

\begin{cor}
    Let $p>1$ and fix a kernel $K \in L^{\infty,p}([0,T];\cL(\RR^d;\RR^m))$. Then the kernel $\kappa^K$ in~\eqref{eq:sig-kernel-noise} uniquely solves  \begin{equation}\label{eq:volterra-kernel-equation-noise}
    \kappa^K_{s,t} = 1+ \int_{0}^{s}\int_{0}^{t}  \kappa^K_{u,v}\langle K(s,u )\dot{x}_u, K(t,v)\dot{y}_v\rangle \dd u\dd v, \quad s,t \in [0,T]^2.
    \end{equation}
\end{cor}

We conclude this section with a characterization of the Volterra signature kernel \eqref{eq:sig-kernel-noise}, restricted to finite state space kernels considered in \Cref{sec:dyn_exp}, that is 
\begin{equation}\label{eq:exp_kernel_recalled}
K_{A,b}^{\Lambda}(t,s) = \sum_{r=1}^q (\mathbf{1}^{\top} e^{-\Lambda(t-s)}b_r) A_r \in \cL(\RR^d;\RR^m),  
\end{equation} 
where $\Lambda \in \RR^{R\times R}$, $A_r \in \RR^{m\times d}$ and $b_r \in \RR^{R}$ 
(see Definition \ref{def:finite_state_space_kernels}). In \Cref{prop:mean_reverting_prop} we characterized Volterra signatures $\VSig{x}{K_{A,b}^\Lambda}$ with a system of mean-reverting equations in the tensor algebra. In the following result we show that the corresponding signature kernel can be characterized through a system of Goursat-PDEs. To this end, given two signals $x,w\in \cC^{0,1}([0,T];\RR^d)$, let us introduce the coefficient matrix \begin{equation}\label{eq:alpha_gamma_2}
    \gamma(s,t)\in \mathbb{R}^{R \times R},  \qquad \gamma_{ij}(s,t) = \la (B.\dot x_s)_i, (B.\dot w_t)_j \ra, 
\end{equation} where we recall that the notation  $(B.\dot{x}_t)$ was introduced in \eqref{eq:vector_prod_lifts}.
\begin{thm}\label{thm:kernel_pde_ffsk}
           Consider the matrix-valued functions $\bK,\mathbf{\Psi},\mathbf{\Phi}:[0,T]\rightarrow \mathbb{R}^{R\times R}$, defined as solutions to the following PDE system on $[0,T]^2$
\begin{align}
        \frac{\partial ^2 \bK_{s,t}}{\partial t\partial s} & = \eta_{s,t}\gamma(s,t) + \Lambda \bK_{s,t} \Lambda^\top -  \Lambda\mathbf{\Psi}_{s,t}-\mathbf{\Phi}_{s,t}\Lambda \label{eq:PDE1} \\ 
        \frac{\partial \mathbf{\Psi}_{s,t}}{\partial s} & = -\Lambda \mathbf{\Psi}_{s,t} + \gamma(s,t)\eta_{s,t}, \label{eq:PDE2} \\ \frac{\partial  \mathbf{\Phi}_{{s,t}}}{\partial t} &= -\mathbf{\Phi}_{s,t}\Lambda^\top+ \eta_{s,t}\gamma(s,t),\label{eq:PDE_3}
    \end{align} where $\eta_{s,t}:=1+\mathbf{1}^\top \mathbf{K}_{s,t}\mathbf{1}$ and
     with the  boundary conditions \[
    \bK_{0,0}=\bK_{0,t}=\bK_{s,0}= \mathbf{\Psi}_{0,t} = \mathbf{\Phi}_{s,0}=\mathbf{0}, \qquad (s,t) \in [0,T]^2,
    \] where we recall that $\gamma$ is defined by \eqref{eq:alpha_gamma_2}. Then, for two paths $x,w\in \cC^{0,1}([0,T];\RR^d)$, and a kernel $K$ given by \eqref{eq:exp_kernel_recalled}, the kernel $\kappa^K$ in \eqref{eq:sig-kernel-noise} can be expressed as \begin{equation}\label{eq:ffsk_sig_kernel}\kappa^K_{s,t} = \eta_{s,t}.\end{equation} 
\end{thm}

\begin{proof}
Let us denote by $\bZ$ (resp. by  $\bY$) the lift in \Cref{prop:mean_reverting_prop} with respect to the path $x$ (resp. the path $w$). From the latter proposition we know that $\kappa_{s,t}= \langle 1+\sum_{\ell=1}^R\bZ_{s,t}^\ell,1+\sum_{\ell=1}^R\bY_{s,t}^\ell \rangle$, where $\bZ$ and $\bY$ uniquely solve \begin{equation}\label{eq:mean_reverting_2}
    \bZ_{s,s}= \mathbf{0}, \quad \dd \bZ_{s,t} = -\Lambda. \bZ_{s,t}\dd t + \left (1+\sum_{i=1}^R\bZ^i_{s,t} \right)\otimes \dd (B.x_t).
    \end{equation}
    and
\begin{equation}\label{eq:mean_reverting_2_1}
    \bY_{s,s}= \mathbf{0}, \quad \dd \bY_{s,t} = -\Lambda. \bY_{s,t}\dd t + \left (1+\sum_{i=1}^R\bY^i_{s,t} \right)\otimes \dd (B.w_t),
    \end{equation} respectively. Relying on the bi-linearity of the inner product, we write \begin{equation}\label{eq:kernel_ffsk_dec}
    \kappa_{s,t}= 1 + \sum_{i,j=1}^R\la \bZ^i_{0,s}, \bY_{0,t}^j \ra =: 1+ \sum_{i,j=1}^R \widetilde{\bK}_{s,t}^{ij}.
\end{equation} 
In order to get a more explicit expression for $\widetilde{\bK}_{s,t}^{ij}$ above, let us recall (similarly to what we did in the derivation of the integral equation in \cref{thm:volterra-kernel-equation}) that from~\eqref{eq:mean_reverting_2}, integrating
against time and using $\bZ_{0,0}=\mathbf{0}$, we obtain the integral form
\begin{equation}\label{eq:Zintegral}
    \bZ_{0,t}^i
    = -\sum_{n=1}^R \Lambda_{in}\int_0^t \bZ_{0,v}^n \,\dd v
      + \int_0^t \Bigl(1+\sum_{\ell=1}^R \bZ_{0,v}^\ell\Bigr)\otimes \dd(B.\dot{x}_v)_i,
    \qquad 1\le i\le R,
\end{equation} and similarly for $\bY$ relying on~\eqref{eq:mean_reverting_2_1}. 
Inserting the integral form into the definition $\widetilde{\bK}^{ij}_{s,t}
:= \langle \bZ^i_{0,s}, \bY^j_{0,t}\rangle$ and exchanging the inner product with
integration (as in the proof of \cref{thm:volterra-kernel-equation}), we find
\begin{multline} \label{eq:dts_expand}
    \frac{\partial^2 \widetilde{\bK}^{ij}}{\partial t\,\partial s}
    = \partial_s\partial_t\,\Bigl\langle
         -\sum_{n=1}^R \Lambda_{in}\int_0^s \bZ_{0,u}^n\,\dd u
         +\int_0^s \Bigl(1+\textstyle\sum_\ell \bZ_{0,u}^\ell\Bigr)\otimes \dd(B.\dot{x}_u)_i, \\
         -\sum_{p=1}^R \Lambda_{jp}\int_0^t \bY_{0,v}^p\,\dd v
         +\int_0^t \Bigl(1+\textstyle\sum_\ell \bY_{0,v}^\ell\Bigr)\otimes \dd(B.\dot{w}_v)_j
       \Bigr\rangle.
\end{multline}
Expanding by bilinearity of the inner product, the above expression can be decomposed as:
\begin{equation}\label{d1}
\frac{\partial^2 \widetilde{\bK}^{ij}}{\partial t\,\partial s}
=T_{1}+\cdots+T_{4} \, ,
\end{equation}
where the terms $T_{1},\ldots,T_{4}$ are respectively defined by
\begin{eqnarray}
T_{1}&=& \sum_{n,p=1}^R \Lambda_{in}\Lambda_{jp}\,
        \partial_s\partial_t \Bigl\langle
          \int_0^s \bZ_{0,u}^n\,\dd u,\;
          \int_0^t \bY_{0,v}^p\,\dd v
        \Bigr\rangle\\
T_{2}&=& -\sum_{n=1}^R \Lambda_{in}\,
       \partial_s\partial_t\Bigl\langle
         \int_0^s \bZ_{0,u}^n\,\dd u,\;
         \int_0^t \Bigl(1+\textstyle\sum_\ell \bY_{0,v}^\ell\Bigr)\otimes\dd(B.\dot{w}_v)_j
       \Bigr\rangle\\
T_{3}&=& -\sum_{p=1}^R \Lambda_{jp}\,
       \partial_s\partial_t\Bigl\langle
         \int_0^s \Bigl(1+\textstyle\sum_\ell \bZ_{0,u}^\ell\Bigr)\otimes\dd(B.\dot{x}_u)_i,\;
         \int_0^t \bY_{0,v}^p\,\dd v
       \Bigr\rangle\\
T_{4}&=& \partial_s\partial_t\Bigl\langle
         \int_0^s \Big(1+\textstyle\sum_\ell \bZ_{0,u}^\ell\Big)\otimes\dd(B.\dot{x}_u)_i,\;
         \int_0^t \Bigl(1+\textstyle\sum_\ell \bY_{0,v}^\ell\Bigr)\otimes\dd(B.\dot{w}_v)_j
       \Bigr\rangle
\end{eqnarray}
Now the terms $T_{1},\ldots,T_{4}$ can be differentiated in a very similar way. 
Indeed, differentiating first in $t$ and then in $s$, and writing out the $k$-th components of $(B.\dot{w}_t)_j$, we get
\begin{equation}\label{eq:T2_eval}
    T_2
    = -\sum_{n=1}^R\Lambda_{in}\sum_{k=1}^m
        \Big \langle \bZ_{0,s}^n,\;\Big (1+\sum_{p=1}^R\bY_{0,t}^p\Big )\otimes e_k \Big \rangle\,(B.\dot{w}_t)_{jk}
    = -\sum_{n=1}^R\Lambda_{in}\sum_{k=1}^m
        \psi^{n,k}_{s,t}\,\beta_{jk}(t),
\end{equation}
where we used the definitions 
\begin{equation}\label{d2}
\psi^{n,k}_{s,t}:=\Big \langle \bZ_{0,s}^n,\Big (1+\sum_{p=1}^R\bY_{0,t}^p\Big )\otimes e_k \Big \rangle,
\quad\text{and}\quad
\beta_{jk}(t):=(B.\dot{w}_t)_{jk}.
\end{equation}
In the same way (leaving the tedious details to the patient reader), we find
\begin{eqnarray*}
 T_1
    &=& \sum_{n,p=1}^R \Lambda_{in}\Lambda_{jp}\,\langle \bZ_{0,s}^n,\,\bY_{0,t}^p\rangle
    = \sum_{n,p=1}^R \Lambda_{in}\Lambda_{jp}\,\widetilde{\bK}_{s,t}^{np} \\
  T_3
    &=& -\sum_{p=1}^R\Lambda_{jp}\sum_{k=1}^m
        \Big \langle \Big (1+\sum_{n=1}^R\bZ_{0,s}^n \Big )\otimes e_k,\;\bY_{0,t}^p\Big \rangle\,(B.\dot{x}_s)_{ik}
    = -\sum_{p=1}^R\Lambda_{jp}\sum_{k=1}^m
        \varphi^{k,p}_{s,t}\,\alpha_{ik}(s) \\
 T_4
    &=& \Bigl(1+\sum_{p,q}\widetilde{\bK}^{pq}_{s,t}\Bigr)
      \langle (B.\dot{x}_s)_i,(B.\dot{w}_t)_j\rangle
    = \kappa_{s,t}\,\gamma_{ij}(s,t),  
\end{eqnarray*}
where we used the notation $\gamma_{ij}$ in \eqref{eq:alpha_gamma_2}, and \eqref{eq:alpha_gamma_2})
\begin{equation*}
\varphi^{k,p}_{s,t}:= \Big \langle \Big ( 1+ \sum_{n=1}^R\bZ_{0,s}^n\Big )\otimes  e_k,\bY_{0,t}^p\Big \rangle ,
\quad\text{and}\quad
\alpha_{ik}(s):=(B.\dot{x}_s)_{ik},
\end{equation*}
and where for the computation of $T_{4}$ we have resorted to the fact that $\kappa_{s,t}=1+\sum_{p,q}\widetilde{\bK}^{pq}_{s,t}$ (see~\eqref{eq:kernel_ffsk_dec}). Now gathering the terms $T_{1},\ldots,T_{4}$ and plugging those expressions into~\eqref{d1}, we discover that
\begin{multline}\label{eq:PDE1_proof}
\frac{\partial^2 \widetilde{\bK}^{ij}}{\partial t\,\partial s}
    = \kappa_{s,t}\,\gamma_{ij}(s,t) + \sum_{n,p=1}^R
           \Lambda_{in}\Lambda_{jp}\,\widetilde{\bK}_{s,t}^{np}\\
           -\sum_{pn=1}^R\sum_{k=1}^m \Big ( \Lambda_{in}\,\psi^{n,k}_{s,t}\,\beta_{jk}(t)
               + \Lambda_{jn}\,\varphi^{k,n}_{s,t}\,\alpha_{ik}(s)
         \Big ).
\end{multline}
We still have to find a more explicit expression for the terms $\psi^{n,k}$ and $\varphi^{k,n}$ in~\eqref{eq:PDE1_proof}. To this aim, we apply the same techniques as in the above computations. That is  starting from the definition~\eqref{d2} of $\psi^{n,k}$ we get \begin{align}
\psi^{n,k}_{s,t}&= \left \la \sum_{\ell=1}^R-\Lambda_{nl}\int_0^s \bZ_{0,u}^\ell \dd u + \int_0^s (1 + \sum_{\ell=1}^R\bZ_{0,u}^\ell) \otimes \alpha_{n}(u)  \dd u, \Big (1+\sum_{p=1}^R\bY_{0,t}^p \Big ) \otimes e_k \right \rangle \nonumber \\ & =\sum_{\ell=1}^R -\Lambda_{n\ell } \int_0^s \Big \langle \bZ^\ell_{0,u}, \Big (1+\sum_{p=1}^R\bY_{0,t}^p \Big ) \otimes e_k \Big \rangle \dd u + \int_0^s \eta_{u,t}  \langle \alpha_{n}(u),e_k \rangle \dd u \nonumber \\ & = \sum_{\ell=1}^R-\Lambda_{n\ell} \int_0^s \psi^{\ell,k}_{u,t} \dd u + \int_0^s \eta_{u,t}\alpha_{nk}(u) \dd u,
\nonumber
\end{align}
 so that in particular \begin{equation}\label{eq:PDE2_proof}
 \psi^{n,k}_{0,t} = 0, \quad \partial_s \psi^{n,k}_{s,t}= -\sum_{\ell=1}^R \Lambda_{n\ell} \psi^{\ell,k}_{s,t} + \eta_{s,t} \alpha_{nk}(s).
 \end{equation} Similarly one finds \begin{equation}
     \label{eq:PDE_3_proof}\varphi^{k,n}_{s,0}=0, \quad 
 \partial_t \varphi^{k,n}_{s,t}= -\sum_{\ell=1}^R \Lambda_{n\ell}\varphi^{k,\ell}_{s,t}+\eta_{s,t} \beta_{nk}(t) .
 \end{equation}
 Now using \eqref{eq:PDE2_proof} for the $\widetilde{\Psi}_{s,t}^{n,j}=\sum_{k=1}^m
        \psi^{n,k}_{s,t}\,\beta_{jk}(t)$ appearing in \eqref{eq:T2_eval} for $1\leq n,j \leq R$, we find $$\widetilde{\Psi}_{0,t}^{n,j}=0, \qquad \partial_s\widetilde{\Psi}_{s,t}^{n,j} = -\sum_{\ell=1}^R \Lambda_{n\ell} \widetilde{\Psi}_{s,t}^{\ell,k}+\eta_{s,t}\Big (\sum_{k=1}^m \alpha_{nk}(s)\beta_{jk}(t)\Big)= -\sum_{\ell=1}^R \Lambda_{n\ell} \widetilde{\Psi}_{s,t}^{\ell,k}+\eta_{s,t}\gamma_{nj}(s,t).$$ Writing $\widetilde{\mathbf{\Psi}}= (\widetilde{\psi}^{n,j})_{1\leq n,j \leq R}$, the last equations show that $\widetilde{\mathbf{\Psi}}$ solves \eqref{eq:PDE2}. Applying the same treatment for $\widehat{\varphi}^{p,i}_{s,t}:=\sum_{k=1}^m
        \varphi^{k,p}_{s,t}\,\alpha_{ik}(s)$ together with \eqref{eq:PDE_3_proof}, leads to the solution $\widetilde{\mathbf{\Phi}}=(\widetilde{\varphi}^{p,i})_{1\leq p,i \leq R}$ of \eqref{eq:PDE_3}. In particular, we have shown that the triplet  $(\widetilde{\bK},\widetilde{ \mathbf{\Psi}},\widetilde{\mathbf{\Phi}})$ indeed solves the system~\eqref{eq:PDE1}-\eqref{eq:PDE_3}, and thus that  $\widetilde{\bK}=\bK$. Hence comparing \eqref{eq:ffsk_sig_kernel} and \eqref{eq:kernel_ffsk_dec}, our theorem is proved. 
 \end{proof}

We end this section with several remarks.
\begin{rem}
    The system of Goursat PDEs \eqref{eq:PDE1}--\eqref{eq:PDE_3} generalizes the PDE considered for classical signatures in \cite{Salvi2021}. The latter is recovered by choosing $R=q=b_1=1$, $\Lambda=0$, and $A=\mathrm{Id}$, that is, by taking the trivial kernel $K(t,s)\equiv \mathrm{Id}$ in \eqref{eq:exp_kernel_recalled}. More importantly, the numerical schemes and computational ideas developed in \cite{Salvi2021} can be extended to the present setting, yielding efficient algorithms for the computation of $\kappa^K$. We refer the interested reader to the accompanying paper \cite{ii_part} for further details,.
\end{rem}

\begin{rem}
Notice that the auxiliary equations \eqref{eq:PDE2}--\eqref{eq:PDE_3} admit the explicit representations
\[
\mathbf{\Psi}_{s,t}
=
\int_0^s e^{-\Lambda(s-u)}\eta_{u,t}\gamma(u,t)\,\dd u,
\quad
\mathbf{\Phi}_{s,t}
=
\int_0^t \gamma(s,u)\eta_{s,u} e^{-\Lambda^\top(t-u)}\,\dd u,
\qquad (s,t)\in[0,T]^2.
\]
In particular, once a grid of \([0,T]^2\) is fixed, the auxiliary variables can be propagated along the grid by using the above integral representations, which leads to efficient finite-difference schemes, see \cite{ii_part}.
\end{rem}

\begin{rem}\label{rem:static_kernel_lift}
    %As emphasized in Remark~\ref{rem:RKHS_lift}, 
    One of the key advantages of the signature kernel trick introduced in \cite{Salvi2021} is that one may first pre-process the data via a feature map $\RR^d \rightarrow \cH$, where $\cH$ is a potentially infinite-dimensional feature space, without losing tractability of the resulting signature kernel. While referring to \cite[Section 5.1]{ii_part} for details, let us now illustrate how this idea applies to the PDE system in Theorem~\ref{thm:kernel_pde_ffsk}. For simplicity, let $q=1$ in \eqref{eq:alpha_gamma_2} so that  $K(t,s)=k(t,s)A$, where $k(t,s)=\mathbf{1}^\top e^{-\Lambda(t-s)}b$ is a scalar kernel and $A: \RR^d \rightarrow \RR^m$ is a linear map. Now suppose we replace the matrix $A$ by some static feature map $A:\RR^d \to \cH$, where $\cH$ is an RKHS induced by some kernel $\kappa_\cH$. Although one may formally define the Volterra signature associated with the $\cH$-valued paths $X_t=A(x_t)$ and scalar kernel $k$, its direct computation is infeasible, since $\cH$ is in general very high- or infinite-dimensional. However, the situation is different for the corresponding kernel $\kappa_{s,t} = \la \VSig{X}{k},\VSig{W}{k}\ra$. Indeed, arguing as in Proposition~\ref{prop:mean_reverting_prop} and Theorem~\ref{thm:kernel_pde_ffsk}, it is natural to expect that the PDE system \eqref{eq:PDE1}-\eqref{eq:PDE_3} continues to characterize this quantity, with respect to the coefficients
    \[
    \gamma_{ij}(s,t):= b^ib^j\la \dot{X}_s,\dot{W}_t \ra_\cH, \qquad (s,t) \in [0,T]^2, \quad 1\leq i,j \leq R,
    \]
    where we recall that $b=b_1\in \mathbb{R}^R$ arises from the kernel definition \eqref{eq:exp_kernel_recalled}. Assuming efficient evaluation of $\kappa_\cH$, any numerical scheme for the PDE system \eqref{eq:PDE1}-\eqref{eq:PDE_3} can be transferred directly to the lifted setting, without increasing its numerical complexity. We introduce a predictor-corrector finite-difference scheme in the accompanying article \cite{ii_part}; see \cite[Algorithm 11]{ii_part}.
\end{rem}

\section{Applications}\label{sec:learning Volterra sde}
In this section, we illustrate the versatility of the $\mathrm{VSig}$ feature map in several applications with synthetic and real-world data. All our applications can be formulated as supervised learning problems with sequential data $x=(x_t)_{t \in \mathbb{T}}$, where $\mathbb{T}$ denotes the time index set. More precisely, we observe input-output pairs $\{(x^{(i)},y^{(i)})\}_{i=1}^M$, which are related through an unknown function $f$. We begin with a synthetic experiment in Section~\ref{sec:VSDE_ex}, where $f$ is the solution map of a \emph{stochastic Volterra equation} driven by Brownian motion, which is learned by regularized least-squares regression on the Brownian Volterra signature. Moving to real-world data, in Section~\ref{sec:spx} we consider the problem of forecasting future S\&P500 realized volatility from past daily log-prices. Similar to the first example, this is achieved by regression on the Volterra signature of the log-prices. Finally, in Section~\ref{sec:classification} we consider multivariate time-series classification on various UEA datasets \cite{bagnall2018uea}. For this, we employ a support vector machine (SVM) classifier based on the Volterra signature kernel introduced in Section~\ref{sec:kernel-trick}.

As motivated in the introduction, the signature transform $x\mapsto \mathrm{Sig}(x)$, see \cref{rem:usual_sig}, provides a powerful feature map and has been successfully employed for various statistical learning problems with sequential data; see, for instance, \cite{Chevyrev2016} and the references therein. The goal of this section is to illustrate that incorporating a kernel $K$ into the lift, namely considering $x\mapsto \VSig{x}{K}$, is highly beneficial for the learning tasks outlined above. In the sequel, we denote by $\texttt{VSig}_K$ a model built on the feature $\VSig{\cdot}{K}$ for some given kernel $K$, while in the special case $K=\mathrm{Id}$ we write \texttt{Sig}.

For all methods, we split the data into training and test sets, and select the involved hyperparameters by cross-validation on the training data. The detailed implementations of the experiments conducted in this article can be found at \url{https://github.com/lucapelizzari/Volterra_signature_learning}. The code relies on the package \texttt{tensordev}, see also  \url{https://github.com/hagerpa/tensordev}, which supports all Volterra signature algorithms presented in the accompanying article \cite{ii_part}.

%Apart from the regularizer in the regression, and the parameter in the SVM, for our real-world application we also treat the kernel $K$ as a hyperparameter - belonging to the family of finite-state space kernels introduced in Section~\ref{sec:dyn_exp}.

\subsection{Linear Volterra SDE dynamics}\label{sec:VSDE_ex} 
We begin with a synthetic toy example involving non-Markovian data. Let \(k \in L^{\infty,1}(\Delta^2;\mathbb{R})\) be a scalar kernel, $B$ a one-dimensional Brownian signal, and denote by $Y$ the solution to the linear Volterra SDE (see, e.g., \cite{Pr85})
\begin{equation}\label{eq:SDE}
Y_t =Y_0 + \int_0^t (b_0+b_1Y_s)k(t,s)\dd s + \int_0^t(\sigma_0+\sigma_1Y_s)k(t,s)\dd B_s, \quad 0<t \leq T,
\end{equation} for some $Y_0,b_0,b_1,\sigma_1,\sigma_2 \in \RR$. Given i.i.d. samples of the pairs $(B,Y)$, we aim to learn the solution map $B \mapsto f(B)=Y$. To this end, let us further denote by $\bar{B}$ the piecewise linear interpolation of $(t,B_t)$ on a fixed grid \(\mathbb{T} = \{0 = t_0 < \dots < t_N = T\}\) with \(N \in \mathbb{N}\), and by $\bar{Y}$ the corresponding solution to \eqref{eq:SDE} driven by $\bar{B}$. We consider an empirical risk minimization for linear functionals $\Psi_\theta(\bx)= \la \theta, \bx \ra$ as discussed before, and a ridge regression over a subset of grid-points $\cD \subseteq \TT$, that is \begin{equation} \label{eq:SDE_loss}\argmin_{|\theta| \leq L} \frac{1}{|\mathcal{D}|\times M}
\sum_{i=1}^{M}\sum_{t_n \in \mathcal{D}}
\Big(\bar{Y}^{(i)}_{t_n}-\big \langle \theta, \VSig{\bar{B}^{(i)}}{K}_{0,t_n}^{t_n} \big \rangle\Big)^2+ \eta \Vert \theta \Vert_{\ell^2}, \quad \eta \geq 0.
\end{equation} 
We compare three models for $K$ in \eqref{eq:SDE_loss}: \begin{itemize}
    \item \texttt{Sig} $(K=\mathrm{Id})$: Expanding classical SDE solution with iterated integrals of time-augmented Brownian motion is very natural for $k \equiv 1$ in \eqref{eq:SDE} (e.g.\ via stochastic Taylor expansions \cite{Kloeden1991}), and was recently considered for non-explosive kernel $k$ in~\cite{jaber2024pathdependentprocessessignatures}, where the authors show how to explicitly construct infinite expansions for linear equations \eqref{eq:SDE}. 
    \item $\texttt{VSig}_k$ ($K=k$): In a situation where the underlying  kernel $k$ is known, our expansion  in \cref{prop:linear_VCDE} naturally suggests to choose $K=k$ in \eqref{eq:SDE_loss}. While (at least for non-singular kernels) solutions to \eqref{eq:SDE} admit expansions in both the classical and the Volterra signature, incorporating the kernel into the signature is expected to be advantageous when working with low truncation levels. 
    \item $\texttt{VSig}_{k_\lambda }$ $(K=k_\lambda)$: In real-world applications $k$ is often not explicitly known, but it might still be beneficial to incorporate a kernel weighting the history of the path. We illustrate this here for the parametric class $\{k_\lambda(t,s)=e^{-\lambda (t-s)}: \lambda \in \mathbb{R}_+\}$ (see Section~\ref{sec:dyn_exp}), where $\lambda$ is treated as a hyperparameter in the learning problem \eqref{eq:SDE_loss}. This approach will be central in \Cref{sec:spx} for financial time series, so that we already illustrate it here.  
\end{itemize}

\begin{figure}
  \centering
  \includegraphics[width=1\textwidth]{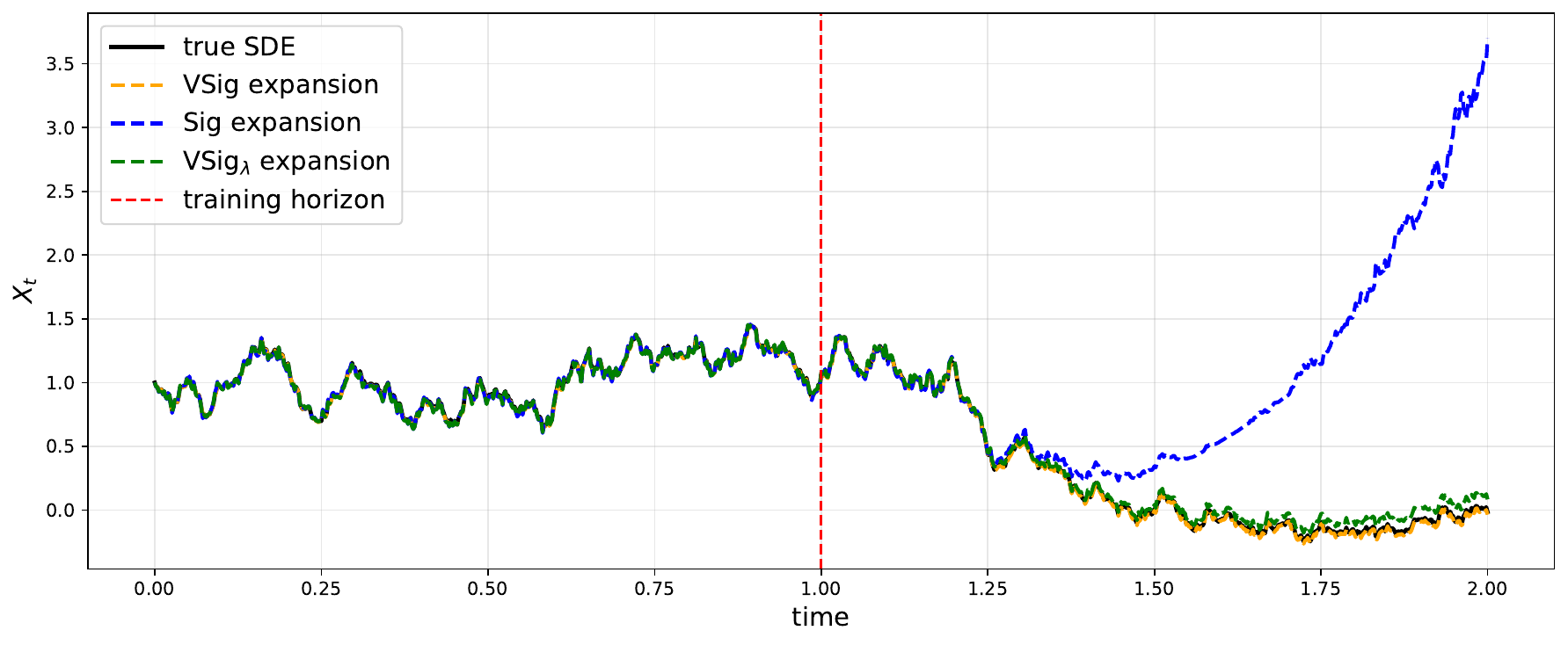}
  \caption{Volterra signature vs.\ classical signature expansions, compared with the fractional SDE solution \eqref{eq:SDE} for one testing sample. The models were trained with $M=900$ training samples and $N=500$ time-steps on $[0,1]$
Parameters: $Y_0=1, b_0=0,b_1=-1,\sigma_0=1,\sigma_1=0.5$, signature truncation $L=6$.}

  \label{fig:Gamma-kernel}. 
\end{figure}
We generate $M=1000$ samples paths  $\bar{B}$ on the grid $\TT$ with $T=2$ and $N=1000$. We choose the underlying kernel to be the fractional kernel $k_\beta(t,s)=\frac{(t-s)^{\beta-1}}{\Gamma(\beta)}$ with $\beta=1.10$, %
and along each sample of $\bar{B}$ we generate a sample of \(\bar{Y}\) using an Euler--Maruyama scheme. For the model  \texttt{Sig} we use the \emph{iisignature} library \cite{reizenstein2018iisignature}, while for the Volterra signatures we implemented algorithms presented in \cite[Section 4.2]{ii_part}. Choosing truncation level $L=6$, the optimization \eqref{eq:SDE_loss} is then easily solved as a linear least-square regression problem, where we consider a $90\%-10\%$ training-testing split, and $\cD = [0,1] \cap \TT$, that is we only train on half of the interval. For the model $\texttt{VSig}_{k_\lambda}$ we additionally perform a grid-search on $[0,10]$ for the parameter $\lambda$.

In Figure~\ref{fig:Gamma-kernel} we plot one testing-sample path of the solution $Y^{(j)}$, as well as the three trained models $\texttt{Sig},\texttt{VSig}_k,\texttt{VSig}_{k_\lambda}$ on the whole interval $[0,2]$. As the plot for one sample suggests -- and Table \ref{tab:MSE_R2_SDE} confirms -- all three expansions on the testing data fit the dynamics $Y$ almost perfectly on the same interval the models were trained. The results are slightly different outside the training interval, where the truncated signature expansion $\texttt{Sig}$ can visibly not  reproduce the dynamics. This does not contradict the explicit expansions \cite{jaber2024pathdependentprocessessignatures} for classical signature, but indicates that deeper levels are necessary. On the other-hand, $\texttt{VSig}_k$  continues to reproduce $Y$, which again reflects the expansion \cref{prop:linear_VCDE} together with the decay of $\mathrm{VSig}$ observed in the proof of \cref{lem:ROC_examples}. Finally, the model $\texttt{VSig}_{k_\lambda}$ also performs reasonably well on the whole interval, and thanks to the flexibility through the parameter $\lambda$, presents a serious competitor to $\texttt{VSig}_k$ even though we do not use the true underlying kernel $k$.

\begin{table}
\centering
\begin{tabular}{lcccc}
\hline
Method & $\mathrm{MSE}\!\mid_{[0,1]}$ & $R^2\!\mid_{[0,1]}$ & $\mathrm{MSE}\!\mid_{[0,2]}$ & $R^2\!\mid_{[0,2]}$ \\
\hline
\texttt{Sig}                 & $2.29 \times 10^{-4}$ & 1.000 & $2.9387$ & -0.923 \\
$\texttt{VSig}_k$            & $6.30 \times 10^{-5}$ & 1.000 & $0.0012$ & 0.999 \\
$\texttt{VSig}_{k_\lambda}$  & $1.38 \times 10^{-4}$ & 1.000 & $0.0124$ & 0.992 \\
\hline
\end{tabular}
\caption{Prediction performance of the three feature maps \texttt{Sig}, $\texttt{VSig}_k$, and $\texttt{VSig}_{k_\lambda}$, evaluated on the intervals $[0,1]$ and $[0,2]$. We report the mean squared error (MSE) and the coefficient of determination $R^2$.}

\label{tab:MSE_R2_SDE}
\end{table}
\subsection{S\&P 500 realized volatility forecasting}\label{sec:spx}
In our second numerical experiment, we apply Volterra signatures to volatility forecasting for the S\&P~500 index. We use daily S\&P~500 log-prices and the median realized volatility ($\sqrt{\text{MedRV}}$) from the Oxford-Man Institute’s realized volatility database \cite{heber2009oxford}, covering the period 2000-01-03 to 2018-06-26.\footnote{The data are publicly available; see, for instance, \url{https://github.com/onnokleen/mfGARCH/raw/v0.1.9/data-raw}.} Given the two time series of daily log-prices $x=\{x_{t_n}:0\le n\le N\}$ and realized volatilities $v=\{v_{t_n}:0\le n\le N\}$, we aim to learn
\begin{equation}\label{eq:forecasting}
\bx_n := (x_{t_0},\ldots,x_{t_n}) \longmapsto f_q(\bx_n)
:= \mathbb{E}\!\left[v_{t_{n+q}} \mid \bx_n\right], \qquad q\in\NN .
\end{equation}
In words, we forecast the realized volatility $q$ trading days ahead, given the past history of daily S\&P~500 log-prices. 
In particular, our predictors are functions of the \emph{price history only} (daily closes), and do not use past realized volatility values as additional inputs.

As already motivated in the introduction, it has been observed by both practitioners and researchers that asset-price volatility is history dependent, and Volterra processes provide a natural framework for modeling such effects. In the sequel, we focus on the following class of finite state space kernels
\begin{equation}\label{eq:kernel_class_SPX}
    K_{\lambda,\alpha,c}(t,s)= \mathbf{1}^\top e^{-\Lambda(t-s)}\begin{pmatrix}
        \alpha_1 \\ \alpha_2
    \end{pmatrix}, \quad \Lambda = \begin{pmatrix}
        \lambda_1 & -c \\ c & \lambda_2
    \end{pmatrix} 
    \qquad \alpha_1,\alpha_2 \in \mathbb{R}, \quad \lambda_1,\lambda_2,c \geq 0,
\end{equation}
which belongs to the general class of finite state-space kernels considered in Section~\ref{sec:dyn_exp} for $R=2$ and $q=1$. For $c=0$, \eqref{eq:kernel_class_SPX} are sum-of-exponentials kernels, which have shown to be effective in volatility modeling, e.g.\ for Markovian approximations of rough volatility models \cite{abi2019multifactor,bayer2023markovian}, and more recently in path-dependent volatility models \cite{guyon2023volatility,gazzani2025pricing}. We propose here a more general class of exponential kernels, with additional flexibility through the coupling/frequency parameter $c$. 

We define the targets by $y^{(i)} := v_{t_{i+q}}$ and for the input $x^{(i)}$, we use the piecewise linear interpolation of the augmented time series $\hat{x}_{t_i} := \big(x_{t_i},\; \sum_{j\le i}|(\delta x)_{t_{j-1},t_j}|,\; t_i\big),$
recalling that time augmentation guarantees linear universality (see Theorem~\ref{thm:lin_univ_exp}). We denote by $\texttt{VSig}_{\lambda,\alpha,c}$ the resulting method based on a Ridge regression (see \eqref{eq:p_SPX_LS} below) using the Volterra signature features $\VSig{\hat{x}^{(i)}}{k_{\lambda,\alpha}}$ and targets $y^{(i)}$. Here, the parameters $(\lambda_1,\lambda_2,c,\alpha_1,\alpha_2)$ are treated as hyperparameters: the decay rates $\lambda_1,\lambda_2$ encode different memory scales, the coupling parameter $c$ allows for interactions between the exponential factors, and the weights $\alpha_1,\alpha_2$ determine their relative contributions.  As before, we also consider the classical signature method \texttt{Sig} for comparison.
\begin{figure}
  \centering
  \includegraphics[width=1\textwidth]{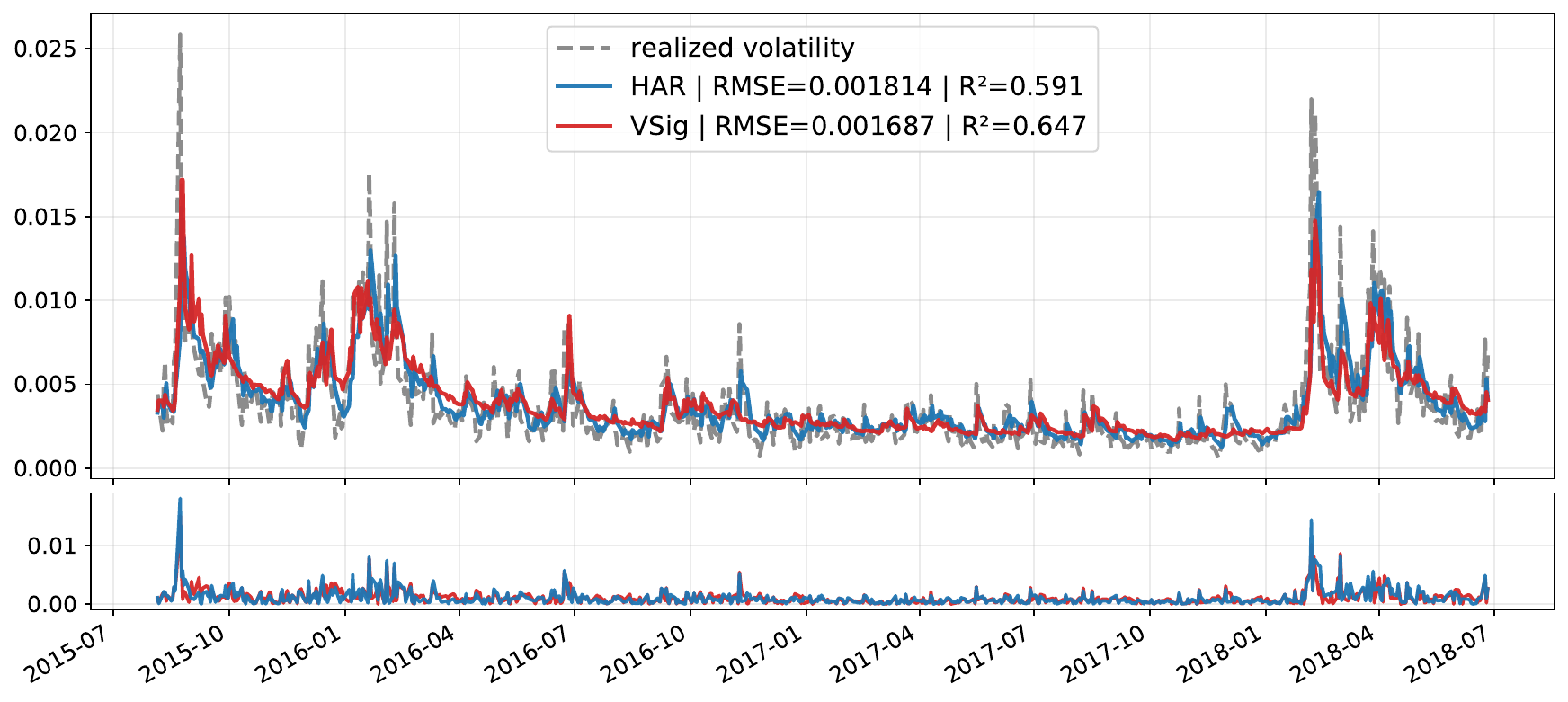}
  \caption{Next-day forecast of realized S\&P~500 volatility using our method \texttt{VSig}, compared with the benchmark \texttt{HAR}, on the test set. The lower subplot shows the absolute forecast errors \ $|\widehat{y}-y|$ for both methods, where $y$ denotes the realized volatility.}

  \label{fig:frequency_band} 
\end{figure}
As a benchmark, we consider the simple but powerful \texttt{HAR} model \cite{corsi2009simple}, i.e.\ a linear model for $v_{t_{n+q}}$ based on historical realized volatility such as $v_{t_n}$, $v_{t_{n-5}}$, and $v_{t_{n-22}}$. We emphasize again that the signature-based models use only daily historical prices as inputs (rather than historical realized volatility, which are computed from intraday price information), and thus rely on a different information set than \texttt{HAR}. Nevertheless, we use \texttt{HAR} forecasts as a standard benchmark for predictive performance.
\begin{rem}
    The model \texttt{VSig} is related to the path-dependent volatility (PDV) model introduced in \cite{guyon2023volatility}, where the authors empirically learn volatility path-dependence via the specification $v_t = \beta_0 + \beta_1 R_{1,t} + \beta_2 \sqrt{R_{2,t}},$ where
 \[
    R_{1,t}= \sum_{t_j\leq t}K_1(t,t_j)r_{t_j}, \qquad R_{2,t}= \sum_{t_j\leq t}
K_2(t,t_j) r_{t_j}^2, \qquad r_{t_j}= \frac{S_{t_j}-S_{t_{j-1}}}{S_{t_{j-1}}}.\] The features $(R_{1,t},R_{2,t})$ can be interpreted as discretized first-level Volterra signature features associated with the kernel $K=\mathrm{diag}(K_1,K_2)$, and already encode history dependence through weighted past returns and squared returns. Our method \texttt{VSig} therefore extends such models by lifting these memory states to higher-order Volterra signatures. This adds flexibility to capture more complex path-dependent interactions, together with our universality results providing theoretical justification for the resulting expressiveness.
\end{rem}

In the sequel, we compare the signature-based methods on sliding past windows $W_n^p := \{t_{n-p},\ldots,t_n\}$ of increasing length $p\in\NN$. More precisely, we consider the Ridge regression problem
\begin{equation}\label{eq:p_SPX_LS}
    \theta^\star=\theta^\star(p,q)
    = \argmin_{|\theta| \leq L}
    \frac{1}{N_{\mathrm{tr}}} \sum_{i=1}^{N_{\mathrm{tr}}} \Big( v_{t_{i+q}}
    - \big\langle \theta,\VSig{\hat{x}|_{W_i^p}}{k_{\alpha,\lambda}}^{(L)} \big\rangle \Big)^2
    + \eta \|\theta\|_{\ell^2}^2,
    \qquad \eta>0,
\end{equation}
where $\hat{x}|_{W_i^p}$ denotes the piecewise linear interpolation of $\hat{x}$ restricted to the past window $W_i^p$, i.e.\ based on the points $\{\hat{x}_{t_{i-p}},\ldots,\hat{x}_{t_i}\}$. Moreover, $L$ is the signature truncation level, which we also treat as a hyperparameter. In our experiments, we consider past-window sizes $p\in\{10,20,\ldots,250\}$ (in trading days) and forecasting horizons $q\in\{1,3,5\}$. We use an $80$--$20\%$ train--test split (i.e.\ $N_{\mathrm{tr}} = 0.8\times N$). The hyperparameters of \texttt{VSig}, as well as the ridge regularization parameter $\eta>0$ and truncation level $L$, are tuned on the last $20\%$ of the training set. For the \texttt{Sig} method, the best results were obtained using truncation level $L=4$, and for $\texttt{VSig}$ already at level $L=3$. The learned kernel parameters (rounded to two digits) for the \texttt{VSig} method are given by $$\lambda_1 = 22.69, \qquad \lambda_2= 0.14, \qquad \alpha_1=0.18, \qquad \alpha_2 = 16.02.$$

In general, \texttt{VSig} outperforms the classical signature model \texttt{Sig} and, for sufficiently long histories (i.e.\ for $p$ large enough), also the benchmark \texttt{HAR}. Given that \texttt{VSig} introduces additional flexibility through kernel hyperparameters, this improvement might not be surprising. Nevertheless, the results indicate that even simple kernel classes can substantially enhance the predictive power of signature features. In particular, the added flexibility appears to be essential for signature-based approaches to match and eventually surpass the \texttt{HAR} benchmark, as we further illustrate below. As a first illustration, Figure~\ref{fig:frequency_band} shows next-day forecasts ($q=1$ and $p=240$) of \texttt{VSig} versus \texttt{HAR} on the test set. Both the coefficient of determination $R^2$ and the root mean-squared error (RMSE) improve under \texttt{VSig}. Moreover, the time-series plot suggests that \texttt{VSig} yields the largest gains during high-volatility periods, consistent with the error subplot.

\begin{figure}
  \centering
  \includegraphics[width=1\textwidth]{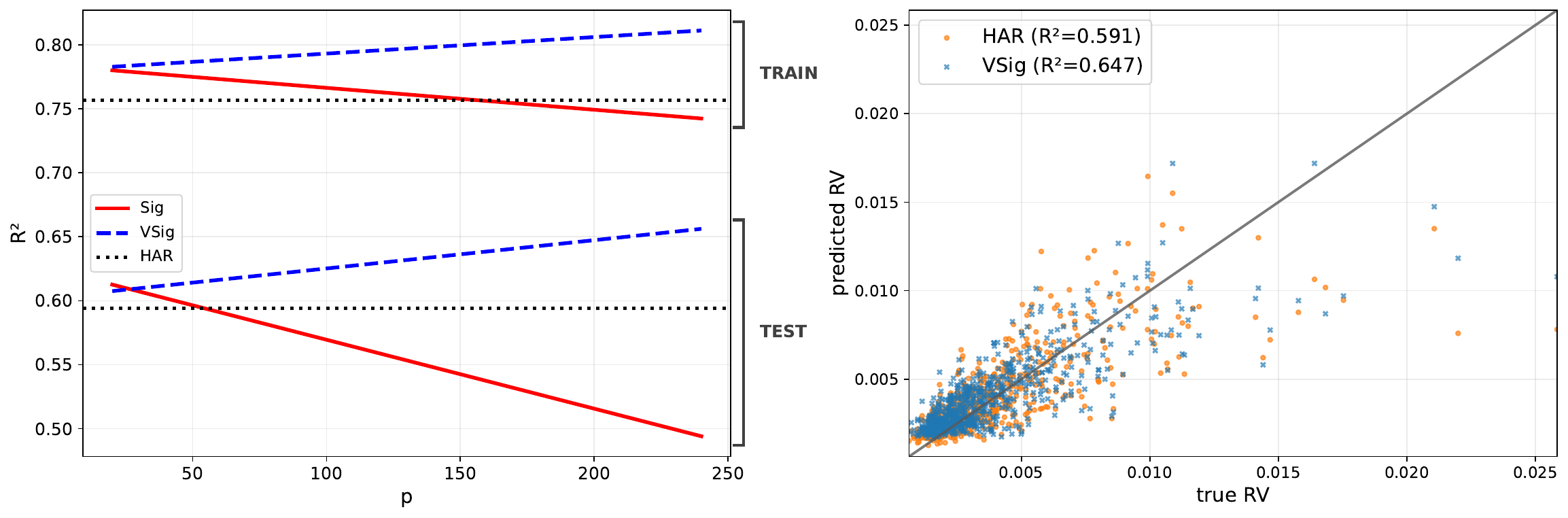}
  \caption{Left: Coefficient of determination $R^2$ as a (linearly interpolated)  function of the past-window size $p$ (days), reported on the training and test sets for the methods \texttt{VSig} and \texttt{Sig}, and the (constant) benchmark \texttt{HAR}. Right: Scatter plots of realized volatility $y$ versus predictions $\widehat{y}$ for our method \texttt{VSig} and the \texttt{HAR} benchmark.}

  \label{fig:r2scatter}
\end{figure}

In the left panel of Figure~\ref{fig:r2scatter}, we compare \texttt{VSig} with the classical signature method \texttt{Sig} by reporting a linear interpolation of the $R^2$ score as a function of the past-window length~$p$. On both the training and the test set, \texttt{Sig} deteriorates as $p$ increases, whereas \texttt{VSig} consistently benefits from incorporating a longer history. This behavior is consistent with the fact that, while past information is relevant for forecasting volatility, its predictive value is not uniform across time: distant observations typically carry less weight than more recent ones. In \texttt{VSig}, the kernel is tuned to encode such time decay and to aggregate past information in a more efficient, memory-aware manner. Some more precise values, also for larger prediction horizons, are reported in Table~\ref{tab:q_by_p_r2_triples}. 

\begin{table}[h]
\centering
\resizebox{\textwidth}{!}{%
\begin{tabular}{ccccccc}
\toprule
 & 20 & 60 & 110 & 150 & 200 & 240 \\
q &  &  &  &  &  &  \\
\midrule
1 & 0.59/\textbf{0.60}/0.60 & 0.58/\textbf{0.62}/0.60 & 0.57/\textbf{0.63}/0.60 & 0.56/\textbf{0.64}/0.60 & 0.52/\textbf{0.65}/0.60 & 0.47/\textbf{0.65}/0.60 \\
3 & 0.34/0.33/0.38 & 0.38/\textbf{0.38}/0.38 & 0.38/\textbf{0.39}/0.38 & 0.39/\textbf{0.40}/0.38 & 0.36/\textbf{0.40}/0.38 & 0.35/\textbf{0.42}/0.38 \\
5 & \textbf{0.22}/0.21/0.25 & \textbf{0.27}/0.24/0.25 & \textbf{0.29}/0.25/0.25 & \textbf{0.31}/0.25/0.25 & \textbf{0.28}/0.26/0.25 & 0.27/\textbf{0.28}/0.25 \\
\bottomrule
\end{tabular}
}
\caption{$R^2$ overview on the test set: (\texttt{Sig}/\texttt{VSig}/\texttt{HAR}) for forecasting horizons $q \in \{1,3,5\}$ and several past-window sizes $p$.}
\label{tab:q_by_p_r2_triples}
\end{table}

\subsection{Multivariate time-series classification}\label{sec:classification}
Finally, we present a first application of the Volterra signature kernel
\eqref{eq:sig-kernel-noise} to support vector machine (SVM) classification
on UEA time-series datasets \cite{bagnall2018uea}. The same benchmark was
considered in \cite[Section~5.1]{Salvi2021} for the classical signature
kernel, including variants with a linear and an RBF static kernel
(cf. Remark~\ref{rem:static_kernel_lift}). We extend this setup to the
Volterra signature kernel \eqref{eq:sig-kernel-noise}, based on the same finite state-space kernels from before, that is \eqref{eq:kernel_class_SPX}.  In the accompanying paper
\cite{ii_part}, we present an efficient and accurate solver (\cite[Algorithm~10]{ii_part}) for the system of
PDEs \eqref{eq:PDE1}--\eqref{eq:PDE_3}, which again is supported in \texttt{tensordev}. %See also \cite[Algorithm~10]{ii_part}.

Table~\ref{tab:classification_vsig_std} reports the baselines from
\cite[Table~1]{Salvi2021}, including the linear kernel, the RBF kernel, the
global alignment kernel (GAK), and the signature-PDE kernel. We add two
Volterra signature variants: \texttt{VSig}, using a linear static kernel, and
\texttt{VSig-RBF}, using an additional RBF static kernel. For both variants,
the parameters in \eqref{eq:kernel_class_SPX} are selected by
hyperparameter optimization on the training data.

In \cite{Salvi2021}, the authors additionally optimize over several data
augmentations, including time augmentation and lead-lag transformations. In
our experiments, we use only time augmentation for the Volterra signature
methods. Overall, the results show that the Volterra signature variants are competitive with both classical signature kernels and standard time-series baselines. The plain Volterra signature kernel \texttt{VSig} improves over \texttt{Sig} on most datasets and achieves the best score among all methods on \texttt{Libras} and \texttt{NATOPS}. This suggests that the additional flexibility of the Volterra kernel can already be beneficial without an additional nonlinear static kernel. The RBF-enhanced variant \texttt{VSig-RBF} gives the best performance on \texttt{ArticularyWordRecognition},\texttt{Cricket}, \texttt{RacketSports}, and \texttt{Heartbeat}. At the same time, standard baselines  remain strongest on datasets such as \texttt{FingerMovements}, \texttt{UWaveGestureLibrary} and \texttt{SelfRegulationSCP1}.
\begin{table}
\begin{tabular}{lccccccc}
\toprule
Datasets/Kernels & Linear & RBF & GAK & \texttt{Sig} & \texttt{VSig} & \texttt{Sig-RBF} & \texttt{VSig-RBF} \\
\midrule
ArticularyWordRecognition & 98.0 & 98.0 & 98.0 & 92.3 & 98.7 & 98.3 & \textbf{99.0} \\
BasicMotions & 87.5 & 97.5 & 97.5 & 97.5 & 97.5 & \textbf{100.0} & 97.5 \\
Cricket & 91.7 & 91.7 & \textbf{97.2} & 86.1 & 90.3 & \textbf{97.2} & \textbf{97.2} \\
Libras & 73.9 & 77.2 & 79.0 & 81.7 & \textbf{87.8} & 81.7 & 86.7 \\
NATOPS & 90.0 & 92.2 & 90.6 & 88.3 & \textbf{93.9} & 93.3 & 92.2 \\
RacketSports & 76.9 & 78.3 & 84.2 & 80.2 & 73.0 & 84.9 & \textbf{90.8} \\
FingerMovements & 57.0 & 60.0 & \textbf{61.0} & 51.0 & 54.0 & 58.0 & 51.0 \\
Heartbeat & 70.2 & 73.2 & 70.2 & 72.2 & 71.7 & 73.6 & \textbf{74.6} \\
SelfRegulationSCP1 & 86.7 & 87.3 & \textbf{92.4} & 75.4 & 80.5 & 88.7 & 89.4 \\
UWaveGestureLibrary & 80.0 & \textbf{87.5} & \textbf{87.5} & 83.4 & 85.6 & 87.0 & 86.6 \\
\bottomrule
\end{tabular}
\caption{Classification accuracies (\%) on UCR/UEA datasets. The best test accuracy in each row is highlighted in bold.}
\label{tab:classification_vsig_std}
\end{table}

\bibliographystyle{plain}
\bibliography{bib}

\end{document}